\documentclass[11pt]{article}

\usepackage{times}
\usepackage{verbatim}
\usepackage{amsthm,amsfonts,amsmath,amssymb,epsfig,color,float,graphicx,verbatim, enumitem}

\def\nnewcolor{0}

\ifnum\nnewcolor=1
\newcommand{\new}[1]{{\color{red} #1}}
\fi
\ifnum\nnewcolor=0
\newcommand{\new}[1]{#1}
\fi

\ifnum\nnewcolor=1
\newcommand{\dnew}[1]{{\color{blue} #1}}
\fi
\ifnum\nnewcolor=0
\newcommand{\dnew}[1]{#1}
\fi

\newif\ifhyper\IfFileExists{hyperref.sty}{\hypertrue}{\hyperfalse}
\hypertrue
\ifhyper\usepackage{hyperref}\fi

\newcommand{\inote}[1]{\footnote{{\bf [[Ilias: {#1}\bf ]] }}}

\makeatletter
\renewcommand{\section}{\@startsection{section}{1}{0pt}{-12pt}{5pt}{\large\bf}}
\makeatother

\newcommand{\R}{{\mathbb{R}}}
\newcommand{\Z}{{\mathbb Z}}
\newcommand{\N}{{\mathbb N}}
\newcommand{\poly}{\mathrm{poly}}
\newcommand{\polylog}{\mathrm{polylog}}
\newcommand{\eps}{\epsilon}
\newcommand{\dtv}{d_{\mathrm{TV}}}
\newcommand{\relu}{\mathrm{ReLU}}
\newcommand{\codim}{\mathrm{codim}}
\newcommand{\eqdef}{\stackrel{{\mathrm {\footnotesize def}}}{=}}
\newcommand{\wh}[1]{{\widehat{#1}}}
\newcommand{\littlesum}{\mathop{{\textstyle \sum}}}
\newcommand{\nspan}{\mathrm{span}}

\newcommand{\dkl}{d_{\mathrm {KL}}}
\newcommand{\VC}{\mathrm{VC}}
\newcommand{\A}{\mathcal{A}}
\newcommand{\bi}{\mathbf{i}}

\newcommand{\E}{{\bf E}}
\newcommand{\Var}{\mathbf{Var}}
\newcommand{\pr}{\mathbf{Pr}}

\usepackage[OT2,T1]{fontenc}
\DeclareSymbolFont{cyrletters}{OT2}{wncyr}{m}{n}
\DeclareMathSymbol{\Sha}{\mathalpha}{cyrletters}{"58}

\newtheorem{theorem}{Theorem}
\newtheorem{lemma}[theorem]{Lemma}
\newtheorem{proposition}[theorem]{Proposition}
\newtheorem{corollary}[theorem]{Corollary}
\newtheorem{claim}[theorem]{Claim}
\newtheorem{fact}[theorem]{Fact}
\theoremstyle{definition}
\newtheorem{definition}[theorem]{Definition}

\title{Small Covers for Near-Zero Sets of Polynomials\\
and Learning Latent Variable Models\footnote{An extended abstract of this work appears in the proceedings of the 61st Annual IEEE Symposium on Foundations of Computer Science (FOCS 2020).}}

\author{
Ilias Diakonikolas\thanks{Supported by NSF Award CCF-1652862 (CAREER), a Sloan Research Fellowship, 
and a DARPA Learning with Less Labels (LwLL) grant.}\\
UW Madison\\
{\tt ilias@cs.wisc.edu}\\
\and
Daniel M. Kane\thanks{Supported by NSF Award CCF-1553288 (CAREER) and a Sloan Research Fellowship.}\\
University of California, San Diego\\
{\tt dakane@cs.ucsd.edu}\\
}

\begin{document}

\maketitle

\thispagestyle{empty}

\vspace{-0.5cm}

\begin{abstract}
Let $V$ be any vector space of multivariate degree-$d$ homogeneous polynomials with co-dimension at most $k$,
and $S$ be the set of points where all polynomials in $V$ {\em nearly} vanish. We establish a qualitatively
optimal upper bound on the size of $\eps$-covers for $S$, in the $\ell_2$-norm. Roughly speaking, we show
that there exists an $\eps$-cover for $S$ of cardinality $M = (k/\eps)^{O_d(k^{1/d})}$. Our result is constructive yielding
an algorithm to compute such an $\eps$-cover that runs in time $\poly(M)$.

Building on our structural result,
we obtain significantly improved learning algorithms for several fundamental high-dimensional
probabilistic models with hidden variables. These include density and parameter estimation for $k$-mixtures
of spherical Gaussians (with known common covariance),
PAC learning one-hidden-layer ReLU networks with $k$ hidden units (under the Gaussian distribution),
density and parameter estimation for $k$-mixtures of linear regressions (with Gaussian covariates), and
parameter estimation for $k$-mixtures of hyperplanes. Our algorithms run in time {\em quasi-polynomial}
in the parameter $k$. Previous algorithms for these problems had running times exponential in $k^{\Omega(1)}$.

At a high-level our algorithms for all these learning problems work as follows: By computing the low-degree
moments of the hidden parameters, we are able to find a vector space of polynomials
that nearly vanish on the unknown parameters. Our structural result allows us to compute a
quasi-polynomial sized cover
for the set of hidden parameters, which we exploit in our learning algorithms.
\end{abstract}

\thispagestyle{empty}
\setcounter{page}{0}

\newpage

\section{Introduction} \label{sec:intro}

\subsection{Background and Motivation} \label{ssec:motivation}
The main motivation behind this work is the problem of designing efficient learning algorithms for
high-dimensional probabilistic models with latent (hidden) variables. This general question
has a long history in statistics, starting with the pioneering work of Karl Pearson~\cite{Pearson:94}
on learning Gaussian mixtures, that introduced the method of moments in this context.
During the past decades, an extensive line of work in theoretical computer science and machine learning
has made significant progress on various statistical and computational aspects of this broad question.

In this paper, we focus our attention on high-dimensional latent variable models with
a large number $k$ of hidden parameters\footnote{By this we mean that $k = \omega(1)$, in which case
an algorithm with runtime exponential in $k$ is not deemed satisfactory.}.
In the settings we study, previously known learning algorithms have
running times that scale exponentially with $k$. Roughly speaking, this exponential dependence is
typically due to some form of ``brute-force'' search, after the high-dimensional problem is
reduced down to a $k$-dimensional one. It should be noted that, in certain regimes, the exponential
dependence on $k$ is inherent, due to either information-theoretic (see, e.g.,~\cite{MoitraValiant:10, HardtP15})
or computational (see, e.g.,~\cite{DKS17-sq}) bottlenecks. For the problems we study here,
there is no (known) a priori reason ruling out $\poly(k)$ time algorithms, while current algorithms
have an $\exp(k^{\Omega(1)})$ dependence.

Motivated by this huge gap in our understanding, we develop new algorithms
for several high-dimensional probabilistic models with running times {\em quasi-polynomial}
in the number $k$ of hidden parameters. More specifically, we design new algorithms for the following
fundamental statistical tasks: density estimation and parameter learning for $k$-mixtures of spherical Gaussians,
PAC learning one-hidden-layer neural networks with $k$ hidden ReLU gates and other well-behaved activations (\new{including generalized linear models}) 
under the Gaussian distribution, density estimation and parameter estimation 
for $k$-mixtures of linear regressions (under Gaussian covariates),
and parameter learning for $k$-mixtures of hyperplanes. See Section~\ref{ssec:intro-apps} for detailed statements
of our results and comparison to prior work.

All our learning algorithms are based on a new technique that we develop in this work.
The key common ingredient is a new result in algebraic geometry
that we believe is of independent interest. In more detail, we establish the following:
Let $V$ be any vector space of multivariate degree-$d$ homogeneous polynomials
with co-dimension at most $k$ and $S$ be the set of points where all polynomials in $V$ {\em nearly} vanish.
Then the set $S$ has an $\eps$-cover, in $\ell_2$-norm, of size $M = (k/\eps)^{O_d(k^{1/d})}$.
Importantly, our proof is constructive immediately giving an algorithm to compute such a cover
that runs in $\poly(M)$ time.

With this structural result in hand, all our learning algorithms follow a common recipe:
First, given a set of samples from our distribution there is an efficient procedure to approximate
the degree-$2d$ moments of the hidden parameters. Then we use our structural result
to compute a small $\eps$-cover for the set of hidden parameters. Once we have
a cover of the parameters, we leverage problem-specific techniques
to perform density estimation or parameter estimation.

\subsection{Overview for Our Approach} \label{ssec:intuition}

In this section, we give an overview of our approach with a focus 
on the problem of learning mixtures of spherical Gaussians.
In particular, we explain how our aforementioned structural result 
(regarding covers of near-zero sets of polynomials)
naturally comes into play to find a cover for the set of hidden parameters.

Suppose we have access to i.i.d. samples from an unknown $k$-mixture of
identity covariance Gaussians on $\R^m$, $X = \sum_{i=1}^k w_i N(\mu_i, I)$,
where $w_i \geq 0$ are the mixing weights, satisfying $\sum_{i=1}^k w_i = 1$, 
and $\mu_i \in \R^m$ are the mean vectors. There are two versions of the learning problem:
(1) Density estimation, where the goal is to compute a hypothesis distribution $H$ 
that is $\eps$-close to $X$ in total variation distance, and (2) Parameter estimation, where
the goal is to approximate the parameters $w_i, \mu_i$ within small error $\eps$.
Our approach yields significantly improved algorithms for both these problems via a common technique.
In particular, we develop a method to efficiently find a cover for the set of hidden mean vectors,
i.e., a set $\mathcal{C} \subset \R^m$ such that for any $\mu_i$, $i \in [k]$, with $w_i$ not too small,
there exists $c \in C$ such that the $\ell_2$-distance $\|c - \mu_i\|_2$ is small.

A natural approach to learn a $k$-mixture of Gaussians is to use the method of moments. 
This method has two steps: (i) We draw sufficiently many samples to accurately 
approximate the first $d$ moments of the mixture $X$.
(ii) We use our approximations to the moments to compute an approximation of the distribution 
or its parameters. Unfortunately, the method of moments faces the following obstacle in our context:
There exists two $k$-mixtures of spherical Gaussians, $X$ and $X'$, 
that are far from each other, but have their first $k$ moments exactly matching.
This means that one cannot compute an approximation to $X$ from the first $d<k$
moments alone.

The above moment-matching statement might suggest that {\em any} moment-based method
cannot lead to learning algorithms with running time $2^{o(k)}$ for our problem.
However, looking at the structure of these moment-matching distributions 
gives us hope. Essentially, these instances are based on a one-dimensional
construction that matches $k$ moments, which is then embedded into a
higher dimensional space. If $X$ and $X'$ are constructed by having all
of their Gaussian components centered on an unknown line $L$, one might not be able to
distinguish $X$ and $X'$ directly by using their low-degree moments,
but looking at second moments should suffice to approximately determine the line $L$. 
Once this line is determined, it would allow us to reduce down to
a one-dimensional problem, which can be efficiently solved by other means. 
Of course, the task of finding the hidden line $L$ could be made more difficult by
adding more components to each of $X$ and $X'$, but it is not clear
whether or not this could successfully disguise this critical line.

In order to obtain a truly insurmountable hard instance, we would need
to construct a $k$-mixture $X$, such that not only do 
the higher moment tensors of $X$ agree with those of some other $k$-mixture $X'$ (that is far from $X$), 
but in addition the higher moment tensors of $X$ are {\em rotation-invariant}. 
Such a (hypothetical) construction would imply that the low-degree moments of $X$ 
are indistinguishable from any rotation of $X$, and therefore it would be impossible 
to locate lower-dimensional sub-structures, like the line $L$ above.

Our approach is motivated by the fact that such a hypothetical hard instance is 
in fact impossible. In particular, we can write our unknown $k$-mixture $X$ as a 
convolution $D \ast G$, where $G \sim N(0, I)$ is the standard Gaussian, 
and $D$ is a discrete distribution on $\R^m$ with support size at most $k$. 
By de-convolving, we can use the moments of $X$ to compute the moments of $D$.
 Now, if $\binom{m+d}{d} > k$, a dimension counting argument implies that
there exists a non-trivial degree-$d$ polynomial $p$
that vanishes on the support of $D$. This means that $\E[p^2(D)]$ is also $0$. 
But if we know the first $2d$ moments of $X$, we can in principle find such a polynomial $p$, 
which would imply that $p$ must be identically zero on the support of $D$. That is,
if we know the first $2d$ moments of $X$, we can find a polynomial $p$ that vanishes on the support of $D$, 
and unless $p(x)$ is a function of $\|x\|_2^2$, this will not be a rotationally
invariant condition, implying that the moments of $X$ cannot be
rotationally invariant.

The above paragraph naturally leads to an idea for an algorithm. 
Note that, for any $d$, the space of degree-$d$ polynomials on $\R^m$
has dimension $N = \binom{m+d}{d}$. 
By the same dimension counting argument,
there exists a subspace $V$ of degree-$d$ polynomials
with dimension at least $N-k$ that vanishes on the support of $D$. 
On the other hand, given the first $2d$ moments of $D$, we can identify
$V$ as the space of polynomials $p$ so that $\E[p^2(D)] = 0$. (We note that 
this is indeed a subspace, since the quadratic form $q \to \E[q^2(D)]$ is positive
semi-definite). If we know $V$, we know that all the component means 
of our mixture must lie on the variety $\mathbb{V}$ defined by the
polynomials in $V$. It is not hard to show that this variety $\mathbb{V}$ will have relatively
small dimension. This holds because the space of degree-$d$ polynomials on $\mathbb{V}$ is 
(degree-$d$ polynomials on $\R^m$)$/V$,which has dimension at most $k$. 
This implies that $\binom{\dim(\mathbb{V})+d}{d} \leq k$,
and in particular that $\dim(\mathbb{V}) = O(d k^{1/d})$. This allows us to reduce our
problem to one on a variety of small dimension that we can hopefully
brute force in time exponential in $k^{1/d}$. Indeed, we are able to show that the variety $\mathbb{V}$ 
will have a small cover. (Of course, having a variety with small dimension does not
imply the existence of a small cover in general. But our variety has additional properties that our proof
exploits.)

The biggest technical obstacle to the approach outlined above is, of course, 
that we cannot have access to the {\em exact} moments of $X$ (and thus $D$), but 
can only hope to approximate them. However, if we have sufficiently accurate 
approximations to the moments of $D$, we can still find a vector space $V$ of degree-$d$ polynomials such 
that for all $p \in V$ we have that $\E[p(D)^2]$ is {\em small}. This implies that for any point 
$x$ in the support of $D$, with reasonable mass, $p(x)$ must {\em nearly} vanish 
for all $p \in V$. At this point, we will need a robust version of the aforementioned structural result,
which is the main geometric result of this work (Theorem~\ref{thm:main-cover-informal}).
This result essentially says the following: Given such a $V$ and a unit ball $B$, if we define $S$ to
be the set of all points $x$ in $B$ such that $|p(x)|$ is small for all $p \in V$, 
then $S$ can be covered by approximately $\exp(O(k^{1/d}))$ many small balls. Moreover, there
is an efficient algorithm to compute such a cover.
This allows us to compute an explicit set of (not too many) hypotheses
means $x_i$ such that each center of a Gaussian in $X$ with reasonable
weight is close to some $x_i$.

Given our cover for the set of possible parameters, we can solve both the density estimation
and the parameter estimation problems as follows: 
For density estimation, we note that $X$ can be approximated as a mixture of the $N(x_i,I)$'s. 
We can thus draw samples from $X$ and use convex optimization to compute 
appropriate mixing coefficients (Proposition~\ref{mixtureProp}).
For parameter estimation, if we assume separation of the components of $X$, 
we can use the list of hypotheses means to do clustering and learn approximations of
the true means using techniques from~\cite{DiakonikolasKS18-mixtures}.

More broadly, our technique can also be applied to a number of other high-dimensional learning problems. The
key requirement is that the unknown distribution in question is determined by a set of $k$ vectors 
$v_i \in \R^m$ and non-negative weights $w_i$, and that we can efficiently approximate the quantity 
$\sum_{i=1}^k w_i p(v_i)$, for any low degree polynomial $p$. Given this primitive, 
we can use our Theorem~\ref{thm:main-cover-informal} to find a subspace $V$ 
of polynomials that almost vanish on the $v_i$'s, and from there compute a small list of
hypotheses so that each relevant $v_i$ must be close to at least one such hypothesis. 
From this point on, we can use efficient algorithms operating on the final cover and/or problem-specific
techniques to complete the learning algorithm.

\subsection{Main Result: Small Covers for Near-Zero Sets of Polynomials}

Let $V$ be any vector space of homogeneous degree-$d$ real polynomials on $\R^m$
with co-dimension $k$. We use $\R_{[d]}[x_1, \ldots, x_m]$ for the vector space of all 
homogeneous degree-$d$ real polynomials on $\R^m$.
Let $S$ be the set of points where all polynomials in $V$ are close to zero.
Our main result shows that $S$ has a small cover that can be computed efficiently.
Specifically, we show:


\begin{theorem}[informal] \label{thm:main-cover-informal}
Let $V$ be any vector space of homogeneous degree-$d$ real polynomials
on $\R^m$ with codimension at most $k$ within $\R_{[d]}[x_1, \ldots, x_m]$.
For $\delta, R > 0$, let
$$S = S(V, R, \delta) \eqdef \{ x \in \R^m: \|x\|_2 \leq R \textrm{ and } |p(x)| \leq \delta \|p\|_{\ell_2} \textrm{ for all } p\in V \} \;.$$
Then, for sufficiently small $\delta>0$,
there exists an $\eps$-cover of $S$ with size at most $M = (2 (R/\eps) d k)^{O(d^2 k^{1/d})}$. Moreover, there exists
an algorithm to compute such a cover in $\poly(M)$ time.
\end{theorem}

\noindent See Theorems~\ref{thm:main-cover} and~\ref{thm:main-cover-alg} for more detailed formal statements.

\paragraph{Very Brief Proof Overview.}
The proof of Theorem \ref{thm:main-cover-informal} is elementary, but quite technically involved. 
At a very high level, we consider what happens when we fix the first $m'$ coordinates of a point $x \in \R^m$. 
Plugging in these values will change $V$ from a space of polynomials in $m$ variables to a space of polynomials in $m-m'$ variables. 
Since the latter space is much smaller, generically we should expect that this restriction of $V$ produces 
a very large space of such polynomials, implying (by way of an inductive application of our theorem) 
that there are very few ways to fill in the remaining coordinates and still lie in $S$. This will hold unless the chosen values 
satisfy the unusual property that when plugged into polynomials of $V$ they cause many of them to vanish or nearly vanish. 
We show that this circumstance is in fact rare by showing that all points for which this holds must lie near a low-dimensional hyperplane. 
By restricting our functions to this hyperplane, we can again use our theorem inductively to handle these bad points.

\paragraph{Discussion.}
It is instructive to consider Theorem~\ref{thm:main-cover-informal} in the special case where $\delta=0$. 
Here $S$ is the intersection of a variety $\mathbb{V}$ with a ball of $\ell_2$-radius $R$, 
and we are asking the natural question of how many balls are needed to cover the real points of an algebraic variety. 
For sufficiently nice varieties, we should expect to have a cover of size approximately $O(R/\eps)^{\dim(\mathbb{V})}$. 
The constraint that the generating set $V$ is so large does imply strong bounds on the dimension of $\mathbb{V}$. 
In particular, the fact that $V$ has codimension $k$ implies that the restriction of the space of degree-$d$ 
homogenous polynomials to $\mathbb{V}$ has dimension at most $k$, which in turn implies that
$\binom{\dim(\mathbb{V})+d-1}{d} \leq k$, and therefore $\dim(\mathbb{V}) = O(dk^{1/d})$. 
Note that this bound is actually tight in the case that $\mathbb{V}$ is a hyperplane, 
thus requiring covers of size $(R/\eps)^{\Omega(d k^{1/d})}$ (for $d\ll \log(k)$) even in the $\delta=0$ case. 

The above argument allows one to show that the dependence of our cover size upper bound on $R/\eps$ is approximately 
best possible, as the dimension should equal the metric dimension which is the limit of the logarithm 
of the cover size over $\log(1/\eps)$. However, in order to prove this for finite values of $(R/\eps)$, 
one needs to have information not just about the dimension of $\mathbb{V}$, 
but also about the geometric complexity of the variety. 
It is perhaps not surprising that such bounds can be obtained (for example because the codimension of $V$ should 
bound the {\em degree} of the variety $\mathbb{V}$), but it seems technically highly non-trivial to do so. 
Further technical complications arise when one considers the case of $\delta>0$, i.e., 
one needs to consider points that are merely \emph{close} to satisfying the equations in 
$\mathbb{V}$.

Another instructive example here is the case where $d=1$. 
In this case, $V$ is a space of linear functions that all vanish on a hyperplane $H$ with dimension at most $k$. 
It is easy to see that only points within distance $\delta$ of $H$ will lie in $S$, 
thus making it easy to produce a cover of size $O(R/\eps)^k$. Given the way that we will use 
Theorem \ref{thm:main-cover-informal} in our applications, the degree-$1$ case will end up looking 
very similar to the dimension reduction techniques already known for many of these problems. 
These techniques involve computing the second moments of the object in question 
and noting that the second moment matrix will have small singular values in directions 
perpendicular to the span of the $k$ hidden parameters. This provides us with an $(m-k)$-dimensional 
subspace of directions on which none of the (significant) parameters has a large projection, 
allowing one to find a subspace $H$ that nearly passes through all of them. 
From this point, one can usually reduce to a $k$-dimensional problem by restricting to or projecting onto $H$. 

In our setting, instead of computing second moments, we compute degree-$2d$ moments. 
This allows us to compute not just linear functions that nearly vanish on our points, 
but many functions of degree up to $d$. This gives us a much-smaller dimensional variety which our points must lie near. 
Unfortunately, since this new variety is potentially much more complicated than a subspace, we cannot generally 
project onto it and reduce to a lower dimensional version of the same problem. However, Theorem~\ref{thm:main-cover-informal} 
will allow us to find a small cover of this variety, which can then be used in a brute force manner to solve many of our problems.

More formally, by computing the first $2d$ moments of our distribution, 
we can solve some equations to compute the first $2d$ moments of our parameters. 
This will allow us to approximate the values of $\sum_{i=1}^k q(v_i)$ for any degree-$2d$ polynomial $q$. 
In particular, we look for degree-$d$ polynomials $p$ for which $\sum_{i=1}^{k} p^2(v_i)$ is small. 
We note that this will hold if and only if $p$ nearly vanishes on all $v_i$. 
However, we are guaranteed that a large space of such polynomials will exist 
and we can find it by an appropriate singular value decomposition. 
These $p$'s will provide the subspace $V$ needed by Theorem \ref{thm:main-cover-informal}, 
which in turn will provide us with a small cover $\mathcal{C}$. The elements of $\mathcal{C}$ 
can be thought of as hypotheses for our parameters, and we are guaranteed that each $v_i$ 
will be close to at least one of our hypotheses. From this point, we can make use 
of various algorithms to solve our problem that will run in time polynomial in the cover size.

\subsection{Applications: Learning Latent Variable Models} \label{ssec:intro-apps}

\new{In this section, we present some algorithmic applications of our main result
to the problem of learning various latent variable models.  
We illustrate the power of our techniques by focusing on a small 
set of learning tasks. For each of these tasks, 
we obtain significantly more efficient algorithms compared to prior work. 
We expect that the algebraic geometry tools introduced in this work are applicable to several other learning tasks. 
This is left as an interesting direction for future work.}


\subsubsection{Learning Mixtures of Spherical Gaussians}
A $k$-mixture of spherical Gaussians is a distribution on $\R^m$
with density function $F(x) = \sum_{j=1}^k w_j N(\mu_j,  I)$,
where $\mu_j \in \R^m$ are the unknown mean vectors and $w_j \geq 0$,
with $\sum_{j=1}^k w_j = 1$, are the mixing weights. We assume that the components
have the same known covariance matrix, which we can take for simplicity to be the identity matrix.

We will consider both density estimation and parameter estimation.
In density estimation, we want to output a hypothesis distribution with small total variation
distance from the target. In parameter estimation, we assume that the component means
are sufficiently separated,  and the goal is to recover the unknown mixing
weights and mean vectors to small error.

\paragraph{Prior Work on Learning Mixtures of Spherical Gaussians}
Gaussian mixture models are one of the most extensively studied latent variable models,
starting with the pioneering work of Karl Pearson~\cite{Pearson:94}.
In this paper, we focus on the important special case where each component
is spherical. Here we survey the most relevant prior work on density estimation and parameter estimation
for this distribution family.

In density estimation, the goal is to output some hypothesis that is close to the unknown mixture in total variation distance.
Density estimation for mixtures of spherical Gaussians in both low and high dimensions
has been studied in a series of works~\cite{FOS:06, MoitraValiant:10,CDSS13, CDSS14, SOAJ14,
DK14, BhaskaraSZ15,  HardtP15, ADLS15, DKKLMS16, LiS17, AshtianiBHLMP18}.
The sample complexity of this learning task for $k$-mixtures on $\R^m$, for variation distance error $\eps$,
is easily seen to be $\poly(mk/\eps)$, and a nearly tight bound of $\tilde{\Theta}(m k/\eps^2)$
was recently shown~\cite{AshtianiBHLMP18}. Unfortunately, all previous algorithms for this learning problem
have running times that scale exponentially with the number of components $k$.
Specifically,~\cite{SOAJ14} gave a proper density estimation algorithm that uses
$\poly(mk/\eps)$ samples and runs in time $\poly(mk/\eps)+m^2 (k/\eps)^{O(k^2)}$.

In parameter estimation, the goal is to output the parameters of the data generating distribution, up to small error.
For this problem to be information-theoretically solvable with polynomial sample complexity,
some further assumptions are needed. The typical assumption involves some kind of pairwise separation
between the component means. The algorithmic problem of parameter estimation for high-dimensional
Gaussian mixtures under separation conditions was first studied by Dasgupta~\cite{Dasgupta:99}, followed by a long series
of works~\cite{AroraKannan:01, VempalaWang:02, AchlioptasMcSherry:05, KSV08, BV:08, RV17-mixtures,
HopkinsL18, KS17, DiakonikolasKS18-mixtures}.
For the simplicity of this discussion, we focus on the case of uniform mixtures with identity covariance components.
\cite{RV17-mixtures} showed that, in order for the problem to be information-theoretically solvable
with $\poly(m, k)$ samples, the minimum pairwise $\ell_2$-mean separation should be $\Theta(\sqrt{\log k})$.
Subsequently, three independent works~\cite{HopkinsL18, KS17, DiakonikolasKS18-mixtures} gave parameter estimation
algorithms with sample complexities and running times $\poly(m, k^{\polylog(k)})$
that succeed under the optimal separation of $\Theta(\sqrt{\log k})$.

Finally, a related line of work~\cite{HK, BhaskaraCMV14, AndersonBGRV14, GeHK15}
studied parameter estimation in a smoothed-like setting, where (instead of separation conditions)
one makes certain condition number assumptions about the parameters. These results
are incomparable to ours, as we make no such assumptions.

We are now ready to state our algorithmic contributions 
for this problem.
For the task of density estimation, we prove:

\begin{theorem}[Density Estimation for Spherical $k$-GMMs]\label{thm:gmm-density-informal}
There is an algorithm that on input $\eps>0$, and
$\tilde{O}(m^2) \poly(k/\eps) + (k/\eps)^{O(\log^2 k)}$ samples
from an unknown $k$-mixture of spherical Gaussians $F$ on $\R^m$, it
runs in time $\poly(m k/\eps)+ (k /\eps)^{O(\log^2 k)}$ and outputs a hypothesis distribution $H$ such that
with high probability $\dtv(H, F) \leq \eps$.
\end{theorem}

\noindent (See Theorem~\ref{thm:gmm-density-est} for a more detailed formal statement.)
Prior to this work, the fastest known algorithm for this learning problem had running time
exponential in $k$, in particular $\poly(m) (k/\eps)^{O(k^2)}$~\cite{SOAJ14}. Interestingly,
our density estimation algorithm is not proper. The hypothesis $H$ it outputs is an $\ell$-mixture
of identity covariance Gaussians, where $\ell \gg k$.

For the task of parameter estimation, we prove:

\begin{theorem}[Parameter Estimation for Spherical $k$-GMMs]\label{thm:gmm-param-informal}
There is an algorithm that on input $\eps>0$, $d \in \Z_+$, and
$N = \tilde{O}(m^2) \poly(k)+ \poly(k/\eps) + k^{O(d)} $ samples
from a uniform $k$-mixture $F = (1/k) \sum_{i=1}^k N(\mu_i, I)$ on $\R^m$
with pairwise mean separation $\Delta = \min_{i \neq j} \|\mu_i - \mu_j\|_2 \geq C \sqrt{\log k}$,
where $C$ is a sufficiently large constant, 
the algorithm runs in time $\poly(N)+k^{O(d^2 k^{1/d})}$, and outputs a list of candidate means
$\widetilde{\mu}_i$ such that with high probability we have that $\|\mu_i - \widetilde{\mu}_{\pi(i)} \| \leq \eps$, $i \in [k]$,
for some permutation $\pi\in \mathbf{S}_k$.
\end{theorem}

\noindent (See Theorem~\ref{thm:gmm-param} for a more detailed statement handling non-uniform mixtures as well.)
Prior to this work,~\cite{HopkinsL18, KS17, DiakonikolasKS18-mixtures} gave algorithms for this problem
with sample complexities and runtimes $\poly(m/\eps, k^{\polylog(k)})$. Our algorithm
provides a tradeoff between sample complexity and running time (by increasing the parameter $d$ from constant to $\log k$).
In particular, for $d=\log k$, the algorithm of Theorem~\ref{thm:gmm-param-informal} matches 
the best known (quasi-polynomial in $k$) sample complexity and runtime.
More importantly, by taking $d$ to be a large universal constant, we obtain an algorithm
with polynomial sample complexity $\poly(m/\eps) k^{c}$, $c>0$, and sub-exponential time $\poly(m/\eps) 2^{O_c(k^{1/c})}$.
No algorithm with polynomial sample complexity and $2^{o(k)}$ time was previously known under 
any $\polylog(k)$ separation.

\paragraph{Additional Discussion.} In this paragraph, we provide two remarks that are useful
to put our algorithmic contributions (Theorems~\ref{thm:gmm-density-informal} and~\ref{thm:gmm-param-informal}) in context.

\cite{DKS17-sq} gave a Statistical Query (SQ) lower bound of $m^{\Omega(k)}$ on the complexity of density estimation
for $k$-mixtures of Gaussians in $\R^m$. The hard instances constructed in that work are far from spherical.
A question posed in~\cite{DKS17-sq} was whether $2^{k^c}$, for some constant $0<c<1$,
or even $k^{\omega(1)}$ SQ lower bounds  can be shown for learning $k$-mixtures of spherical Gaussians.
The algorithmic results of this paper were inspired by our unsuccessful efforts to prove such lower bounds.
In particular, an SQ lower bound of the form $2^{k^c}$ is ruled out by Theorem~\ref{thm:gmm-density-informal}. 
An SQ lower bound of the form $k^{\omega(1)}$ is still possible, in principle. 
Given our quasi-polynomial upper bound, it is a plausible
conjecture that a $\poly(k)$ time algorithm is attainable.

The list-decodable Gaussian mean estimator of~\cite{DiakonikolasKS18-mixtures}, with runtime $m^{O(\log(1/\alpha))}$,
combined with a known dimension-reduction~\cite{VempalaWang:02} and a post-processing
clustering step, gives a $\poly(m/\eps, k^{\log k})$ sample and time algorithm
for parameter learning of spherical $k$-GMMs, under the information-theoretically optimal mean separation. Due to an SQ lower bound shown in~\cite{DiakonikolasKS18-mixtures} for list-decodable mean estimation, Theorem~\ref{thm:gmm-param-informal} cannot be obtained via a reduction to list-decoding.




\subsubsection{Learning One-hidden-layer ReLU Networks}
A one-hidden-layer ReLU network with $k$ hidden units is any function $F: \R^m \to \R$
that can be expressed in the form $F(x) = \sum_{i=1}^k a_i \relu (w_i \cdot x)$, for some unit vectors $w_i \in \R^m$
and $a_i \in \R_+$, where $\relu(t) = \max\{0, t\}$, $t \in \R$. We will denote by $\mathcal{C}_{m, k}$
the class of all such functions.

The PAC learning problem for $\mathcal{C}_{m, k}$ is the following:
The input is a multiset of i.i.d. labeled examples $(x, y)$, where $x \sim N(0,I)$ and $y = F(x)+\xi$,
for some $F \in \mathcal{C}_{m, k}$ and $\xi \sim N(0,\sigma^2)$, with $\xi$ independent of $x$.
We will call such an $(x,y)$ a {\em noisy sample} from $F$.
The goal is to output a hypothesis $H: \R^m \to \R$ that with high probability is close to $F$ in $L_2$-norm.

\paragraph{Prior Work on Learning One-hidden-layer ReLU Networks}
In recent years, there has been an explosion of research
on provable algorithms for learning neural networks
in various settings, see, e.g.,~\cite{Janz15, SedghiJA16, DanielyFS16, ZhangLJ16,
ZhongS0BD17, GeLM18, GeKLW19, BakshiJW19, GoelKKT17, Manurangsi18, GoelK19, VempalaW19} for
some works on the topic.
Many of these works focused on parameter learning---the problem of
recovering the weight matrix of the data generating neural network.
We also note that PAC learning of simple classes of neural networks has been studied
in a number of recent works~\cite{GoelKKT17, Manurangsi18, GoelK19, VempalaW19}.

The work of~\cite{GeLM18} studies the parameter learning of positive linear combinations of ReLUs
under the Gaussian distribution in the presence of additive noise.
It is shown in~\cite{GeLM18}  that the parameters can be approximately
recovered efficiently, under the assumption that the weight matrix is full-rank with bounded condition number.
The sample complexity and running time of their algorithm scales polynomially with the condition number.
More recently,~\cite{BakshiJW19, GeKLW19} obtained efficient parameter learning algorithms for vector-valued
depth-$2$ ReLU networks under the Gaussian distribution. Similarly, the algorithms in these works
have sample complexity and running time scaling polynomially with the condition number.

In contrast to parameter estimation, PAC learning one-hidden-layer ReLU networks
does not require any assumptions on the structure of the weight matrix. The PAC learning
problem for this class is information-theoretically solvable with polynomially
many samples. The question is whether a computationally efficient algorithm exists.
Until recently, the problem of PAC learning positive linear combinations of ReLUs
had remained open, even under Gaussian marginals and for $k=3$, and had been posed as an open problem
by Klivans~\cite{Kliv17}. Recent work~\cite{DKKZ20} gave the first non-trivial PAC algorithm for this problem. The algorithm
in~\cite{DKKZ20} uses $\poly(mk/\eps)$ samples, and has runtime $\poly(mk/\eps)+(k/\eps)^{O(k^2)}$.

\smallskip

Our main result for this learning problem is the following:

\begin{theorem}[PAC Learning $\mathcal{C}_{m, k}$] \label{thm:relus-density-informal}
There is a PAC learning algorithm for $\mathcal{C}_{m, k}$ with respect to $N(0, I)$
with the following performance guarantee:  Given $\eps>0$, and $O(m^2k^2/\eps^6) + (k/\eps)^{O(\log k)}$
noisy samples from an unknown $F \in \mathcal{C}_{m, k}$, the algorithm runs in time
$\poly(mk/\eps)+(k/\eps)^{O(\log^2 k)}$, and outputs a hypothesis $H: \R^m \to \R$
that with high probability satisfies $\|H-F\|_2^2 \leq \eps^2 (\|F\|_2^2+\sigma^2)$.
\end{theorem}

\noindent (See Theorem~\ref{thm:sum-relus} for a more detailed statement.)
Interestingly, our PAC learning algorithm is not proper. The hypothesis $H$ it outputs is a positive linear combination
of $\ell$ ReLUs, for some $\ell \gg k$.


Our algorithm establishing Theorem~\ref{thm:relus-density-informal} does not make crucial use of the assumption 
that the activation function is a ReLU. The only properties we require is that our activation function has 
bounded higher moments and non-vanishing even-degree Fourier coefficients. 
\new{We note that our algorithmic ideas can be extended to other activation functions
satisfying these properties (see Theorem~\ref{thm:glm}).}

\subsubsection{Learning Mixtures of Linear Regressions}

A $k$-mixture of linear regressions ($k$-MLR), specified by
mixing weights $w_i \geq 0$, where $\sum_{i=1}^k w_i=1$,
and regressors $\beta_i \in \R^m$, $i \in [k]$, is
the distribution $Z$ on pairs $(x, y) \in \R^m \times \R$, where $x \sim N(0, I)$
and with probability $w_i$ we have  that $y = \beta_i \cdot x + \nu$, where 
$\nu \sim N(0, \sigma^2)$ is independent of $x$.

We study both density estimation and parameter learning for $k$-MLRs.
For simplicity of the presentation, we will assume in this section that $\max_i \|\beta_i\|_2 \leq 1$
and that the mixing weights are uniform.

\paragraph{Prior Work on Learning Mixtures of Linear Regressions}
Mixtures of linear regressions are a natural probabilistic model introduced in~\cite{DV89, JJ94}
and have been extensively studied in machine learning. Prior work on this problem is quite
extensive. The reader is referred to Section~1.2 of~\cite{CLS20} for a detailed summary
of prior work on this problem. Here we focus on the prior work that is most closely related
to the results of this paper.

Most prior work on learning MLRs has focused on the parameter estimation problem.
A line of work (see, e.g.,~\cite{ZJD16, LL18, KC19} and references therein)
has focused on analyzing non-convex methods (including expectation maximization and alternating minimization).
These works establish local convergence guarantees: Given a sufficiently accurate solution (warm start),
these non-convex methods can efficiently boost this to a solution with arbitrarily high accuracy.
The focus of our algorithmic results in this section is to provide such a warm start. We note that the local convergence
result of~\cite{LL18} applies for the noiseless case, while the more recent result of~\cite{KC19} can handle non-trivial
regression noise when the weights of the unknown mixture are known.

The prior works most closely related to ours are~\cite{LL18, CLS20}. The work of \cite{LL18} focuses on the noiseless setting
($\sigma = 0$) and provides an algorithm with sample complexity and running time scaling exponentially with $k$.
The main bottleneck of their algorithm lies in a univariate parameter estimation step, which relies on the method
of moments and requires $k^{O(k)}$ samples and time. The recent work~\cite{CLS20} pointed out that the exponential
dependence on $k$ is inherent in this approach: One can construct a pair of $k$-MLRs whose moment tensors of
degree up to $\Omega(k)$ match, but their parameters are far from each other. \cite{CLS20} concludes that
``any moment-based estimator'' would therefore require runtime $\exp(\Omega(k))$. 
Our approach also uses moments, but exploits the underlying symmetry to circumvent this obstacle.

The fastest previously known algorithm for the parameter estimation problem of $k$-MLRs was given in~\cite{CLS20}.
This work circumvents the aforementioned exponential barrier by considering moments of carefully
chosen projections of the Fourier transform. Roughly speaking,~\cite{CLS20} gives algorithms whose sample complexity
and running time scales with $\exp(\tilde{O}(k^{1/2}))$. In more detail, for the noiseless ($\sigma = 0$)
and uniform weights case with separation $\Delta>0$, the algorithm of~\cite{CLS20} has sample complexity and runtime
of the form $\poly(m k /\Delta) \, (k \ln(1/\Delta))^{\tilde{O}(k^{1/2})}$. For the noisy case, when $\sigma = O(\eps)$ and the
weights are uniform, the algorithm of~\cite{CLS20} has sample complexity and runtime of the form
$\poly(m k /(\eps \Delta)) \, (k/\eps)^{\tilde{O}(k^{1/2}/\Delta^2)}$.

In summary, prior to this work, the best known learning algorithm for $k$-MLRs
had sample complexity and running time scaling exponentially with $k^{1/2}$~\cite{CLS20}.

We are now ready to state our results for this problem.
For density estimation, we show:


\begin{theorem}[Density Estimation for $k$-MLR]\label{thm:mlr-density-informal}
There is an algorithm that on input $\eps>0$, and
$N=\left(m^2\poly(k)+ k^{O(\log k)}\right) \tilde{O}(\log(1/\sigma)) +(k/\eps)^{O(\log^2 k)}$ samples
from an unknown $k$-MLR $Z$ on $\R^m \times \R$, it runs in $\poly(N)$ time and outputs a hypothesis distribution $H$ such
with high probability $\dtv(H, Z) \leq \eps$.
\end{theorem}

\noindent (See Theorem~\ref{thm:mlr-density} for a detailed statement handling general mixtures.)
To the best of our knowledge, this is the first algorithm for density estimation of $k$-MLRs
with running time sub-exponential in $k$.

For the parameter estimation problem, we provide two algorithmic results -- one for the noiseless case (corresponding to $\sigma = 0$)
and one for the noisy case (corresponding to $\sigma > 0$). We note that the $\sigma = 0$ case is already quite challenging, 
and most prior work (with provable guarantees) for the large $k$ regime focuses 
on this case (see, e.g.,~\cite{ZJD16, LL18, CLS20}).

For the noiseless case, we achieve exact recovery (see Theorem~\ref{thm:mlr-param-no-noise} for a more detailed statement):

\begin{theorem}[Parameter Estimation for $k$-MLR, Noiseless Case]\label{thm:mlr-param-no-noise-informal}
There is an algorithm that given $N=\left(m^2 \poly(k)+ k^{O(\log k)}\right) \tilde{O}(\log(k \log(m)/\Delta))$
samples from an unknown $k$-MLR $Z$ on $\R^m \times \R$ with uniform weights and
pairwise separation $\Delta = \min_{i \neq j} \|\beta_i - \beta_j\|_2>0$,
the algorithm runs in time $\poly(N, k^{\log^2 k})$, and outputs a list of hypothesis vectors
$\widetilde{\beta}_i$ such that with high probability we have that $\beta_i  = \widetilde{\beta}_{\pi(i)}$, $i \in [k]$,
for some permutation $\pi\in \mathbf{S}_k$.
\end{theorem}

Our second result can handle additive noise 
(see Theorem~\ref{thm:mlr-param-noise} for a more detailed statement).

\begin{theorem}[Parameter Estimation for $k$-MLR, Noisy Case]\label{thm:mlr-param-noise-informal}
There is an algorithm that on input $\eps>0$,
$N=\left(m^2\poly(k)+ k^{O(\log k)} \right) \tilde{O}(\log(k \log(m)/\Delta)) + \tilde{O}(m) \poly(k, 1/\eps)$
samples from an unknown $k$-MLR $Z$ with uniform weights and mean separation
$\Delta = \min_{i \neq j} \|\mu_i - \mu_j\|_2$ such that $\Delta/\sigma$ at least an appropriate polynomial in $k \log(m)$,
the algorithm runs in $\poly(N, k^{\log^2 k})$ time, and outputs a list of hypothesis vectors
$\widetilde{\beta}_i$ such that with high probability we have that $\|\beta_i - \widetilde{\beta}_{\pi(i)} \| \leq \eps$, $i \in [k]$,
for some permutation $\pi\in \mathbf{S}_k$.
\end{theorem}


\subsubsection{Learning Mixtures of Hyperplanes}
Our final learning application is for the problem of parameter estimation for mixtures of hyperplanes.
A $k$-mixture of hyperplanes is a distribution on $\R^m$ with density function
$F(x) = \sum_{j=1}^k w_j N(0, I - v_j v_j^T)$,
where $v_j \in \R^m$ with $\|v_j\|_2=1$ and
$w_j \geq 0$ with $\sum_{j=1}^k w_j = 1$.

We study parameter estimation for this probabilistic model under $\Delta$ pairwise separation for the $v_j$'s.
Specifically, we will assume that we know some $\Delta>0$ such that for all $i \neq j$
and $\sigma_i, \sigma_j \in \{ \pm 1\}$, we have that $ \| \sigma_i v_i- \sigma_j v_j\|_2 \geq \Delta$.
Note that the $v_j$'s are only identifiable up to sign, which motivates this definition.

For simplicity, we will assume uniform weights in this section, i.e., that all the $w_i$'s are $1/k$.
The goal of parameter learning in this context is to output a list of unit vectors $\{\tilde{v}_j\}_{j=1}^k$ such that
there is a permutation $\pi\in \mathbf{S}_k$ and a list of signs $\sigma_j \in \{ \pm 1\}$ for which
$v_j = \sigma_j \tilde{v}_{\pi(j)}$ for all $j \in [k]$.

\paragraph{Prior Work on Learning Mixtures of Hyperplanes}
Mixtures of hyperplanes is a natural probabilistic model that was recently studied in~\cite{CLS20}, motivated
by its connection to the subspace clustering problem (see, e.g.,~\cite{ParsonsHL04, Vidal11} for overviews).
In the subspace clustering problem, the data is assumed
to be drawn from a union of linear subspaces, and the algorithmic problem is to identify the hidden subspaces.
The mixtures of hyperplanes model can be viewed as a hard instance of subspace
clustering, but is also of interest in its own right as it arises in various contexts (see~\cite{CLS20} for a detailed discussion).

The fastest previously known algorithm for the parameter estimation problem of $k$-mixtures of hyperplanes
was given in~\cite{CLS20}. In more detail, for uniform weights and separation $\Delta>0$, the algorithm of~\cite{CLS20}
has sample complexity and runtime of the form $\poly(m k /\Delta) \, (k \ln(1/\Delta))^{\tilde{O}(k^{3/5})}$.

\smallskip

Our main result in this section is the following theorem (see Theorem~\ref{thm:mh-param}):

\begin{theorem}[Parameter Learning for $k$-mixtures of Hyperplanes]\label{thm:mh-param-informal}
There is an algorithm that on input $N = O(k/\Delta)^{O(\log k)}+O(m^2) \poly(k\log(m)/\Delta)$
samples from a uniform $k$-mixture of hyperplanes on $\R^m$ with pairwise separation $\Delta>0$,
the algorithm runs in time $\poly(N)+ m^2 \log(\log(m)/\Delta) k^{O(\log^2 k)}$ and
with high probability outputs an $\eps$-approximation to the unknown parameter vectors.
\end{theorem}

\subsection{Organization}

The structure of the paper is as follows:
In Section~\ref{sec:prelims}, we provide the necessary definitions and technical facts.
In Section~\ref{sec:small-cover}, we prove our main geometric result.
Sections~\ref{sec:strategy} describes how our main geometric result is used for our learning theory applications.
The next sections present our learning algorithms in order: mixtures of spherical Gaussians (Section~\ref{sec:gmm}), 
positive linear combinations of ReLUs (Section~\ref{sec:relu}) and GLMs (Section~\ref{sec:glm}), mixtures of linear regressions (Section~\ref{sec:mlr}), and mixtures of hyperplanes (Section~\ref{sec:mh}). Some omitted proofs are deferred to an Appendix.

\newpage

\section{Preliminaries} \label{sec:prelims}

\paragraph{Basic Notation and Definitions.}
For $n \in \Z_+$, we denote by $[n]$ the set $\{1, 2, \ldots, n\}$.
For a vector $v \in \R^n$, let $\| v \|_2$ denote its Euclidean norm.
\new{We denote by $x \cdot y$ the standard inner product between $x, y \in \R^n$.}


\new{For $a, b \in \R$, we will write $a \gg b$ (or $b \ll a$) to mean that there exists a sufficiently large
constant $C>0$ such that $a \geq C b$. We will denote by $\delta_0$ the Dirac delta function 
and by $\delta_{i, j}$ the Kronecker delta.}

For $x \in \R^n$ and $r >0$, let $B_n(x, r) = \{z \in \R^n: \|z-x\|_2 \leq r \}.$
Let $S \subset \R^n$ and $\eps>0$. We say that a set $C \subset \R^n$ is an {\em $\eps$-cover of $S$}
(with respect to the $\ell_2$-norm) if for all $x \in S$ there exists $c_x \in C$ such that $x \in B_n(c_x, \eps)$.

Throughout the paper, we let $\otimes$ denote the tensor/Kronecker product.
\new{For a vector $x \in \R^n$, we denote by $x^{\otimes d}$ the $d$-th order tensor product of $x$.}

We will denote by $N(\mu, \Sigma)$ the multivariate Gaussian distribution with mean $\mu$ and covariance $\Sigma$.
The underlying dimension will be clear from the context. For a random variable $X$ and $p \geq 1$, 
we will use $\|X\|_p$ to denote its {\em $L_p$-norm}, i.e., $\|X\|_p \eqdef \E[|X|^p]^{1/p}$, assuming the RHS is finite.

The \emph{total variation distance} between probability distributions  $P$ and $Q$ on $\mathbb{R}^m$, 
denoted $\dtv(P, Q)$, is defined as $\dtv (P, Q) \eqdef \sup_{A \subseteq \R^m} |P(A) - Q(A)|$.

\subsection{Tools from Linear Algebra} \label{ssec:prelims-la}


If $V$ is a subspace of a finite dimensional vector space $W$, then the {\em codimension of $V$ in $W$}
is the difference $\codim_W(V) = \dim(W) - \dim(V)$.
We will make essential use of the following basic fact (see Appendix~\ref{ssec:codim} for the simple proof):


\begin{fact} \label{fact:codim-bound}
Let $U, V, W$ be finite dimensional vector spaces with $U \subseteq W$. Then $\codim_U(V \cap U) \leq \codim_W(V)$.
\end{fact}

\paragraph{Polynomials and Tensors.}
Let $\R[x_1, \ldots, x_n]$ be the vector space of real polynomials in variables
$x_1, \ldots, x_n$. \new{If $x = (x_1, \ldots, x_n)$ is a vector of indeterminates, we will 
sometimes use the notation $\R[x]$.}
A real polynomial in $n$ variables is called {\em homogeneous degree-$d$}
if it only contains monomials of degree exactly $d$.
Let $\R_{[d]}[x_1, \ldots, x_n]$ denote the vector space of real homogeneous degree-$d$
polynomials in variables $x_1, \ldots, x_n$.

A tensor $A$ of dimension $n$ and order $d$ is a multilinear map
defined by a $d$-dimensional array with real entries $A_{\alpha}$,
where $\alpha = (\alpha_1, \ldots, \alpha_d)$ with $\alpha_i \in [n]$.
A tensor $A$ is symmetric if $A_{\alpha} = A_{\alpha_{\sigma}}$,
where $\alpha_{\sigma}  = (\alpha_{\sigma_1}, \ldots, \alpha_{\sigma_d})$,
for any permutation $\sigma: [d] \to [d]$.
\new{For tensors $A, B$ of dimension $n$ and order $d$,
we will denote by $\langle A, B \rangle$ their entry-wise inner product.
For a tensor $A$, let $\|A\|_2 = \langle A, A \rangle^{1/2}$ denote the $\ell_2$-norm of its entries.
}

Recall that there is a bijection between the space of real symmetric tensors of dimension $n$ and order $d$
and the space of real homogeneous degree-$d$ polynomials in $n$ variables.
\new{The inner product of two real homogeneous degree-$d$ polynomials $p, q \in \R_{[d]}[x_1, \ldots, x_n]$,
denoted by $\langle p , q \rangle$, is defined to be the inner product of their corresponding symmetric tensors, i.e.,
if $p(x) = \langle A_p , x^{\otimes d} \rangle$ and $q(x) = \langle A_q , x^{\otimes d} \rangle$ then
$\langle p , q \rangle = \langle A_p, A_q \rangle$.
Consequently, the $\ell_2$-norm of a homogeneous polynomial $p \in \R_{[d]}[x_1, \ldots, x_n]$
is the $\ell_2$-norm of the corresponding symmetric tensor,
i.e., if $p(x) = \langle A_p , x^{\otimes d} \rangle$, then $\|p\|_{\ell_2} := \|A_p\|_2$.
By the multi-linearity property of tensors, the $\ell_2$-norm of a homogeneous
polynomial is rotationally invariant.}


\new{
An $n$-dimensional multi-index $\alpha$ is an $n$-tuple of non-negative integers,
i.e., $\alpha = (\alpha_1, \ldots, \alpha_n) \in \N_0^n$. We will define
the length of the multi-index $\alpha$ as $|\alpha| = \sum_{i=1}^n \alpha_i$.
We will also denote $\alpha! = \prod_{i=1}^n \alpha_{i} !$ and use $c(\alpha) = |\alpha|! / \alpha!$
for the multinomial coefficient.
For a vector of indeterminates $x = (x_1, \ldots, x_n)$,
we will denote the monomial corresponding to the multi-index $\alpha$ by
$x^{\alpha} = \prod_{i=1}^n x_i^{\alpha_i}$.

With this notation, we have the following fact:

\begin{fact} \label{fact:ip-monomials}
For any multi-indices $\alpha, \beta \in  \N_0^n$, and $x = (x_1, \ldots, x_n)$ a vector of indeterminates,
we have that
\begin{eqnarray*}
\langle x^{\alpha}, x^{\beta} \rangle = \left\{
\begin{array}{ll}
      0 \;,& \alpha \neq \beta \\
      \frac{1}{c(\alpha)} = \frac{\alpha!}{|\alpha|!} \;,& \alpha = \beta\\
\end{array}
\right.
\end{eqnarray*}
\end{fact}
}

Our proofs will repeatedly use the following simple lemma (see Appendix~\ref{ssec:poly-bounds} for the simple proof):

\begin{lemma} \label{lem:poly-bounds}
For any $p \in \R_{[d]}[x_1, \ldots, x_n]$ and $x, y \in \R^n$, we have that:
\begin{itemize}
\item[(i)] $|p(x)| \leq \|x\|_2^d \; \|p\|_{\ell_2}$, and
\item[(ii)] $|p(x)-p(y)| \leq d \max\{ \|x\|_2 , \|y\|_2\}^{d-1} \, \|x-y\|_2 \, \|p\|_{\ell_2}.$
\end{itemize}
\end{lemma}

\subsection{Tools from Probability} \label{ssec:prelims-prob}

\paragraph{Concentration and Anti-concentration for Gaussian Polynomials.}
We will require standard concentration and anti-concentration bounds for multivariate 
degree-$d$ polynomials under the standard Gaussian distribution.
For a polynomial $p: \R^n \to \R$, we consider the random variable $p(x)$, 
where $x \sim N(0, I)$. We will use $\|p\|_r$, for $r \geq 1$, to denote the $L_r$-norm
of the random variable $p(x)$, i.e., $\|p\|_r \eqdef \E_{x \sim N(0, I)} [|p(x)|^r]^{1/r}$.

We first recall the following moment bound for low-degree
polynomials, which is equivalent to the well-known hypercontractive
inequality of \cite{Bon70,Gross:75}:
\begin{theorem} \label{thm:hc}
Let $p: \R^n \to \R$ be a degree-$d$ polynomial and $q>2$.  Then 
$\|p\|_q \leq (q-1)^{d/2} \|p\|_2$.
\end{theorem}

The following concentration bound for low-degree polynomials, a
simple corollary of hypercontractivity, is well known (see, e.g.,~\cite{ODonnellbook}):
\begin{theorem}\label{thm:deg-d-chernoff}
Let $p: \R^m \to \R$ be a degree-$d$ polynomial. For any $t>e^d$,
we have
$$\pr_{x \sim N(0, I)} \left[ |p(x)|\geq t \|p\|_2 \right]\leq \exp(-\Omega(t^{2/d})).$$
\end{theorem}

We will also require the following anti-concentration bound for Gaussian polynomials:

\begin{theorem}[\cite{CW01}] \label{thm:cw}
Let $p: \R^n \to \R$ be a nonzero real degree-$d$ polynomial. For all $\eps>0$ and $t\in \R$ we have
\[ \pr_{x \sim N(0, I)}\left[|p(x)-t|\leq\eps \cdot
\sqrt{\Var_{x \sim N(0, I)}[p(x)]} \right]\leq O(d\eps^{1/d}) .\]
\end{theorem}

Additionally, we will require basic facts KL Divergence (Pinsker's inequality), 
a classical result from empirical process theory (VC Inequality), and basics on Hermite analysis.
These tools are reviewed in Appendix~\ref{app:prob}.

\newpage

\section{Main Geometric Result} \label{sec:small-cover}
In Section~\ref{ssec:small-cover-exists}, we show the existence 
of small covers for near-zero sets of polynomials. In Section~\ref{ssec:cover-alg},
we point out how to turn our existence proof into an efficient algorithm to compute such a 
small cover.

\subsection{Existence of Small Covers} \label{ssec:small-cover-exists}

Our main geometric result is the following theorem:


\begin{theorem} \label{thm:main-cover}
There exists a universal constant $C>0$ such that the following holds:
Let $d, k, m \in \Z_+$ and $V$ be any vector space of homogeneous degree-$d$ real polynomials
in $m$ variables with codimension at most $k$ within $\R_{[d]}[x_1, \ldots, x_m]$.
For $\delta, R > 0$, let
$$S = S(V, R, \delta) \eqdef \{ x \in \R^m: \|x\|_2 \leq R \textrm{ and } |p(x)| \leq \delta \|p\|_{\ell_2} \textrm{ for all } p\in V \} \;.$$
Then, for any \new{$\eps \geq \delta^{\frac{1}{(C+1)d}} (2 R d k m)^{\frac{C}{C+1}}$},
there exists an $\eps$-cover of $S$ with size at most \new{$(2 (R/\eps) d k m)^{C^2 d^2 k^{1/d}}$.}
\end{theorem}

\new{
\paragraph{Detailed Proof Overview.}
The proof of Theorem~\ref{thm:main-cover} is elementary, though highly technical. 
Before we give the formal proof, we start by explaining the
main ideas here.
Fundamentally, the proof is recursive, and we show that given $V$,
$\delta, \eps$, and $R$, we can find an appropriate cover of the
corresponding set $B$ by reducing to a number of smaller and similar
looking problems. The first step in this reduction involves writing
$\R^m$ as $\R^{m'} \times \R^{m-m'}$ for $m'$ a sufficient multiple of $k^{1/d}$. We
then cover $B(0,R)$ by a number of cylinders of the form $B(x,\eps') \times
\R^{m-m'}$ ($\eps'$ a carefully chosen constant on the order of $\delta$). Our
basic plan is to show that for most $x$ that there is a small cover of
the intersection of $B$ with the associated cylinder, and then to show
that the bad $x$ all lie close to a hyperplane of co=dimension at least
$m'/2$.

For each cylinder, we note that for $x' \in B(x,\eps')$ and $y \in \R^{m-m'}$
not too large that $p(x',y)$ is close to $p(x,y)$ for all $p$. This means
that in order to cover $B(x,\eps') \times \R^{m-m'}$, it suffices to find a
cover of just $\{x\} \times \R^{m-m'}$ with slightly stronger parameters. The
set that needs to be covered is the set of points that nearly vanish
on every polynomial in $V$, where $V$ is a set of degree-$d$ polynomials in
$m$ variables. We restrict our attention to those polynomials $p \in V$
that when restricted to $x$ in their first $m'$ coordinates leave a degree
$d-1$ polynomial in $m-m'$ variables. We note that for any such polynomial
$p$, if substituting $x$ into its first $m'$ coordinates does not decrease
its $L_2$ norm by too much, the resulting polynomial $q(y) := p(x,y)$ must
nearly vanish on every point of $B \cap \{x\} \times \R^{m-m'}$. One way of
formalizing this is as follows. We let $W$ be the subspace of $V$
consisting of polynomials that are homogeneous degree-$1$ in the
$x$-coordinates. We define a linear transformation $A_x$ mapping $W$ to
degree-$(d-1)$ polynomials in the $y$-coordinates by evaluation on the $x$
coordinates. We note that all points in $B \cap \{x\} \times \R^{m-m'}$ nearly
vanish on all polynomials in $U$, where $U$ is the span of the eigenvectors
of $A_x$ with not-too-small eigenvalues. If the number of such
eigenvectors is large, then we can recursively find a small cover of 
$B \cap \{x\} \times \R^{m-m'}$. In particular, if the number of small eigenvectors
is less than $k' = k^{(d-1)/d}$, the recursive bounds should suffice. We
call such $x$ good. We will need a different technique for finding a
cover of the bad points.

For this, we claim that there is a hyperplane H of codimension at
least $m'/2$ so that all of the bad $x$ lie close to $H$. If we can show
this, we can cover all of the bad cylinders by recursively finding a
cover of $H$ (considering only the polynomials in the variables along
$H$). To prove this statement, we proceed by contradiction. If no $H$ can
be found, there must be many bad $x_i$ so that $x_{i+1}$ is far from the
span of $\{x_1,\ldots,x_i\}$. To each $x_i$ we associate degree-$(d-1)$
polynomials $p_{i1},\ldots,p_{ik'}$ corresponding to the small eigenvectors
of $A_{x_i}$. We let $q_i(x)$ be the linear function $q_i(x) = x \cdot x_i$,
and consider the set $S$ of polynomials $q_i(x) p_{ij}(y)$. It is not hard
to see that each polynomial in $S$ has a large component orthogonal to
the previous elements, and thus their Gram matrix must have a
relatively large determinant. On the other hand, $|S| \gg m'  k' \gg k$, so
there must be many linear combinations of polynomials in $S$ that lie in
$W$. It is also not hard to see that any polynomial in $W$ will have
relatively small inner product with any polynomial in $S$, and this will
imply that the Gram matrix of $S$ has relatively small determinant,
yielding a contradiction with our previous bound.
}

We are now ready to proceed with the formal proof of Theorem~\ref{thm:main-cover}.
We start with the following definition:

\begin{definition} \label{def:f}
In the context of Theorem~\ref{thm:main-cover}, let $f(R,d,\eps,\delta,k,m)$ be the smallest integer such that
for any subspace $V$ of $m$-variable homogeneous degree-$d$ real polynomials
with codimension at most $k$ in $\R_{[d]}[x_1, \ldots, x_m]$, the set
$S = S(V, R, \delta)  \eqdef \left\{ x \in \R^m: \|x\|_2 \leq R \textrm{ and }
|p(x)| \leq \delta \|p\|_{\ell_2} \textrm{ for all } p\in V \right\}$
has an $\eps$-cover, in $\ell_2$-norm, of cardinality at most $f(R, d, \eps,\delta,k,m)$.
\end{definition}

The proof of Theorem~\ref{thm:main-cover} relies on the following crucial proposition:

\begin{proposition}\label{prop:recursion}
For any \new{$0< \eps' < \eps \leq R$}, 
$k, m, k', m' \in \Z_+$ with $k' < k$ and $m' <m$,
and $\eta>0$ with $\eps-\eps' \gg \new{k^{1/2} \, \eta^{1/4} R^{3/4} d^{-1/8}}$,
we have that
\begin{eqnarray*}
f(R,d,\eps,\delta,k,m) &\leq& f(R,d,\eps',\delta,k,m') f\left(R,d-1,\eps-\eps', \big(\delta+\new{d \, (2R)^{d-1} \eps'}\big)/\eta, k', m-m' \right) \\
                                  &+& f\left(R,d,\eps-\eps' -\new{O\big(k^{1/2} \eta^{1/4}  R^{3/4} d^{-1/8}\big)}, \delta, k, m-m'+2(k/k')\right) \;.
\end{eqnarray*}
\end{proposition}

\begin{proof}[Proof of Proposition~\ref{prop:recursion}]
\dnew{The basic strategy of our proof will be as follows. 
Firstly, it is straightforward to show that the projection of $S$ onto the first $m'$-coordinates 
has an $\eps'$ cover of size at most $f(R,d,\eps',\delta,k,m')$. For each of the points $c$ in this cover, 
we will need to cover the cylinder $(B(c,\eps')\times \R^{m-m'})\cap S$. 
We show that for most such $c$ there is such a cover of size at most 
$f\left(R,d-1,\eps-\eps', \big(\delta+\new{d \, (2R)^{d-1} \eps'}\big)/\eta, k', m-m' \right)$. 
For the remaining points of $S$, we show that they all lie close to a hyperplane $H$ 
of dimension at most $m-m'+2k/k'$. By considering the projection of these points onto $H$, 
we are left with a similar problem in a smaller dimensional space, 
and show that all of these remaining points have a cover of size at most 
$f\left(R,d,\eps-\eps' -\new{O\big(k^{1/2} \eta^{1/4}  R^{3/4} d^{-1/8}\big)}, \delta, k, m-m'+2(k/k')\right)$.}

We begin by decomposing $\R^m$ as $\R^{m'}\times \R^{m-m'}$.
Each element $z \in \R^m$ can be written as $z = (x,y)$ with $x\in \R^{m'}$ and $y\in \R^{m-m'}$.
Let $V_x$ be the subset of $V$ that consists of homogeneous degree-$d$ polynomials
that only depend on coordinates in $x$.

\dnew{To cover the projection onto the $x$-coordinates, we note that $|p(x,y)|=|p(x)|$ 
must be small for all $(x,y)\in S$ and $p\in V_x$. This means that these points 
are a set of near zeroes of a large space of polynomials. More formally,} 
by Fact~\ref{fact:codim-bound}, $V_x$ is a subspace with codimension at most $k$
in $\R_{[d]}[x_1, \ldots, x_{m'}]$, i.e., within the space of all homogeneous degree-$d$ polynomials
in these variables.
Consider the set of $S_x \subset \R^{m'}$ defined as follows:
$$S_x \eqdef S_x(V_x, R, \delta) = \{ x \in \R^{m'}: \|x\|_2 \leq R \textrm{ and }
|p(x)| \leq \delta \|p\|_{\ell_2} \textrm{ for all } p\in V_x \} \;.$$
First, we claim that $S_x \supseteq \Pi_x(S) \eqdef \{ x \in \R^{m'}\mid \exists y \in \R^{m-m'} \textrm{ with } (x, y) \in S\}$,
i.e., $S_x$ contains the projection of $S$ onto the $x$-coordinates. Indeed, let $x \in \Pi_x(S)$. Then there exists
$y \in  \R^{m-m'}$ such that (a) $\|(x, y) \|_2 \leq R$ and (b) $|p(x, y)| \leq \delta \|p\|_{\ell_2}$ for all $p \in V$.
Condition (a) a fortiori implies that $\|x\|_2 \leq R$.
Since $V_x \subseteq V$ and $p(x, y) = p(x)$ for all $p \in V_x$,
condition (b) gives that $|p(x)| \leq \delta \|p\|_{\ell_2}$ for all $p \in V_x$.
Therefore, $x \in S_x$.

Let $\mathcal{C}_{\eps'}$ be an $\eps'$-cover of $S_x$ (and therefore of $\Pi_x(S)$) with minimum cardinality.
Since $V_x$ is a subspace with codimension at most $k$ in $\R_{[d]}[x_1, \ldots, x_{m'}]$,
\new{by Definition~\ref{def:f}} we have that $|\mathcal{C}_{\eps'}| \leq f(R, d, \eps',\delta, k, m')$.

\dnew{For each $c\in\mathcal{C}_{\eps'}$, we would like to cover the cylinder $(B(c,\eps')\times \R^{m-m'})\cap S$. We note that this is a set of points where $|p(x,y)|$ is small for all $p\in V$. However, for $x$ in this set $p(x,y)\approx p(c,y)$. Thus, it suffices to consider the set of points where $|p(c,y)|$ is small, which will allow us to reduce more easily to a similar-looking problem. In particular,} for each $c\in\mathcal{C}_{\eps'}$, we consider the set
$$\new{S_{y, c}} \eqdef \{y \in \R^{m-m'}: \|y\|_2 \leq R \textrm{ and }
|p(c,y)| \leq (\delta +  \new{d \, (2R)^{d-1} \eps'}) \|p\|_{\ell_2} \textrm{ for all } p \in V  \} \;.$$
We require the following claim:

\begin{claim} \label{claim:triangle-ineq-cover}
For any given $c \in \mathcal{C}_{\eps'}$, let $\mathcal{D}_{c, \eps-\eps'}$
be an $(\eps-\eps')$-cover of $\new{S_{y, c}}$.
The set $c\times \mathcal{D}_{c, \eps-\eps'}$ is an $\eps$-cover for $S\cap (B_{\new{m'}}(c,\eps')\times \R^{m-m'})$.
\end{claim}

\begin{proof}
To prove this claim, we start by noting that
for any $(x,y) \in S\cap (B_{\new{m'}}(c,\eps')\times \R^{m-m'})$,
where $x \in \R^{m'}$ and $y \in \R^{m-m'}$, we have that
(a) $\|x - c\|_2 \leq \eps'$ and (b) $y \in \new{S_{y, c}}$.
Condition (a) follows directly from the fact that $x \in B_{\new{m'}}(c,\eps')$.
To show condition (b), we start by noting that $\|y\|_2 \leq \|(x, y)\|_2 \leq R$,
where the second inequality holds since $(x, y) \in S$. Moreover, for each $p\in V$
we have that 
\begin{eqnarray*}
|p(c,y)|
&\leq& |p(x,y)|+ |p(c,y) - p(x,y)| \\
&\leq& \delta \|p\|_{\ell_2}+ \new{d \, \max\{\|x\|_2, \|c\|_2\}^{d-1}} \|c-x\|_2 \, \|p\|_{\ell_2} \\
&\leq& \delta \|p\|_{\ell_2}  + \new{d \, (2R)^{d-1}}  \eps' \, \|p\|_{\ell_2} \\
&=& (\delta+ \new{d \, (2R)^{d-1}} \, \eps' ) \, \|p\|_{\ell_2} \;,
\end{eqnarray*}
where the first inequality is the triangle inequality,
the second inequality uses that $|p(x ,y)| \leq \delta \|p\|_2$ (since $(x, y) \in S$)
and Lemma~\ref{lem:poly-bounds} \new{(ii)},
the third inequality uses that $\|x\|_2 \leq \| (x,y) \|_2 \leq R$ (since $(x, y)\in S$) and
$\|c\|_2 \leq \|c-x\|_2 + \|x\|_2 \leq \eps'+R \leq 2R$.

Since $y \in \new{S_{y, c}}$ and $\mathcal{D}_{c, \eps-\eps'}$ is an $(\eps-\eps')$-cover of $\new{S_{y, c}}$,
there exists $d\in \mathcal{D}_{c, \eps-\eps'}$ with $\|d-y\|_2 \leq \eps-\eps'$. Therefore,
for any $(x,y) \in S\cap (B_{\new{m'}}(c,\eps')\times \R^{m-m'})$ we have that
$\|x - c\|_2 \leq \eps'$ and $\|d-y\|_2 \leq \eps-\eps'$ for some $d\in \mathcal{D}_{c, \eps-\eps'}$.
The claim now follows from the Pythagorean theorem.
\end{proof}

Note that the set $\cup_{c \in \mathcal{C}_{\eps'}}(c\times \mathcal{D}_{c, \eps-\eps'})$
is an $\eps$-cover of $S$. To see this, fix any $z = (x, y) \in S$, where $x \in \R^{m'}$ and $y \in \R^{m-m'}$.
Since $x \in \Pi_x(S) \subseteq S_x$, there exists
$c \in \mathcal{C}_{\eps'}$ such that $\|x-c\|_2 \leq \eps'$. For this choice of $c$, by definition we have that
$z \in S\cap (B_{\new{m'}}(c,\eps')\times \R^{m-m'})$. By Claim~\ref{claim:triangle-ineq-cover}, there exists a point
$z' \in c \times \mathcal{D}_{c, \eps-\eps'}$ such that $\|z-z'\|_2 \leq \eps$, as desired.

Therefore, if we could show that the set $\new{S_{y, c}}$ has a small a $(\eps-\eps')$-cover
for all $c\in \mathcal{C}_{\eps'}$, we would obtain a small $\eps$-cover of $S$. While this strong statement
may not hold, we will show that $\new{S_{y, c}}$ has a small $(\eps-\eps')$-cover
for \emph{most} points $c\in \mathcal{C}_{\eps'}$. In particular, we will show that
$S_{y,c}$ has a small $(\eps-\eps')$ cover for all points $c\in \mathcal{C}_{\eps'}$
\emph{except} for those near a low-dimensional subspace $H$.
We will then separately show how to construct a small cover of the points near $H \times \R^{m-m'}$.

\dnew{To understand why we will have a small cover for most $S_{y,c}$, 
note that, for $y\in S_{y,c}$ and $p\in V$, $|p(c,y)|$ will be small. 
We note that a basis of $p\in V$ consists of $\dim(\R_{[d]}(z_1,\ldots,z_m))-k$ polynomials, 
and so we have a very large number of polynomials $p(c,-)$ that must be small at all $y\in S_{y,c}$, 
especially considering that $\dim(\R_{[d]}(z_1,\ldots,z_m))$ is much bigger than $\dim(\R_{[d]}(y_1,\ldots,y_{m-m'}))$. 
However, this intuition will be wrong if for many of these polynomials $p$ it is the case that 
$\|p(c,-)\|_{\ell_2} \ll \|p\|_{\ell_2}$. So, to see when this works and when it does not, 
we will need to consider when this kind of restriction leaves us 
with a polynomial of reasonable size.}

\new{To proceed, we require a few additional definitions.}

\begin{definition} \label{def:W}
Let $W$ be the subspace of $V$ consisting of polynomials in $(x, y)$, where $x \in \R^{m'}$ and $y \in \R^{m-m'}$,
that are homogeneous degree-$1$ in the $x$-variables
and homogeneous degree-$(d-1)$ in the $y$-variables.
\end{definition}

By Fact~\ref{fact:codim-bound}, $W$ is a subspace of the space of all polynomials
that are homogeneous degree-$1$ in the $x$-variables
and homogeneous degree-$(d-1)$ in the $y$-variables with codimension at most $k$.

Note that for each $x\in \R^{m'}$, $x$ defines a linear map $A_x: W \to \R_{[d-1]}[y]$ from $W$
to the vector space of homogeneous degree-$(d-1)$ polynomials in the $y$ coordinates.
In particular, for a polynomial $p\in W$, we define $A_x(p)$ by evaluation as $(A_x(p))(y) = p(x,y)$.

With this setup, we introduce the notion of a good point:

\begin{definition} \label{def:good}
\new{Fix $k' \in \Z_+$ and $\eta>0$.}
A point $x \in \R^{m'}$ is called {\em $(k', \eta)$-good} if the linear map $A_x$
has at most $k'$ left singular values smaller than $\eta$
(or equivalently if $A_x$ has at least $\dim(\R_{[d-1]}[y])-k'$ singular values that are at least $\eta$).
A point  $x \in \R^{m'}$ is called {\em $(k', \eta)$-bad} otherwise.
\end{definition}


Our next key claim is that for any good point $c\in  \mathcal{C}_{\eps'}$,
the set $\new{S_{y, c}}$ has a small cover:

\begin{claim}\label{claim:good-points-cover}
For any $(k', \eta)$-good point $c\in  \mathcal{C}_{\eps'}$, $\new{S_{y, c}}$ has an $(\eps-\eps')$-cover of size
$$f(R,d-1,\eps-\eps',(\delta+\new{d \, (2R)^{d-1} \eps'})/\eta,k',m-m') \;.$$
\end{claim}
\begin{proof}
\dnew{The idea of this proof is as described above. 
For $c$ a good point, we note that for $p\in W$ any singular vector of $A_c$ with large singular value, 
we have that for any $y\in S_{y,c}$ it holds
$$
|p(c,y)| \leq (\delta+d(2R)^{d-1}\eps') \|p\|_{\ell_2} \leq (\delta+d(2R)^{d-1}\eps')/\eta \|p(c,-)\|_{\ell_2} \;.
$$
This gives a large dimensional subspace of polynomials that 
nearly vanish on $S_{y,c}$ allowing us to bound the size of its cover.}

In particular, fix any $(k', \eta)$-good point $c\in \mathcal{C}_{\eps'}$.
Let $V_c$ be the vector space of homogeneous degree-$(d-1)$ polynomials
in the $y$ coordinates (i.e., in $\R^{m-m'}$) spanned by
the left singular vectors of $A_c$ with singular value more than $\eta$.
By \new{Definition~\ref{def:good}}, $V_c$ has codimension at most $k'$
within $\R_{[d-1]}[y]$.
Furthermore, for any $p\in V_c$ there exists
$q\in W \new{\subseteq V}$ with $A_c q = p$ and $\|p\|_{\ell_2} \geq \eta \|q\|_{\ell_2}$,
where the last inequality follows from the definition of $V_c$.
\dnew{In particular, the singular value decomposition gives us
orthonormal sets of polynomials $\{v^{(i)}\} \in W$ and $\{u^{(i)}\} \in \R_{[d-1]}[y]$ such
that $A_c v^{(i)} = \sigma_i u^{(i)}$. By definition, $V_c$ is spanned by the $u^{(i)}$'s
with corresponding $\sigma_i \geq \eta$. In particular, we can write $p = \sum_i a_i u^{(i)}$
and we have that $\|p\|_{\ell_2}^2 = \sum_i a_i^2$. Note that if we consider the polynomial
$q = \sum_i a_i \sigma_i^{-1} v^{(i)}$, then we have that $q \in W$, $A_c q = p$,
and $\|q\|_{\ell_2}^2 = \sum_i a_i^2 \sigma_i^{-2}$.
Since $\sigma_i \geq \eta$ for all $i$ with $a_i\neq 0$,
we get that $\|q\|_{\ell_2}^2 \leq \eta^{-2}\|p\|_{\ell_2}^2$ or $\|p\|_{\ell_2} \geq \eta \|q\|_{\ell_2}$.}

For any $p\in V_c$ and $y\in \new{S_{y, c}}$, we have that
$$|p(y)| = |q(c,y)| \leq (\delta + \new{d \, (2R)^{d-1} \eps'})\|q\|_{\ell_2}
\leq (\delta + \new{d \, (2R)^{d-1} \eps'})\|p\|_{\ell_2}/\eta \;,$$
where the equality follows from the definition of $p$ and $q$, the first inequality uses the definition of
$\new{S_{y, c}}$, and the second inequality uses that $\|q\|_{\ell_2} \leq \|p\|_{\ell_2}/ \eta$.
Therefore, $\new{S_{y, c}}$ is contained in the set
$$S'_{y, c} \eqdef \left\{ y \in \R^{m-m'}: \|y\|_2 \leq R \textrm{ and }
|p(y)|\leq (\delta + \new{d \, (2R)^{d-1} \eps'})\|p\|_{\ell_2}/\eta \textrm{ for all } p\in V_c  \right\} \;.$$
Since $V_c$ is a vector space of codimension at most $k'$ within $\R_{[d-1]}[y]$,
we have that \new{$S'_{y, c}$} (and therefore $\new{S_{y, c}}$) has an $(\eps-\eps')$-cover of size
$f(R,d-1,\eps-\eps',(\delta+\new{d \, (2R)^{d-1} \eps'})/\eta,k',m-m')$,
as desired. This completes the proof of Claim~\ref{claim:good-points-cover}.
\end{proof}

By Claims~\ref{claim:triangle-ineq-cover} and~\ref{claim:good-points-cover}, the subset of $S$
consisting of points $z = (x, y)$ whose $x$-coordinate is within
$\ell_2$-distance $\eps'$ of a $(k', \eta)$-good point $c \in \mathcal{C}_{\eps'}$
has an $\eps$-cover of size at most
$$f(R,d,\eps',\delta,k,m')f\left(R,d-1,\eps-\eps', \big(\delta+\new{d \, (2R)^{d-1} \eps'} \big)/\eta,k',m-m' \right) \;.$$


To complete the proof of Proposition~\ref{prop:recursion},
we proceed to establish an upper bound on the size of an $\eps$-cover for the subset of $S$
consisting of points $z=(x, y)$ whose $x$-coordinate is within $\ell_2$-distance $\eps'$
of a $(k', \eta)$-bad point $c \in \mathcal{C}_{\eps'}$.
To that end, we prove the following key lemma:

\begin{lemma}\label{lem:hyperplane-close-bad}
There exists a \new{subspace} $H$ in \new{$\R^{m'}$} of dimension at most $2k/k'$
so that all the $(k', \eta)$-bad points in $\mathcal{C}_{\eps'}$
are within $\ell_2$-distance \new{$O(\eta^{1/4} k^{1/2} R^{3/4} d^{-1/8})$} of $H$.
\end{lemma}
\begin{proof}
We proceed by contradiction. Let $t \eqdef \lceil 2k/k'\rceil$.
If the lemma statement does not hold, there exists a sequence
of $(k', \eta)$-bad points $x^{(1)},x^{(2)},\ldots,x^{(t)}$ in $\mathcal{C}_{\eps'}$,
where each $x^{(j)}$ has a component orthogonal to the span of $x^{(1)},\ldots,x^{(j-1)}$ of $\ell_2$-norm
at least $\gamma$,where $\gamma$ is a sufficiently large constant multiple of  \new{$\eta^{1/4} k^{1/2} R^{3/4} d^{-1/8}$}.
This sequence of points can be constructed inductively as, by assumption, not all of the $(k', \eta)$-bad points
in $\mathcal{C}_{\eps'}$ are within $\ell_2$-distance $\gamma$ of the hyperplane
spanned by $x^{(1)},\ldots,x^{(j-1)}$, allowing us to find an appropriate $x^{(j)}$.

By definition, each $(k', \eta)$-bad point $x^{(i)} \in \mathcal{C}_{\eps'}$ in the aforementioned sequence
has associated with it at least $k'$ orthogonal homogeneous degree-$(d-1)$ polynomials in the $y$-variables
corresponding to left singular vectors of $A_{x^{(i)}}$ with singular value at most $\eta$.
Let $p_{i,1},\ldots,p_{i,k'}$ be a set of $k'$ orthonormal such polynomials,
i.e., assume w.l.o.g. that the $p_{i, j}$'s  \dnew{are orthogonal and} satisfy $\|p_{i,j}\|_{\ell_2}=1$.
For each $x^{(i)}$, $i \in [t]$, we also consider the linear polynomial $q_i(x):= x \cdot x^{(i)}$.
Let $B_{i,j}(x,y)$ be the polynomial in $\R^m$ defined as $B_{i,j}(x,y) = q_i(x)p_{i,j}(y)$
for $i \in [t]$ and $j \in [k']$. \new{Let $U$ be the set of all $B_{i,j}$'s, i.e., $|U| = t \, k' \geq 2k$.}

We will require the following claim (whose simple proof is in Appendix~\ref{app:inner-prod-d}):
\begin{claim}\label{claim:inner-prod-d}
Let $q_1(x), q_2(x)$ be homogeneous degree-$1$ polynomials in $x$ and $p_1(y), p_2(y)$
be homogeneous degree-$(d-1)$ polynomials in $y$. Then
$\langle q_1 \, p_1,   q_2 \, p_2 \rangle = (1/d) \, \langle q_1, q_2 \rangle \, \langle p_1, p_2 \rangle$.
\end{claim}

Fix any $p\in W$. If $\{z^{(1)},\ldots,z^{(m')}\}$ is a basis of $\R^{m'}$,
we can write $p(x,y) = \sum_{j=1}^{m'} (x\cdot z^{(j)})p_j(y)$,
for some \new{homogeneous degree-$(d-1)$} polynomials $p_j(y)$.
If we pick the $z^{(j)}$'s so that $z^{(1)}\cdot x^{(i)}=1$ and $z^{(j)} \cdot x^{(i)}=0$ for $j>1$,
\new{for some $i \in [t]$},
then the  polynomial $(x \cdot z^{(j)})$, for $j>\new{1}$,
is orthogonal to $q_i(x)$, and $(x\cdot z^{(1)})$ has inner product $1$ with $q_i(x)$.
\new{By Claim~\ref{claim:inner-prod-d}, this implies that
$\langle p, B_{i,j} \rangle = (1/d) \, \langle p_1, p_{i,j}\rangle$.
On the other hand, it is easy to see that $(A_{x^{(i)}}(p))(y) = p(x^{(i)},y) = p_1(y).$
Therefore, we have that
\begin{equation} \label{eqn:ub-ip}
\left| \langle p , B_{i,j} \right \rangle| = (1/d) \left| \langle A_{x^{(i)}}p, p_{i,j} \rangle \right| =
(1/d) \left| \langle p , A_{x^{(i)}}^T p_{i,j} \rangle \right| \leq (\eta/d) \, \|p\|_{\ell_2} \;,
\end{equation}
where the last inequality is Cauchy-Schwarz using the assumption that
$p_{i,j}$ is a singular vector of $A_{x^{(i)}}^T$ with singular value at most $\eta$.}

We will use \eqref{eqn:ub-ip} to prove a contradiction, based
on an analysis of the eigenvalues \new{of the Gram matrix of the $B_{i, j}$'s}.
By construction, we have that each $B_{i,j}$ has a component orthogonal
to all of the $B_{i',j'}$ with $i' < i$ or with $i'=i$ and $j'\neq j$ of $\ell_2$-norm at least \new{$\gamma/\sqrt{d}$}.
\new{Fix an ordering of the $B_{i, j}$'s in increasing order of $i$. For simplicity, we use a single index $\ell = (i, j)$
and will refer to the set of $B_{\ell}$'s in this ordering. Note that $\ell \in L$, where the index set $L$ has size
$|L|  =  |U| = \Theta(k)$.

Let $M$ be the matrix whose rows are the $B_{\ell}$'s, in increasing
order of $\ell$ according to our ordering. That is, $M$ is a linear operator such that $M^T e_{\ell}  = B_{\ell}$, where $e_{\ell}$
is the standard basis vector whose $\ell$-th coordinate is $1$. By writing the rows of $M$ in the appropriate basis
(i.e., a basis where the $j$-th term is the orthogonal part of the $j$-th row),
we obtain a lower triangular matrix with diagonal entries of magnitude at least $\gamma/\sqrt{d}$.
Therefore, if $M M^T$ is the corresponding Gram matrix, we have that
\begin{equation} \label{eqn:det-lb}
\det(M \new{M^T})^{\new{1/2}}\geq (\gamma/\sqrt{d})^{|U|} \;.
\end{equation}
Note that for any $p = \sum_{\ell \in L} a_{\ell} B_{\ell} \in W$, with $a = [a_{\ell}]_{\ell \in L} \in \R^{|U|}$,
we can write:
\begin{eqnarray*}
\left\| a^T MM^T \right\|_2
&=& \left\| p M^T \right\|_2  =\left\|  [\langle p, B_{\ell} \rangle] \right\|_2
\leq \sqrt{|U|} \max_{\ell \in L}  \left | \langle p,  B_{\ell} \rangle \right| \\
&=& O(\sqrt{k} \eta/ d)  \|p\|_{\ell_2} \\
&=& O(\sqrt{k} \eta/ d) \|a\|_1 \max_{\ell \in L} \|B_{\ell}\|_{\ell_2} \\
&=& O(\sqrt{k} \eta /d) \, O(\sqrt{k}) \|a\|_2 \, (R/ \sqrt{d}) \\
&=& O(k \eta R/d^{3/2}) \|a\|_2\;,
\end{eqnarray*}
where the second line follows from \eqref{eqn:ub-ip} and the fourth line
uses Claim~\ref{claim:inner-prod-d} to get that
$$\|B_{i, j}\|_{\ell_2} =  (1/\sqrt{d}) \|q_i\|_{\ell_2} \, \|p_{i, j}\|_{\ell_2}  = (1/\sqrt{d}) \|x^{(i)}\|_2  \leq R / \sqrt{d} \;.$$
Since there is at least a $|U|-k \geq |U|/2$ dimensional subspace of such $a$'s, it follows that
$M M^T$ has at least $|U|/2$ many eigenvalues of size at most $O(k R \eta/d^{3/2})$.

Since the determinant is the product of the eigenvalues, it follows that
\begin{equation} \label{eqn:det-ub}
\det(MM^T) \leq \left(\lambda_{\max} \, O\left(kR \eta / d^{3/2}\right) \right)^{|U|/2}\;,
\end{equation}
where $\lambda_{\max}$ is the largest eigenvalue of $MM^T$.
Combining \eqref{eqn:det-lb} and \eqref{eqn:det-ub}, we obtain that
the largest \new{eigenvalue} of \new{$M M^T$} is at least
$\lambda_{\max} = \Omega\left(\gamma^4 /(k R \eta d^{1/2}) \right)$.
On the other hand, the eigenvalues of $MM^T$ can be bounded from above as follows:
$$v^T(MM^T)v = \|v^T M\|_{\ell_2}^2 = \left\| \littlesum_{\ell \in L} v_{\ell} B_{\ell} \right\|_2^2
\leq \|v\|_1^2 \max_{\ell \in L} \|B_{\ell}\|_2^2 \leq O(k R^2 / d) \|v\|_2^2 \;,$$
which implies that
$\gamma = O(\eta^{1/4} k^{1/2} R^{3/4} d^{-1/8})$.
This gives the desired contradiction, completing the proof of Lemma~\ref{lem:hyperplane-close-bad}.
}
\end{proof}



By Lemma~\ref{lem:hyperplane-close-bad}, all points $z = (x, y)$ in $S$, with $x \in \R^{m'}$ and $y\in \R^{m-m'}$,
whose $x$-coordinates are within $\ell_2$-distance $\eps'$ of a $(k', \eta)$-bad point $c \in \mathcal{C}_{\eps'}$
have their $x$-coordinates within $\ell_2$-distance $\eps'+ \new{O(\eta^{1/4} k^{1/2} R^{3/4} d^{-1/8})}$
from some specific origin-centered hyperplane $H$ of dimension at most $2k/k'$.
Therefore, all such points $z \in S$ are within $\ell_2$-distance $\eps'+ \new{O(\eta^{1/4} k^{1/2} R^{3/4}d^{-1/8})}$
from the origin-centered hyperplane $H' = H \times \R^{m-m'}$ of dimension at most $m-m'+2k/k'$.
Any such point $z$ can be written as $(z_{H'}, z_p)$, where $z_{H'}$ is
the orthogonal projection onto $H'$ and $z_p$ is the orthogonal complement,
where $\|z_p\|_2 \leq \eps'+\new{O(\eta^{1/4} k^{1/2} R^{3/4}d^{-1/8})}$.

Let $V_H$ be the subspace of \new{$m$-variable homogeneous} polynomials in $V$
that depend only in the coordinates of \new{$z_{H'}$},
\new{i.e., for each $p \in V_H$ we have that $p(z_{H'}, z_p) = p(z_{H'})$.}

By Fact~\ref{fact:codim-bound}, $V_H$ has codimension at most $k$
in the space of all \new{degree-$d$ homogeneous} polynomials in these variables.
Let $$S_{H'}  \eqdef \left\{ z \in \R^m : \|z\|_2 \leq R , \; z \in H',
\textrm{ and } |p(z)|\leq \delta \|p\|_{\ell_2}  \textrm{ for all } p\in V_H \right\} \;.$$
Note that for all $z \in S$ we have $z_H\in S_H$.
\new{We can apply our inductive hypothesis to obtain a small cover of $S_{H'}$.
To do so, we need to perform a change of variables to associate $H' = H \times \R^{m-m'}$ with $\R^{m-m'+\dim(H)}$.
This does not affect our bounds because the defined $\ell_2$-norm
on homogeneous polynomials is rotationally invariant. Therefore,}
there is an $(\eps- \eps' - \new{O(\eta^{1/4} k^{1/2} R^{3/4} d^{-1/8})})$cover of $S_H$ of size
$$f\left(R,d,\eps-\eps' - \new{O(\eta^{1/4} k^{1/2} R^{3/4} d^{-1/8})}, \delta, k, m-m'+2(k/k')\right) \;.$$
Using the points of this cover as centers gives us an appropriate $\eps$-cover
for the set of points $z = (x, y)$ in $S$ with $x$-coordinate within distance $\eps'$
of any $(k', \eta)$-bad point $c \in \mathcal{C}_{\eps'}$.
This completes the proof of Proposition~\ref{prop:recursion}.
\end{proof}

\noindent We are now ready to give the proof of Theorem~\ref{thm:main-cover}.

\begin{proof}[Proof of Theorem~\ref{thm:main-cover}]
We proceed by induction and a careful application of Proposition~\ref{prop:recursion}.


We start by noting that $f(R,d,\eps,\delta,k,m)$ is bounded from above $O(R/\eps)^m$,
i.e., the size of an $\eps$-cover of $B_m(0, R)$.
We will use this trivial upper bound when the dimension $m$ is sufficiently small.

The proof will proceed by induction on $d+m$. We will prove the following inductive hypothesis:
There exists a (sufficiently large) universal constant $C>0$ such that
if $0< \eps \leq R$ and
\new{
\begin{equation}\label{eqn:delta-ub}
\delta \leq \eps^d \left( (\eps/R)/(2kd m) \right)^{Cd} \;,
\end{equation}
}
we have that $f(R,d,\eps,\delta,k,m) \leq \left( 2(R/\eps) d k m \right)^{C^2 d^2 k^{1/d}}.$

The trivial upper bound of $O(R/\eps)^m$ on the cover size
already implies our inductive hypothesis when $m \leq C d^2 k^{1/d}$.

\dnew{
When $d=1$, a codimension $k$ subspace of linear functions in $V$ defines
a dimension at most $k$ subspace $U\subseteq \R^m$, where $V$
is the set of linear functions vanishing on $U$. We claim that all points of $S$
must be within $\ell_2$-distance $\delta \leq \eps/2$ of $U$.
This is because for any $x$, there exists a unit vector $v$ perpendicular to $U$ so that $v\cdot x$ is the distance from $x$ to $U$. However, this means that $p(x)=x\cdot v$ vanishes on $U$, so $p\in V$. Therefore, since $\|p\|_{\ell_2} = \|v\|_2 = 1$,
we have that the distance from $x$ to $U$ is $|v\cdot x| \leq \delta \|p\|_{\ell_2} \leq \delta.$

Therefore, there is a $k$-dimensional subspace $U$ so that all points of $S$ are within $\eps/2$ of $U$.
Note that $B_m(0,R)\cap U$ has an $\eps/2$-cover of size $O(R/\eps)^k$, which
gives an $\eps$-cover of $S$ of the appropriate size.
}

For the induction step, we will use the maximum allowable value of $\delta$,
i.e., the RHS of \eqref{eqn:delta-ub} (noting that increasing the value of $\delta$ only makes the claim in question stronger), and we will apply Proposition~\ref{prop:recursion}
with the following parameters:
\begin{equation}\label{eqn:params}
\eps' \eqdef \delta/((2R)^{d-1}d), \quad
k' \eqdef \lfloor k^{1-1/d} \rfloor, \quad
m' \eqdef \lceil 3k/k' \rceil, \quad
\eta \eqdef \eps \, (\eps/R)^4/\new{\poly(C, d, k, m)} \;,
\end{equation}
\new{for an appropriately large polynomial function $\new{\poly(C, d, k, m)}$.}
Since $m > C d^2 k^{1/d}$, the definition of $m'$ above implies that $m'<m$.
Also, we clearly have that $k'<k$.

For Proposition~\ref{prop:recursion} to be applicable, we also need that
$\eps-\eps' \gg k^{1/2}\eta^{1/4}R^{3/4}d^{-1/8}$. To see this, we first note
that, by the definition of $\eps'$ and $\delta$, we get
\begin{equation} \label{eqn:eps-prime-bound}
\eps'  \leq \delta / (2R)^{d-1} \leq \new{\left(\eps/(2 k d m)^C\right)^d /(2R)^{d-1} \leq \eps/(2 k d m)^{C d} }  \;.
\end{equation}
Moreover, we have that
\new{
\begin{equation} \label{eqn:eta-term-bound}
k^{1/2}\eta^{1/4}R^{3/4} d^{-1/8} \leq k^{1/2}\eta^{1/4}R^{3/4} =  k^{1/2} \eps (\eps/R)^{1/4} / \poly(C, d, k, m)
\leq \eps/ \poly(C, d, k, m) \;,
\end{equation}
where we used that $\eps \leq R$ and that the polynomial function
in the denominator of $\eta$ is of sufficiently large constant degree.
By choosing $C$ to be a sufficiently large universal constant \dnew{and the denominator of $\eta$ to be sufficiently large}, the above
implies that $\eps-\eps' \gg k^{1/2}\eta^{1/4}R^{3/4}d^{-1/8}$, as desired.
}

Since the conditions of Proposition \ref{prop:recursion} are satisfied, we have
that $f(R,d,\eps,\delta,k,m)$ is at most the sum of
\begin{equation}\label{term1Eqn}
f(R,d,\eps',\delta,k,m')f(R,d-1,\eps-\eps',(\delta+d(2R)^{d-1}\eps')/\eta,k',m-m')
\end{equation}
and
\begin{equation}\label{term2Eqn}
f(R,d,\eps-\eps'-O(k^{1/2}\eta^{1/4}R^{3/4}d^{-1/8}),\delta,k,m-m'+2(k/k')) \;.
\end{equation}
We start by analyzing~\eqref{term1Eqn}.
By the definition of $m'$ and $\eps'$, our trivial upper bound on the cover size gives
\begin{equation}\label{eqn:first-term}
f(R,d,\eps',\delta,k,m') = O(R/\eps')^{m'} = (2(R/\eps) k d m )^{O(Cdm')} \new{\leq} (2(R/\eps)kd m)^{C^{\new{2}} d k^{1/d}} \;,
\end{equation}
\new{where the second equation uses the definition of $\eps'$ and the last inequality
follows from the definition of $m'$ assuming that $C$ is sufficiently large.}

We consider the parameters of the recursive call $f(R,d-1,\eps-\eps',(\delta+d(2R)^{d-1}\eps')/\eta,k',m-m')$,
i.e., the second term in \eqref{term1Eqn}. To be able to apply the inductive hypothesis for this term,
we need to show that it satisfies \new{the version of~\eqref{eqn:delta-ub} for the corresponding parameters},
i.e., that
$$(\delta+d(2R)^{d-1}\eps')/\eta \leq \new{(\eps-\eps')^{d-1} \left( ((\eps-\eps')/R)/(2 k' (d-1) (m-m')) \right)^{C (d-1)}} \;.$$
\new{We will establish the above inequality as follows:
By the definition of $\eps'$, the LHS is equal to $2\delta/\eta$.
To bound the RHS, we make two simplifications. First, we observe that the RHS only decreases
if we replace $k'$ by $k$, $d-1$ by $d$, and $m-m'$ by $m$.
Second, we note that the RHS changes by a factor of at most $2$
if $\eps-\eps'$ is replaced by $\eps$. Indeed, by \eqref{eqn:eps-prime-bound}, we have that
$\eps - \eps' \geq \eps (1-1/(2 k d m)^{C d})$ and the ratio between the relevant quantities is
$(\eps/(\eps-\eps'))^{(C+1)(d-1)}$. Therefore, to show the desired inequality, it suffices to show that
$$\delta/\eta \ll  \eps^{d-1} \left( (\eps/R)/(2 k d m ) \right)^{C (d-1)} \;.$$
By the definition of $\delta$, the RHS above is equal to $\delta (2kdm / (\eps/R))^C/\eps$. Hence,
the above is equivalent to showing that $\eps/\eta \ll (2kdm (R/\eps))^C$. By the definition of $\eta$,
we need that $(R/\eps^4) \poly(C, d, k, m) \ll (2kdm (R/\eps))^C$, which holds if $C$ is a sufficiently
large constant.}

We now proceed to analyze the recursive call~\eqref{term2Eqn}.
To be able to apply the inductive hypothesis for this term,
we similarly need to show that
\new{$$\delta \leq \tilde{\eps}^d \left( (\tilde{\eps}/R)/(2kd (m-m'+2k/k')) \right)^{Cd} \;,$$}
where $\tilde{\eps} := \eps-\eps'-O(k^{1/2}\eta^{1/4}R^{3/4}d^{-1/8})$.
\new{Note that $m-m'+2k/k' \leq m-1$. By the definition of $\delta$, the desired inequality holds, as long as
$(\eps/\tilde{\eps}) \leq (m/(m-1))^{C/(C+1)}$.
By~\eqref{eqn:eps-prime-bound} and~\eqref{eqn:eta-term-bound}, we obtain that
$\tilde{\eps} \geq \eps \left(1-1/(2m) \right)$. Thus, it suffices to show that
$1/(1-1/(2m)) \leq (m/(m-1))^{C/(C+1)}$. Recalling that $m \geq 2$, the latter inequality
is easily seen to hold for a sufficiently large constant $C>0$.}

\new{
We can now apply the inductive hypothesis for both~\eqref{term1Eqn} and~\eqref{term2Eqn}.
Using the fact that $(k')^{1/(d-1)}< k^{1/d}$ and $\eps-\eps' \geq \eps(1-1/(2kdm)^{C d})$,
we obtain that the second terms of~\eqref{term1Eqn} can be bounded as follows
$$\left( (R/(\eps-\eps')) 2d k' (m-m') \right)^{C^2 (d-1)^2 k^{1/d}}  < \left( (R/\eps) 2d k m \right)^{C^2 (d-1)^2 k^{1/d}} \; e^{C} \;.$$
Using the upper bound from~\eqref{eqn:first-term} on the first term of~\eqref{term1Eqn}, we obtain that
\eqref{term1Eqn} is bounded from above by
$$ (1/(2m)\dnew{)} ((R/\eps) 2kdm)^{C^2 d^2 k^{1/d}} \;.$$
Moreover, we can bound~\eqref{term2Eqn} as follows:
\begin{eqnarray*}
\left((R/\tilde{\eps}) 2kd (m-1)\right)^{C^2 d^2 k^{1/d}}
&\leq& \left((R/\eps) 2kd \frac{4m}{4m-1}(m-1)\right)^{C^2 d^2 k^{1/d}} \\
&\leq& \left((R/\eps) 2kd  \left(1-\frac{1}{2m}\right) m \right)^{C^2 d^2 k^{1/d}} \\
&<& (1-1/(2m)) ((R/\eps)2kdm)^{C^2 d^2 k^{1/d}} \;,
\end{eqnarray*}
where we used that $m-m'+2k/k' \leq m-1$ and $\tilde{\eps} \geq \eps \left(1-1/(4m) \right)$.
Summing these two terms proves our inductive step and completes the proof of Theorem~\ref{thm:main-cover}.}
\end{proof}

\dnew{
We note that Theorem~\ref{thm:main-cover} has poor dependence on $m$.
This will not matter for our applications, which all begin by reducing to the case
$m\leq k$. However, if one wants a better bound in general,
there is a black-box way to remove most of the dependence on $m$.}

\begin{proposition}  \label{prop:main-no-m}
If $\delta < \eps^{d}/2$, then
$f(R,d,\eps,\delta,k,m)\leq f(R,d,\eps/2,\delta,k,O(k^2(R/\eps)^2))+1$.
\end{proposition}
\begin{proof}
Suppose that $V$ is a codimension $k$ subspace of $\R_{[d]}[x_1,\ldots,x_m]$.
We claim that there is a low-dimensional subspace $H \subseteq \R^m$ such
that every point in $S(V,R,\delta)$ is $\eps/2$-close to $H$.

To begin, we write $V^\perp=\nspan(\{p_1,\ldots,p_k\})$,
where $p_i(x) = \langle A_i, x^{\otimes d} \rangle$ and
the $A_i$'s are an orthonormal set of symmetric tensors.
In particular, by duality, $V$ is the set of polynomials $p(x)=\langle A ,  x^{\otimes d} \rangle$,
where $A$ is a symmetric tensor orthogonal to all of the $A_i$.
If we let $W=\nspan(\{A_1,\ldots,A_k\})$, then for any $x\in \R^m$ we can write $x^{\otimes d}$
as $x_W+x_V$, where $x_W$ is the orthogonal projection onto $W$ and $x_V$ the orthogonal \new{complement}.
We note by the above that the polynomial $p$ given by $p(y)=\langle x_V , y^{\otimes d} \rangle$
is in $V$ and that $|p(x)| = \|x_V\|_2^2 = \|x_V\|_2 \, \|p\|_{\ell_2}$.
Therefore, if $x\in S(V,R,\delta)$, it must be the case that $\|x_V\|_2 \leq \delta$.

Therefore, if $x\in S(V,R,\delta)$, we can write
\begin{equation}\label{tensorSumEqn}
x^{\otimes d} = \sum_{i=1}^k c_i A_i + x_V \;,\end{equation}
where $c_i$ are real numbers with $\sum_{i=1}^k c_i^2 \leq \|x\|_2^{2d}$
and $\|x_V\|_2 \leq \delta$. Now if $y$ is a unit vector, taking the inner product
of $y$ with each side (in one of the tensor directions) and taking the norms of each side,
we find that
$$
|x\cdot y|\|x\|_2^{d-1} = \|y\cdot x^{\otimes d}\|_2 \leq \sum_{i=1}^k c_i \|y\cdot A_i\|_2 + \|y\cdot x_V\|_2 \leq \sum_{i=1}^k c_i \|y\cdot A_i\|_2 + \delta.
$$
Equivalently, if $\|x\|_2 \geq \eps$, we have that
$$
|x\cdot y| \leq \sum_{i=1}^k (c_i/\|x\|_2^{d-1}) \|y\cdot A_i\|_2 + \delta/\|x\|_2^{d-1} \leq \sqrt{k}R \max_i  \|y\cdot A_i\|_2 + \delta/\eps^{d-1} \;,
$$
where we used the fact that $\sum_{i=1}^k (c_i/\|x\|_2^{d-1})^2 \leq \|x\|_2^2$,
and thus $\sum_{i=1}^k (c_i/\|x\|_2^{d-1}) \leq \sqrt{k}\|x\|_2$.

Each $A_i$ can be thought of as a linear transformation mapping vectors to rank $d-1$-tensors.
As such, its Frobenius norm is $\|A_i\|_2 =1$. Therefore, each $A_i$ has at most $O(k(R/\eps)^2)$
singular vectors of singular value at least $\eps/(4\sqrt{k}R)$.
Let $H\subseteq \R^m$ be the space of dimension $O(k^2(R/\eps)^2)$ spanned by these singular vectors for all $i$.
Now if $y$ is perpendicular to $H$, then plugging into the above, for $x\in S(V,R,\delta)$, we have that
$$
|x\cdot y| \leq \sqrt{k}R \max_i  \|y\cdot A_i\|_2 + \delta/R^{d-1} \leq \eps/2 \;.
$$
This means that any $x\in S(V,R,\delta)$ is either in $B_m(0,\eps)$ (which can be covered by a single ball)
or within Euclidean distance $\eps/2$ of $H$.

Therefore, to get an $\eps$-cover of $S(V,R,\delta)$, it suffices to get an $\eps/2$-cover of the projection onto $H$.
If we let $V_H$ be the subspace of $V$ consisting only of polynomials $p(x)=p(\pi_H(x))$
that depend only on the projection onto $H$, by Fact \ref{fact:codim-bound}, $V_H$ is of codimension at most $k$
in the space of all such polynomials. Therefore, we can find such a cover of size at most $f(R,d,\eps/2,\delta,k,O(k^2(R/\eps)^2)).$
\end{proof}

\subsection{Algorithmic Version of Theorem~\ref{thm:main-cover}}\label{ssec:cover-alg}

In this subsection, we show:

\begin{theorem} \label{thm:main-cover-alg}
In the context of Theorem~\ref{thm:main-cover}, given a basis for the vector space $V$,
there is an algorithm to compute an $\eps$-cover of $S$ with size 
at most \new{$M = (2 (R/\eps) d k m)^{C^2 d^2 k^{1/d}}$} that runs in $\poly(M)$ time.
\end{theorem}


\begin{proof}
The proof of Theorem~\ref{thm:main-cover} presented in the previous section 
can be made algorithmic in a straightforward manner. First, note that the base cases 
of the induction described above are easy to implement algorithmically. 
A random set of size $O(R/\eps)^m$ can easily be seen to cover $B_m(0,R)$, 
and thus $S$, with high probability. When $d=1$, it is easy to compute the subspace 
$U$ of points on which $V$ vanishes. Then, as described in the proof of Theorem \ref{thm:main-cover}, 
a cover on $U\cap B_m(0,R)$ suffices.

Otherwise, we can set $\eps',k',m',\eta$ as in the inductive step given in the proof of Theorem \ref{thm:main-cover}, 
and we will need to make algorithmic a version of Proposition \ref{prop:recursion}. 
For this, we partition the coordinates as $x$ and $y$ as described in the proof. 
We compute $V_x$ and compute a cover $\mathcal{C}_{\eps'}$ of $S_x$ of size $O(R/\eps')^{m'}$. 
Next we can compute $W$ \new{using linear algebra} to compute the intersection of $V$ with polynomials 
homogeneous of degree-$1$ in the $x$-coordinates and homogeneous of degree $d-1$ in the $y$-coordinates. 
Then for each $c\in \mathcal{C}_{\eps'}$, we can compute $A_c$ and determine whether or not it is $(k',\eta)$-good.

For the points that are $(k',\eta)$-good, we compute $V_c$ as the span of the left eigenvectors of $A_c$ with the largest eigenvalues. 
We then compute an $(\eps-\eps')$-cover of $S'_{y,c}$, which we can find recursively of size 
$((R/(\eps-\eps'))2k'(d-1)m)^{C^2(d-1)^2 k'^{1/(d-1)}}$. This cover will give us a cover of $(B_{m'}(c,\eps') \times \R^{m-m'})\cap S$, 
as described in the proof of Proposition \ref{prop:recursion}. 
Doing this for all good $c\in \mathcal{C}_{\eps'}$, gives a set of size at most
$$
(1-1/(2m))((R/\eps)2dkm)^{C^2d^2k^{1/d}} \;,
$$
as described in our proof.

We now just need a cover of the points whose $x$-coordinates are within $\eps'$ of a bad point of $\mathcal{C}_{\eps'}$. 
We note that we can produce a hyperplane $H$ that nearly passes through all of these points inductively. 
In particular, we begin with $H=0$ and while there is a bad point $c\in \mathcal{C}_{\eps'}$ 
not within distance $\gamma$ (a sufficiently large multiple of $\eta^{1/4}k^{1/2}R^{3/4}d^{-1/8}$) of $H$, 
we let $H$ be $H + \langle c\rangle.$ We note that by the proof of Lemma \ref{lem:hyperplane-close-bad}, 
$H$ will have dimension at most $2k/k'$.

Next, letting $H' = H+\R^{m-m'}$, we note that all of the points with $x$-coordinates close to a bad point 
are $\gamma$-close to $H'$. \new{By linear algebra}, we can compute $V_{H'}$ the subspace 
of the set of polynomials in $V$ that do not depend on the coordinates orthogonal to  $H'$. 
By applying our algorithm recursively to these polynomials on $H'$, we produce 
a cover of size at most $(1/2m)((R/\eps)2kd)^{C^2d^2 k^{1/d}}$. Combining this with the cover for points 
whose $x$-coordinate is close to a good point, this gives us a full cover of $S$ of appropriate size.

It is not hard to verify that the runtime of this algorithm is within polynomial factors 
of the upper bound provided on the final cover size, completing the proof.
\end{proof}

\newpage

\section{Overall Strategy for Learning Applications} \label{sec:strategy}
\new{In Section~\ref{ssec:cover-params}, we explain how and under what conditions
one can use Theorem~\ref{thm:main-cover-alg} to obtain an $\eps$-cover for the set of parameters
in a given learning application.
Section~\ref{ssec:template} presents a template for all our applications that we will follow
in the subsequent sections.}

\subsection{From Covers of Near-Zero Sets of Polynomials to Covers of the Parameters} \label{ssec:cover-params}

The overall strategy of our algorithmic applications is as follows.
We have an underlying learning problem that is defined by a collection of $k$
vectors $v_i \in \R^m$ and corresponding non-negative weights $w_i \geq 0$, $i \in [k]$.
We assume that we have an \new{efficient} method for computing the weighted
low-degree moments of the $v_i$'s. That is, we assume that we can \new{efficiently} obtain
a sufficiently good approximation to the tensors $\sum_{i=1}^k w_i v_i^{\otimes 2d}$,
or equivalently that we can approximate $\sum_{i=1}^k w_i p(v_i)$ for any monomial $p$ of degree $2d$.
By linearity, this allows us to approximate $\sum_{i=1}^k w_i p(v_i)$ to small
error for \emph{any} \new{degree-$2d$ homogeneous} polynomial $p$.

\new{Let $p: \R^m \to \R$ be a real degree-$d$ homogeneous polynomial.
We consider the quadratic form $Q:\R_{[d]}[\new{x_1, \ldots, x_m}]\rightarrow \R$
defined by letting $Q(p)$ be our aforementioned approximation to $\sum_{i=1}^k w_i p^2(v_i)$.
Note that $Q$ has the following crucial property: If $p$ vanishes on all of the $v_i$'s,
then $Q(\new{p})$ {\em nearly} vanishes.
}

This \new{property} allows us to \new{efficiently} compute a subspace $V$
of $\R_{[d]}[x]$, so that for every $p\in V$ we will have that $|p(v_i)|$
is very small for all $i \in [k]$ such that the corresponding weight $w_i$ is \new{not negligibly small}.
It is not hard to see that $V$ will have small codimension, so using \new{Theorem~\ref{thm:main-cover-alg}},
we can efficiently compute a small cover for the set of possible values for such $v_i$'s.

In particular, we show:
\begin{proposition}\label{coverApplicationProp}
Let $m, k \in \Z_+, R >0$, $v_i \in \R^m$ and $w_i \in \R_+$, for all $i \in [k]$.
There is an algorithm that takes as input $R, k, m$, parameters $\delta,\eps>0$ with $\delta \new{\leq} \eps^{2d} ((\eps/R)/(2kmd))^{2Cd}$ for $C>0$
a sufficiently large constant, and a tensor $T: \new{\R^d \to \R}$ such
that $\|T - \sum_{i=1}^k w_i v_i^{\otimes 2d}\|_2 \new{\leq} \delta$, runs in time $(2(R/\eps)kmd)^{O(d^2 k^{1/d})}$,
and outputs a set $\mathcal{C} \new{\subset \R^m}$ of cardinality at most $(2(R/\eps)kmd)^{O(d^2 k^{1/d})}$,
satisfying the following property: For any $i \in [k]$ with $w_i \geq ((\eps/R)/(2kmd))^{Cd}$ \new{and $\|v_i\|_2\leq R$},
there is a $c\in \mathcal{C}$ with $\|v_i-c \|_2 \leq \eps$.
\end{proposition}
\begin{proof}
The algorithm is described below:
\begin{enumerate}
\item Define the quadratic form $Q: \R_{[d]}[x] \to \R$, where $Q(p)$ is defined as follows:
We can write $p^2(x) = A_{p} (x,x,\ldots,x)$, for some uniquely defined
rank-$2d$ symmetric tensor $A_{p}$.
We define $Q(p) = \langle A_{p}, T \rangle$.

\item Let $V\subseteq \R_{[d]}[x]$ be the subspace spanned by all
but the top-$k$ eigenvectors of $Q$ (with respect to our \new{$\|\cdot\|_{\ell_2}$ norm on polynomials}).
\item Run the algorithm from \new{Theorem~\ref{thm:main-cover-alg}} on \new{input $V, \eps, k, m, \delta, R$}
to obtain the set $\mathcal{C}$.
\end{enumerate}
To show that this algorithm works, we first note that $Q$ is in fact a quadratic form, \new{as the $A_{p}$ are quadratic in $p$}.
In fact, if $p(x) = B_{p} (x,x,\ldots,x)$, for a symmetric tensor $B_{p}$ of rank $s$,
then $A_{p}$ is the symmetrization of $B_{p} \otimes B_{p}$.
It then follows from the Cauchy-Schwartz inequality that
$\|A_p\|_2 \leq \|B_p\|_2^2 = \|p\|_{\new{\ell_2}}^2$.
Therefore, we have that
\begin{align*}
Q(p) & = \langle T ,  A_{p} \rangle\\
& = \left\langle \sum_{i=1}^k w_i v_i^{\otimes t} , A_{p} \right\rangle +
      \left\langle  T-\sum_{i=1}^k w_i v_i^{\otimes 2d} , A_{p} \right\rangle\\
& = \sum_{i=1}^k w_i p^2(v_i)+ O\left(\left\| T-\sum_{i=1}^k w_i v_i^{\otimes 2d} \right\|_2 \left\| A_{p} \right\|_2 \right)\\
& = \sum_{i=1}^k w_i p^2(v_i)+O\left(\delta\|p\|_{\ell_2}^2\right) \;.
\end{align*}
Therefore, $Q(p)$ is indeed a good approximation to the quadratic form
$p \rightarrow \sum_{i=1}^k w_i p^2(v_i)$, as desired.

We next show that $Q$ has \new{many} small eigenvalues.
In particular, let $U$ be the space of polynomials $p \in \R_{[d]}(x)$ so that $p(v_i)=0$ for all $i \in [k]$.
Note that $U$ is the kernel of the map $E:\R_{[d]}(x)\rightarrow \R^k$ given by
$E(p)=(p(v_1),\ldots,p(v_k))$. Therefore, $U$ has co-dimension at most $k$ in $\R_{[d]}[x]$.
On the other hand, for $p\in U$, we have that $\sum_{i=1}^k w_i p^2(v_i)=0$, and therefore
$|Q(p)| = O(\delta \|p\|_{\ell_2}^2)$. Thus, the $(k+1)^{st}$ largest eigenvalue of $Q$ is at most $O(\delta \|p\|_{\ell_2}^2)$.
In particular, this implies that $V$ is spanned by eigenvalues of at most this size. Therefore, for all $p\in V$,
we have that
$$
\sum_{i=1}^k w_i p^2(v_i)+O(\delta\|p\|_{\ell_2}^2) = Q(p) = O(\delta\|p\|_{\ell_2}^2) \;.
$$
That is, for $p\in V$, we have that
$$
\sum_{i=1}^k w_i p^2(v_i)=O(\delta\|p\|_{\ell_2}^2) \;.
$$
In particular, this means that for all $p\in V$ and all $w_i \geq ((\eps/R)/(2kmd))^{Cd}$, we have that
$$
|p(v_i)| = O\left(\eps^{d} ((\eps/R)/(2kmd))^{Cd/2}\|p\|_{\ell_2}\right) \;.
$$
This implies that $v_i$ satisfies the condition for being in \new{the set $S$} in \new{Theorem~\ref{thm:main-cover}},
and therefore there exists $c\in \mathcal{C}$ so that $\|v_i-c\|_2 \leq \eps$.
\end{proof}

This cover will prove useful to us, however, the requirement that we learn
these $m$-dimensional tensors for potentially large values of $m$ is suboptimal.
We show by a similar technique that we can often reduce the problem
to an at most $k$-dimensional one.

\begin{proposition}\label{dimensionReductionProp}
Let $v_i,w_i,k,m,T,\delta,R$ be as in Proposition \ref{coverApplicationProp} with $d=1$ and $w>0$.
There is an algorithm that, given this input, computes a subspace $U$ of dimension at most $k$
so that, for all $w_i \geq w$, we have that $v_i$ is within $\ell_2$-distance $O((\delta/w)^{1/2})$ of $U$.
Furthermore, this algorithm runs in polynomial time.
\end{proposition}
\begin{proof}
The algorithm is as follows:
\begin{enumerate}
\item Compute $Q$ and $V$ as in Proposition~\ref{coverApplicationProp}.
\item Let $U$ be the set of points $x\in \R^m$ so that $p(x)=0$ for all polynomials $p$ in $V$.
\end{enumerate}
To show correctness, it is not hard to see that, since $V$ has co-dimension at most $k$,
$U$ will have dimension at most $k$. In particular, if $p_1,\ldots,p_{m-k}$ is a basis for $V$,
$U$ is the subspace defined by these $m-k$ linear constraints, and so will have dimension $k$.

On the other hand, suppose that we have an $i \in [k]$ with $w_i\geq w$.
We note that there is a unit vector $u_i$ orthogonal to $U$
so that the $\ell_2$-distance from $v_i$ to $U$ is $u_i\cdot v_i$.
Let $p(x)$ be the polynomial $p(x)=x \cdot u_i$. Since $p$ vanishes on $U$, it must be in $V$.
This means that the $\ell_2$-distance from $v_i$ to $U$ is
$$
|u_i\cdot v_i| = |p(v_i)| = O((\delta/w)^{1/2}\|p\|_{\ell_2}) = O((\delta/w)^{1/2}) \;,
$$
as desired.
\end{proof}

A common application of this technique is where we are learning a high-dimensional distribution $D$
that is given as a mixture $D=\sum_{i=1}^k w_i \theta(v_i)$, where $\theta$ is some family of distributions
parameterized by the vectors $v_i$. Proposition~\ref{coverApplicationProp} will allow us to get a large
list of hypotheses that will include approximations to all of the large components in this mixture.
It will usually be the case therefore, that we can approximate $D$ by another distribution
$D' = \sum_{i=1}^k w_i' \theta(c_i)$, where the $c_i$'s are in our cover.
In particular, if $\|v-c\|_2 \leq \eps$ implies that $\dtv(\theta(v),\theta(c)) < \eta$,
then by replacing $v_i$ by the closest $c_i$ and letting $w_i=w_i'$,
it is not hard to see that $\dtv(D,D') \leq \eta+k((\eps/R)/(2kmd))^{Cd}$,
as the $L_1$-distance between $w_i\theta(v_i)$ and $w_i\theta(c_i)$
will always be at most $\min(((\eps/R)/(2kmd))^{Cd},w_i\eta)$.

This shows that $D$ can be approximated by a mixture of the distributions $\theta(c_i)$.
It turns out that we can always find such a distribution efficiently:

\begin{proposition}\label{mixtureProp}
Let $p_1,\ldots,p_n$ be explicit probability distributions
and let $X$ be a probability distribution such that, for some $w_1,\ldots,w_n$,
we have $\dtv(X,\sum_{i=1}^n w_ip_i) \leq \eps$. Then there exists a $\poly(n/\eps)$-time
algorithm that given $p_1,\ldots,p_n$ and $N > n/\eps^2$ samples from $X$,
returns a distribution $p$ such that with probability at least $2/3$,
we have that $\dtv(X,p)=O(\sqrt{\eps\log(n/\eps)})$.
\end{proposition}
\begin{proof}
Let $\Delta$ be the set of distributions of the form $\sum_{i=1}^n w_i p_i$,
where $w_i\geq \eps/n$ for all $i$ and $\sum_{i=1}^n w_i =1$.
Note that $\Delta$ is a convex set and that there exists a $p^\ast\in \Delta$
with $\dtv(p^\ast,X)\leq 2\eps$. For a distribution $p$, let $L(p,x) = \log(p(x))$.
We note that $L(p):=\E[L(p,X)] = D(X||p)+H(X)$,
where $D(X||p)$ is the  KL-divergence.
Our strategy will be to find a $p\in \Delta$ that is an empirical minimizer of $L(p)$.

In particular, given our $N$ samples $x_1,\ldots,x_N$ and a distribution $p$, we define
$$
\hat{L}(p) = \frac{1}{N}\sum_{i=1}^N L(p,x_i) \;.
$$
We claim that with high probability over our samples,
for every pair $p,q\in \Delta$ it holds
$\hat{L}(p)-\hat{L}(q) = L(p)-L(q)+O(\eps\log(n/\eps))$. To see this,
we note that $L(p)-L(q) = \E[\log(p(X)/q(X))]$.
Since $p$ and $q$ are both in $\Delta$ and are mixtures of the $p_i$ with mixing weights at least $\eps/n$,
it is easy to see that $\eps/n \leq p(x)/q(x) \leq n/\eps$. From this, we find that
$$
L(p)-L(q) = \log(n/\eps) + \int_{-\log(n/\eps)}^{\log(n/\eps)} \pr(\log(p(x)/q(x))>t) dt \;.
$$
Similarly, we have
$$
\hat{L}(p)-\hat{L}(q) = \log(n/\eps) + \int_{-\log(n/\eps)}^{\log(n/\eps)} \frac{\#\{x_i: \log(p(x_i)/q(x_i))>t\}}{N} dt \;.
$$
It thus suffices to show that, with high probability over our samples,
for all $p,q\in \Delta$ and $t \in \R$, it holds that
$$
\left|\pr[\log(p(x)/q(x))>t] - \frac{\#\{x_i: \log(p(x_i)/q(x_i))>t\}}{N} \right| = O(\eps) \;.
$$
Note that $\log(p(x)/q(x))>t$ is equivalent to $e^tp(x)-q(x)>0$.
Since $p$ and $q$ are linear combinations of the $p_i$'s,
this in turn is equivalent to saying that $\sum_{i=1}^n a_ip_i(x) > 0$,
for some specific constants $a_i$. However, the class of sets defined
by this equation, for some numbers $a_i$, has VC-dimension $n$
(as it is just the set of halfspaces in $n$ dimensions, after embedding
$x$ into $\R^n$ as $(p_1(x),\ldots,p_n(x))$.
Therefore, our result holds by the VC-Inequality (Theorem~\ref{thm:vc}).

Our algorithm uses convex optimization to find a $p^\ast$ such that
$\hat{L}(p^\ast)$ is within $O(\eps\log(1/\eps))$ of the global maximum over all $p\in \Delta$.
By the above, this must be a maximizer of $L(p^\ast)$ up to $O(\eps\log(n/\eps))$.
Next, we note that for $p,q\in \Delta$, since $\log(p(x)/q(x))$ is bounded,
if $\dtv(X,Y)=O(\eps)$ then $\E[\log(p(X)/q(X))]$ and $\E[\log(p(Y)/q(Y))]$
differ by $O(\eps\log(n/\eps))$. Taking $p\in\Delta$ with $\dtv(X,p)=O(\eps)$,
we apply the above with $Y=p$ and $q=p^\ast$. This says that
$$
L(p)-L(p^\ast) = D(p||p^\ast) + O(\eps\log(n/\eps)) \;.
$$
Since $p^\ast$ is a near maximizer of $L$, the left hand size above is at most $O(\eps\log(n/\eps))$.
This in turn implies that $D(p||p^\ast) = O(\eps\log(n/\eps))$ and, by Pinsker's inequality (Fact~\ref{fact:pinsker})
that $\dtv(p,p^\ast)=O(\sqrt{\eps\log(n/\eps)})$.
Our final result now follows from the triangle inequality.
\end{proof}

\subsection{Template Approach for Learning Applications}\label{ssec:template}

In this section, we describe at a high-level how the preceding
theorems are used to make our applications work.

\subsubsection{Setup}

First, we need to define our problem in the context described.
In particular, we have access to an object parameterized by $k$
vectors $v_i$ and $k$ non-negative weights $w_i$.

\subsubsection{Moment Computation}

Critically, we need a way to compute approximations of the moments $\sum_{i=1}^k w_i v_i^{\otimes 2d}$.
For this, it suffices for every degree-$2d$ monomial $p(x)$ to be able
to approximate $\sum_{i=1}^k w_i p(v_i)$ to error $\delta/m^d$.
This is usually done by finding some polynomial function $P$ of our samples
that is an unbiased estimator of $\sum_{i=1}^k w_i p(v_i)$ and computing an empirical mean.

\subsubsection{(Optional) Rough Clustering}

One issue with this technique is that our requirements on error
are often dependent on the upper bound $R$ we have on the \new{$\ell_2$-norm} of the $v_i$'s.
As having large $v_i$ will usually also make our sample complexity to approximate moments
higher as well, it is often important to reduce to the case where $R$ is relatively small.
This can often be done by performing some kind of rough clustering of samples
to split our problem into components whose $v_i$ all lie in a relatively small ball.

\subsubsection{(Optional) Dimension Reduction}

We will often want to use Proposition~\ref{dimensionReductionProp} to reduce
to the case where the underlying problem is $k$-dimensional.
This is because the $m$-dimensional version of the problem will often incur
runtime and sample complexity proportional to $m^d$.

\subsubsection{Covering}

Next we compute \new{the weighted moments} of the $v_i$'s
and use Proposition~\ref{coverApplicationProp} to efficiently find
an appropriate cover.

\subsubsection{From Covers to Learning}
Finally, we use this cover to learn. For density estimation, this entails some sort of algorithm
with time polynomial in the cover size, analogous to Proposition~\ref{mixtureProp}.
For parameter estimation, we need to employ additional problem-specific algorithmic ideas.

\section{Mixtures of Spherical Gaussians} \label{sec:gmm}
\subsection{Setup}

\begin{definition}[Mixtures of Spherical Gaussians] \label{def:gmm}
An $m$-dimensional {\em $k$-mixture of spherical Gaussians} (spherical $k$-GMM)
is a distribution on $\R^m$ with density function $F(x) = \sum_{j=1}^k w_j N(\mu_j,  I)$,
where $\mu_j \in \R^m$, $w_j \geq 0$, for all $j \in [k]$, and $\sum_{j=1}^k w_j = 1$.
\end{definition}


We study both density estimation and parameter estimation.
In density estimation, we want to output a hypothesis distribution with total variation
distance at most $\eps$ from the target.
In parameter estimation, we assume that the means of the components are sufficiently separated,
and the goal is to recover the unknown mixing
weights and mean vectors to small error $\eps$. Specifically,
we would like to return a list $\{ (\widetilde{w}_j, \widetilde{\mu}_j), j \in [k] \}$ such that
for some permutation $\pi\in \mathbf{S}_k$,
$|w_j-\widetilde{w}_{\pi(j)}| \leq \eps$, and $\|\mu_j-\widetilde{\mu}_{\pi(j)}\|_2  \leq \eps$,
for all $j \in [k]$.


For density estimation, we prove:

\begin{theorem}[Density Estimation for Spherical $k$-GMMs]\label{thm:gmm-density-est}
There is an algorithm that on input $d \in \Z_+$, $\eps>0$, and
$N = \tilde{O}(m^2) \poly(k/\eps) + (k/\eps)^{O(d^2 k^{1/d})}$ samples
from an unknown spherical $k$-GMM $F$ on $\R^m$, the algorithm
runs in time $\poly(m k/\eps)+ (2 k d /\eps)^{O(d^2 k^{1/d})}$ and outputs a hypothesis distribution $H$ such that
with high probability $\dtv(H, F) \leq \eps$.
\end{theorem}


For parameter estimation, we prove:

\begin{theorem}[Parameter Estimation for Spherical $k$-GMMs]\label{thm:gmm-param}
There is an algorithm that on input $d \in \Z_+$, $\eps>0$, and
sample access to an unknown spherical $k$-GMM $F$ on $\R^m$ with minimum weight $p_{\min}$ and pairwise
mean separation at least a sufficiently large multiple of $\sqrt{\log(1/p_{\min})}$, the algorithm draws
$N = \poly(m/(\eps p_{\min})) + k^{O(d)} $ samples from $F$, runs in time
$\poly(N)+k^{O(d^2 k^{1/d})}$, and with high probability outputs an $\eps$-approximation to the unknown
mean vectors and weights.
\end{theorem}


\subsection{Rough Clustering}

\new{In this subsection, we show that we can efficiently pre-process our problem to reduce to the case
that all the component means have appropriately bounded $\ell_2$-norm. In particular, we show the following:}

\dnew{
\begin{lemma}\label{lem:gmm-rough-clustering}
Let $N$ be a positive integer and $X$ be a $k$-mixture of spherical Gaussians in $\R^m$.
There exists an algorithm that, given $C k^2 N$ independent samples from $X$,
for a sufficiently large constant $C>0$, runs in time $\poly(N,m,k)$
and computes at most $k$ centers $C_i \in \R^m$ such that the following holds:
With high constant probability, to each mixing component with weight $w_i\geq 1/(kN)$
there will be an associated center $C_i$ with $\|C_i-\mu_i\|_2=O(k(\sqrt{m}+\log(Nk/\eps)))$.
Moreover, there is in efficient algorithm that given a sample from $X$,
with probability at least $1-1/N$ returns the center associated
with the component that the sample was drawn from.
\end{lemma}

The idea of this lemma is to help us reduce to the case where all of our means are in a ball of bounded radius.
In particular, if we take our $N$ samples from $X$, there is a decent probability that every sample
is correctly assigned to its component's center. If so, replacing these samples $x$ by $x-C_i$
will give us $N$ i.i.d. samples from $X' := \sum_{i=1}^k w_i N(\mu_i-C_i,I)$,
a $k$-mixture of spherical Gaussians whose means are all within $\ell_2$-distance
$O(k(\sqrt{m}+\log(Nk/\eps)))$ of the origin.
}

\begin{proof}
The basic idea here is that all of the samples from a given Gaussian component on $\R^m$ will be within $\ell_2$-distance
about $O(\sqrt{m})$ of each other. If we cluster together close points, we can try to identify the components.
This will not work directly, since we may have pairs of components that are close to each other
and whose samples will lie within $O(\sqrt{m})$ of each other. However, no chain of such close samples will get us
more than $\dnew{O}(k\sqrt{m})$ away. This allows us to cluster together points whose means
are within $\ell_2$-distance about $O(k\sqrt{m})$ of each other.

More formally, we note that, for any $\eta >0$, a random sample $x \sim N(\mu,I)$ satisfies
$\|x-\mu\|_2 \leq 2(\sqrt{\new{m}}+\log(1/\eta))$ with probability $1-\eta$.
Therefore, if we take $1/\eta$ iid samples, with high probability it will hold that
\begin{itemize}
\item Every sample taken is within $\ell_2$-distance $2(\sqrt{m}+\log(1/\eta))$ of some $\mu_i$.
\item For each component $i$ with $w_i \new{\geq} k \eta$, we will have at least one sample within $\ell_2$-distance
\dnew{$2(\sqrt{m}+\log(1/\eta))$} from the corresponding mean vector $\mu_i$.
\end{itemize}
If both of these conditions hold, we can perform a rough clustering on the points.
In particular, we declare two points to be ``close'' if their Euclidean distance is at most
$10(\sqrt{m}+\log(1/\eta))$, and declare them to be in the same cluster if they are connected
by some chain of close points. We note that since any two points from the same component
of our mixture are close with high probability, these chains need not be longer than $O(k)$ in length.
So, each cluster of points has diameter $O(k(\sqrt{m}+\log(1/\eta)))$.
For each cluster, we pick a center $C_i$.

We then note that a random sample $x$ drawn from our mixture $X = \sum_{i=1}^k w_i N(\mu_i,I)$ satisfies
the following conditions with probability at least $1-k^2 \eta$:
\begin{enumerate}
\item $x$ comes from a component $N(\mu_i, I)$ with $\|\mu_i-C_j\|_2 = O(k(\sqrt{m}+\log(1/\eta)))$, for some center $C_j$.
\item For the $C_j$ chosen above, $x$ is within $\ell_2$-distance $4(\sqrt{m}+\log(1/\eta))$ of some point of that cluster,
but not within this distance of any point of any other cluster.
\end{enumerate}
The first claim will hold if at least one of the original samples $x_0$ drawn
to produce the clusters is within $\ell_2$-distance $2(\sqrt{m}+\log(1/\eta))$ of $\mu_i$.
The second claim above holds because, with high probability, $x$ is within
$\ell_2$-distance $2(\sqrt{m}+\log(1/\eta))$ of $\mu_i$. This means that it is within
$\ell_2$-distance $4(\sqrt{m}+\log(1/\eta))$ of $x_0$. If so, it cannot be this close to an $x_1$
from another cluster, since then, by the triangle inequality, $x_0$ and $x_1$ will be close,
and thus in the same cluster.

This means that if we draw an independent set of $N = o(1/(k^2 \eta))$
additional points from $X$,
we can associate them to clusters so that with high probability the following holds:
\begin{itemize}
\item All of the samples from the same component of $X$ end up in the same cluster.
\item All of the components whose samples are associated with a given cluster
have means that are within $\ell_2$-distance $O(k(\sqrt{m}+\log(1/\eta)))$ of the mean of that cluster.
\end{itemize}

\dnew{This completes our proof.}
\end{proof}

\subsection{Moment Computation}

\new{
The following lemma shows that we can efficiently approximate any entry of the tensor
$\sum_{i=1}^k w_i \mu_i^{\otimes 2d}$ to small error:

\begin{lemma}\label{lem:gmm-moment-est}
Suppose that we have sample access to $X = \sum_{i=1}^k w_i N(\mu_i, I)$, where $\|\mu_i\|_2 \leq R$,
for all $i \in [k]$, for a parameter $R>0$. There is an algorithm that, given $\delta>0$, $d \in \Z_+$,
and a multi-index $\bi \in [m]^{2d}$, draws $(R m d)^{O(d)}/\delta^2$ samples from $X$, runs in sample-polynomial time,
and outputs an approximation $T_{\bi}$ of $(\sum_{i=1}^k w_i \mu_i^{\otimes 2d})_{\bi}$
with expected squared error $O(\delta^2)$.
\end{lemma}
}
\begin{proof}
Let $He_n(t)$, $t\in \R$, $n \in \Z_+$, denote the probabilist's Hermite polynomial.
\new{
We will show the following claim:

\begin{claim}\label{claim:gmm-moment-formula}
For any $\alpha \in \N^{m}$ we have that:
\begin{equation}\label{eqn:gmm-moment}
\sum_{i=1}^k w_i {\mu_i}^{a} = \E\left[\prod_{j=1}^{\new{m}} He_{a_j}(X_j) \right] \;,
\end{equation}
where $\mu_i \in \R^m$, $i \in [k]$, are the component means of $X = \sum_{i=1}^k w_i N(\mu_i, I)$.
\end{claim}
}
\begin{proof}
Note that if $G=N(0,1)$ and $\mu \in \R$, we can write
\begin{align*}
\E[He_n(G+\mu)] & = \E\left[\sum_{i=0}^n \left(\frac{\partial}{\partial x} \right)^i He_n(G)/i! \mu^i \right] \\
& = \E\left[\sum_{i=0}^n He_{n-i}(G) \frac{(n)(n-1)\cdots(n-i+1)\mu^i}{i!} \right] = \mu^n \;,
\end{align*}
where the first line above is by Taylor expanding $He_n$ about $G$.
Next suppose that $X = N(\tilde{\mu},I)$ for some vector $\tilde{\mu} = (\tilde{\mu}_1, \ldots, \tilde{\mu}_m) \in \R^m$.
For $a = (a_1, \ldots, a_m) \in \Z_+^m$, we have that
$$
\E\left[\prod_{i=1}^{\new{m}} He_{a_i}(X_i) \right] = \prod_{i=1}^{\new{m}} \E[He_{a_i}(X_i)]
= \prod_{i=1}^{\new{m}} \tilde{\mu}_i^{a_i} = \new{\tilde{\mu}^{a}}.
$$
\new{Finally, let $X = \sum_{i=1}^k w_i N(\mu_i, I)$, where $\mu_i \in \R^m$, $i \in [k]$.
By linearity we get that}
$$
\sum_{i=1}^k w_i {\mu_i}^{a} = \E\left[\prod_{j=1}^{\new{m}} He_{a_j}(X_{\new{j}}) \right] \;,
$$
as desired. This completes the proof of Claim~\ref{claim:gmm-moment-formula}.
\end{proof}

Given $N$ independent samples from $X = \sum_{i=1}^k w_i N(\mu_i, I)$,
we can use Claim~\ref{claim:gmm-moment-formula} to approximate $\sum_{i=1}^k w_i \mu_i^{a}$ by the empirical mean of
$\prod_{j=1}^{\new{m}} He_{a_j}(X_{\new{j}})$. Recall our assumption that $\|\mu_i\|_2 \leq R$, $i \in [k]$, for some parameter $R>0$.

To bound the sample complexity, it suffices to bound the variance of the term in the RHS of
\eqref{eqn:gmm-moment}.

We note that $He_a(t)$, $t \in \R$, is a degree-$a$ polynomial
with sum of absolute values of coefficients at most $a!$.
So, \new{if $|a| = 2d$}, $\prod_{j=1}^{\new{m}} He_{a_j}(X_j)$ will have degree at most $2d$,
and the sum of the absolute values of its coefficients will be at most $(2d)!$.
Therefore, its absolute value will be at most $(1+\|X\|_2^{2d})(2d)!$.
Over any component, we have that $\E[\|X\|_2^2]\dnew{= m+\|\mu\|_2^2=O(R^2+m)}$.
Therefore, by hypercontractivity, it follows that $\E[\|X\|_2^{\dnew{4}d}] \leq O((R^{\dnew{4}d}+m^{\dnew{2}d}))d^{O(d)}$.
Thus, the variance of $\prod_{j=1}^{\new{m}} He_{a_j}(X_j)$ will be $O(Rmd)^{O(d)}$.

\dnew{This is because
\begin{align*}
\Var\left( \prod_{j=1}^{\new{m}} He_{a_j}(X_j) \right) & \leq \E\left[\left(\prod_{j=1}^{\new{m}} He_{a_j}(X_j)\right)^2 \right]\\
& \leq \E\left[(1+|X|^2)^{4d} d^{O(d)} \right] \leq (1+R^{4d}+m^{2d})d^{O(d)} = O(Rmd)^{O(d)}.
\end{align*}
}
\new{Recall that  if $\wh{Z}_N$ is the empirical distribution obtained
by taking $N$ iid samples from the random variable $Z$, then $\E[(\wh{Z}_N - \E[Z])^2 ] = \Var[Z]/N$.}

Therefore, with $N = (R m d)^{C'd}/\delta^2$ samples from $X$, for $C'$ a sufficiently large constant,
we can approximate $\sum_{i=1}^k w_i {\mu_i}^{a}$, for $|a| = 2d$, to expected $L_2^2$-error $\delta^2$.
\end{proof}

\new{
Using Lemma~\ref{lem:gmm-moment-est} to approximate each entry of
$\sum_{i=1}^k w_i \mu_i^{\otimes 2d}$ to appropriately high accuracy,
we can approximate the entire tensor $\sum_{i=1}^k w_i \mu_i^{\otimes 2d}$ within small $\ell_2$-error.

\begin{corollary}\label{cor:gmm-tensor-est}
By taking $N > (R m d)^{Cd}/\delta^2$ samples from $X$, for an appropriate constant $C>0$,
we can efficiently compute a tensor $T$ such that with high constant probability it holds
$\| T - \sum_{i=1}^k w_i \mu_i^{\otimes 2d} \|_2^2 \leq \delta^2$.
\end{corollary}
\begin{proof}
We take $N = m^{2d} (R m d)^{C'd}/\delta^2$  samples from $X$,
and consider the tensor $T = (T_{\bi})$, $\bi \in [m]^{2d}$,
as our approximation to $\sum_{i=1}^k w_i \mu_i^{\otimes 2d}$.
By Lemma~\ref{lem:gmm-moment-est}, we have that
$$\E \left[ \left\| T - \littlesum_{i=1}^k w_i \mu_i^{\otimes 2d} \right\|_2^2 \right] \leq m^{2d} (\delta/m^d)^2
= O(\delta^2) \;.$$
The corollary follows from Markov's inequality.
\end{proof}
}

\subsection{Dimension Reduction}

\dnew{
After reducing the radius, we can perform dimension reduction.
Our dimension reduction procedure is described in the following lemma.

\begin{lemma}\label{GMMDimReductionLem}
There exists an algorithm that given $N=\poly(km/\eps)$ i.i.d. samples from $X$, for a sufficiently large degree polynomial,
runs in $\poly(N, k, m)$ time and computes a subspace $H$ in $\R^m$ of dimension at most $2k$ such
that with large constant probability the following holds:
For every $i \in [k]$ with $w_i \geq \eps/k$, we have that $\mu_i$ is within $\ell_2$-distance $\eps$ of $H$.
\end{lemma}

We note that (perhaps after replacing $\eps$ by a slightly smaller quantity)
in order to solve either the density estimation or parameter estimation problems,
it will suffice to solve the same problem after projecting $X$ onto the subspace $H$.
For density estimation, we note that $X$ is $O(\eps)$-close in total variation distance to
$X_H = \sum_{i=1}^k w_i N(\pi_H(\mu_i),I)$. Note that $X_H$ is just the product of
$\pi_H(X)$ with a standard Gaussian in the orthogonal directions. Thus, if we can learn
$\pi_H(X)$ to error $O(\eps)$, we can also learn $X$ to error $O(\eps)$.

For parameter estimation, we note that every center with non-trivial weight is $\eps$-close,
in $\ell_2$-distance, to its projection on $H$. In particular, a parameter estimation algorithm
applied to $\pi_H(X)$ will learn each $\pi_H(\mu_i)$ to error $\delta/w_i$.
We note that if $\eps < \delta$, this will means that for $w_i < \delta$ there is nothing to show,
and for $w_i > \delta$, we have that $\mu_i$ is at most $\eps$ distance away from $\pi_H(\mu_i)$,
introducing at most an additional $\delta$-error between $\mu_i$ and our approximation.

We now prove Lemma \ref{GMMDimReductionLem}.

\begin{proof}
We begin by applying Lemma \ref{lem:gmm-rough-clustering} with $N=(2km/\eps)^C$, for $C$ a sufficiently large constant.
This gives us a number of centers $C_i$. If we then take $N$ additional samples and consider the differences
between the point and the associated center, this will give us i.i.d. samples from $X' = \sum_{i=1}^k w_i N(\mu_i-C_i,I)$,
a mixture of spherical Gaussians with means of $\ell_2$-norm at most $O(k\sqrt{m}+k\log(k/\eps))$.
Then applying Corollary \ref{cor:gmm-tensor-est}, we can use these samples to produce an estimation
to $\sum_{i=1}^k w_i (\mu_i-C_i)^{\otimes 2}$ to error $O(\eps^3/(Ck))$.
Finally, applying Proposition \ref{dimensionReductionProp},
we can compute a subspace $U$ of dimension at most $k$,
such that for ever $i$ with $w_i \geq \eps/k$,
we have that $\mu_i-C_i$ is within $\ell_2$-distance $\eps$ of $U$.
Letting $H$ be the span of $U$ and the $C_i$'s yields our result.
\end{proof}
}

\dnew{
\subsection{Clustering and Cover}
Now that we have reduced to $k$ dimensions, we can (after reapplying rough clustering in order to reduce the radius to $\poly(k)$)
more readily afford to compute higher moments. We can use this to compute a cover. We note that we will usually apply this
lemma after first projecting onto the subspace $H$ found by Lemma \ref{GMMDimReductionLem}, and thus $m$ will be $O(k)$.

\begin{lemma}\label{GMMClusteringLemma}
Let $X=\sum_{i=1}^k w_i N(\mu_i,I)$ be a mixture of Gaussians in $\R^m$ and let $\eps>0$.
There exists an algorithm that given $N=(2kmd/\eps)^{\Theta(d)}$ samples (with sufficiently large constant in the exponent)
computes a cover $\mathcal{C}$ of size at most $(2kmd/\eps)^{O(d^2 k^{1/d})}$, such that with high probability
for every $i$ with $w_i \geq (\eps/(dkm))^{\Omega(d)}$, we have that there is a $c\in\mathcal{C}$ with $\|\mu_i-c\|_2 \leq \eps$.
Furthermore, this algorithm runs in time at most $(2kmd/\eps)^{O(d^2 k^{1/d})}$.
\end{lemma}
\begin{proof}
We begin by letting $N'$ be a small multiple of $N/k^2$.
Running Lemma \ref{lem:gmm-rough-clustering} with parameter $N'$,
gives us a list of centers $C_i$ so that every component with non-trivial mass is associated to some $C_i$,
and so that we can use our remaining $N'$ samples to produce $N'$ i.i.d. samples from
$X' = \sum_{i=1}^{k'} w_i N(\mu_i-C_i,I)$, where the $w_i$ of mass less than $1/(kN')$ are excluded from the list.
We note that this is a mixture of spherical Gaussians with means of $\ell_2$-norm at most $R=O(k(\sqrt{m}+d\log(kd/\eps)))$.
Using Corollary \ref{cor:gmm-tensor-est}, $N'$ samples suffice to compute the $2d$-th moment tensor of this mixture
to error $\eps^{2d}((R/\eps)2kmd)^{2Cd}$, for the constant $C$ required by Proposition \ref{coverApplicationProp}.
Applying this proposition gives us a cover $\mathcal{C}_0$ of appropriate size, such that for every $i$ with
$w_i\geq \eps/k$, we have that there is some $c\in \mathcal{C}_0$ with $\|\mu_i-C_i-c\|_2 \leq \eps$.
Letting $\mathcal{C}$ be the set of points of the form $C_i+c$, where $C_i$ is a center and $c\in \mathcal{C}_0$,
gives us an appropriate cover.
\end{proof}
}

\subsection{Density Estimation}

\dnew{Here we prove Theorem \ref{thm:gmm-density-est}.}

\begin{proof}
\dnew{We begin by computing a hyperplane $H$ as described in Lemma \ref{GMMDimReductionLem},
and note as described that it suffices to solve the problem
on the mixture of Gaussians $\pi_H(X)$ in $O(k)$ dimensions.
We will assume henceforth that $m=O(k)$.

Using Lemma \ref{GMMClusteringLemma}, we can compute an $\eps$-cover $\mathcal{C}$ of size $S=(k/\eps)^{O(d^2k^{1/d})}$.
We note that $X$ is $O(\eps)$-close to a mixture of Gaussians with centers in $\mathcal{C}$.
This is because if each center is rounded to the nearest element of $\mathcal{C}$, the ones with $w_i\leq \eps/k$
contribute at most $\eps$-error in total, while the ones with larger $w_i$ contribute error $O(\eps w_i)$, which sums to $O(\eps)$.
Applying Proposition~\ref{mixtureProp} to the distributions $N(c,I)$, with $c\in\mathcal{C}$,
we can learn $X$ to total variation error $O(\sqrt{\eps \log(S/\eps)})$ in $\poly(S/\eps)$ time.
Reparametrizing and replacing $\eps$ by a small enough multiple of $\eps^2/(d^2 k^{1/d} \log(kd/\eps))$ yields our result.}

\end{proof}

\subsection{Parameter Estimation}

Here we prove Theorem \ref{thm:gmm-param}.
\begin{proof}
\dnew{
We assume that for all $i \neq j$, we have that $\|\mu_i-\mu_j\|_2 \geq 60\sqrt{\log(1/p_{\min})}$.

We begin by producing a list of candidate means.
Using Lemma \ref{GMMDimReductionLem}, with $\poly(m/p_{\min})$ samples and time,
we can compute an $O(k)$-dimensional hyperplane $H$ such that all of the $\mu_i$'s
for which $w_i\geq p_{\min}$ (i.e., all of them) are within distance $1/2$ of $H$.
Next, applying Lemma \ref{GMMClusteringLemma} to $\pi_H(X) = \sum_{i=1}^k w_i N(\pi_H(\mu_i),I)$,
we can use $(2kd)^{O(d)}$ samples and $\poly(S)$ time, with $S=(2kd)^{O(d^2 k^{1/d})}\poly(1/p_{\min})$,
to produce a set $\mathcal{C}$ of size at most $S$ so that for every $i$ with $w_i\geq p_{\min}$
there is a $c\in\mathcal{C}$ so that $\|\pi_H(\mu_i)-c\|_2<1/2$,
which by the triangle inequality implies that $\|c-\mu_i\|_2 \leq 1$.
}

Once we have constructed our cover of the candidate means, we can use techniques from~\cite{DiakonikolasKS18-mixtures}.
We begin by taking an additional set $T$ of $N=O(k^3/p_{\min}^2)$ samples.
For each $c \in \mathcal{C}$, we determine whether there is a weight function $u:T \to \R$, such that
\begin{enumerate}
\item For each $x\in T$, $u_x \in [0,1]$.
\item The sum $\sum_{x\in T} u_x \geq p_{\min}|T|/2$.
\item For any other $c'$ in $\mathcal{C}$, it holds
$$
\sum_{x \in T : (x-c) \cdot (c'-c)/\|c'-c\|_2 > 2 \sqrt{\log(1/p_{\min})}} u_x \leq p^2_{\min} |T|/10 \;.
$$
\end{enumerate}
We call $c$ good if there is such a $u$.
We note that it can be determined whether or not such a $u$ exists by linear programming in $\poly(S,|T|)$ time.

We note that if $\|c-\mu_i \|_2<1$ for some $i \in [k]$, then letting $u_x=1$, for $x$
drawn from the $i^{th}$ component and $0$ otherwise,
satisfies the above with high probability. Finally, we claim that there exists no set
of more than $4/p_{\min}$ such $c\in\mathcal{C}$ that are pairwise separated
by more than $4\sqrt{\log(1/p_{\min})}$. Indeed, if we had such a set $c_1,\ldots,c_{\dnew{t}}$,
then we \dnew{can reach a contradiction by considering the total weight of
all points under the $c_i$s' weight functions.
In particular, if $u_x^i$ is the weight function associated with $c_i$, we have that}:
\begin{align*}
|T| & \geq \sum_{x\in T} 1\\
& \geq \sum_{i=1}^t \sum_{x\in T: \mathrm{argmin}_j \|c_j-x\|_2 = i} u_x^i\\
& \geq \sum_{i=1}^t \sum_{x\in T: (x-c_i)\cdot (c_j-c_i)/\|c_j-c_i\|_2 < 2\sqrt{\log(1/p_{\min})} \textrm{ for all } j\neq i} u_x^i\\
& \geq \sum_{i=1}^t \left(\sum_{x\in T} u_x^i - \sum_{j\neq i} \sum_{x\in T: (x-c_i)\cdot (c_j-c_i)/\|c_j-c_i\|_2 < 2 \sqrt{\log(1/p_{\min})}} u_x^i\right)\\
& \geq \sum_{i=1}^t p_{\min} |T|/2 - (t-1)p_{\min}^2 |T|/10\\
& \geq |T|(tp_{\min}/2-(t p_{\min})^2/10) \;.
\end{align*}
This is a contradiction when $t=\lceil 4/p_{\min} \rceil.$

This means that if we take any maximal set of good elements of $\mathcal{C}$
that are pairwise separated by $4\sqrt{\log(1/p_{\min})}$,
this set has size at most $(4/p_{\min})$. Call such a set $\mathcal{C'}$.
Note that every $\mu_i$ is within $\ell_2$-distance $1$ of a good element of $\mathcal{C}$,
which is within $\ell_2$-distance $4\sqrt{\log(1/p_{\min})}$ of an element of $\mathcal{C'}$.

Next take an additional $\dnew{m}/p_{\min}$ samples from $X$.
To each sample associate the closest element of $\mathcal{C'}$.
We claim that with high probability every sample coming from a component $N(\mu_i, I)$
is closest to an element $c \in \mathcal{C'}$ with $\|\mu_i-c\|_2 <15 \sqrt{\log(1/p_{\min})}$.
This holds for the following reason.
Let $c_0, c_1 \in \mathcal{C'}$ be some elements with $\|c_0-\mu_i \|_2 \leq 4 \sqrt{\log(1/p_{\min})}$
and $\|c_1-\mu_i \|_2 \geq 15\sqrt{\log(1/p_{\min})}$. We claim that with probability at least $1-p_{\min}^4$
a random sample from $N(\mu_i,I)$ is closer to $c_0$ than to $c_1$, and note that this will prove our claim.

To show this, we let $v$ be the unit vector in the direction of $c_1-c_0$.
We note that $x$ is closer to $c_1$ than $c_0$ if and only if $v\cdot x$ is closer to $v\cdot c_1$ than to $v\cdot c_0$.
However, with high probability, $|v\cdot x -v\cdot \mu_i | < 3 \sqrt{\log(1/p_{\min})}$.
On the other hand, $|v\cdot c_0 -v\cdot \mu_i | < 4\sqrt{\log(1/p_{\min})}$, but
$$|v\cdot c_1 -v\cdot \mu_i| \geq |c_1-c_0| - |v\cdot c_0 -v\cdot \mu_i| \geq  11\sqrt{\log(1/p_{\min})}.$$
Thus, we have that $|v\cdot x -v\cdot c_0| \leq 7 \sqrt{\log(1/p_{\min})}$,
but $|v\cdot x -v\cdot c_1| \geq 8 \sqrt{\log(1/p_{\min})}$.

Next consider two samples to be in the same cluster if and only if
the associated elements of $\mathcal{C'}$ are within $\ell_2$-distance
$30 \sqrt{\log(1/p_{\min})}$ of each other. If the condition above holds,
any two samples from the same component will lie in the same cluster.
However, our separation assumption implies that samples from different
components will not. Thus, each cluster of samples consist of i.i.d. samples from that component.
With high probability, each component has at least $k$ samples from it, so taking the sample mean
will give us an approximation to the mean of that cluster to error $O(1)$. Using this approximation as
warm start, we can apply the algorithm of~\cite{RV17-mixtures} to obtain an $\eps$-approximation of each $\mu_i$
with $\poly(1/\eps,1/p_{\min})$ further samples. 
\dnew{This completes our proof.}
\end{proof}

\section{Positive Linear Combinations of ReLUs} \label{sec:relu}

\subsection{Setup}

\begin{definition}[One-hidden-layer ReLU networks] \label{def:sum-of-relus}
Let $\mathcal{C}_{m, k}$ denote the concept class of one-hidden-layer ReLU networks
on $\R^m$ with $k$ hidden units. That is, $F \in \mathcal{C}_{m,k}$ if and only if
there exist $k$ \new{unit} vectors $w_i \in \R^m$ and non-negative coefficients $a_i \in \R_+$, $i \in [k]$,
such that $F(x) = \sum_{i=1}^k a_i \relu (w_i \cdot x)$,
where $\relu(t) = \max\{0, t\}$, $t \in \R$.
\end{definition}

The PAC learning problem for the class $\mathcal{C}_{m, k}$
is the following: The input is a multiset of i.i.d. labeled examples $(x, y)$,
where $x \sim N(0,I)$ and $y = F(x)+\xi$, for an unknown $F \in \mathcal{C}_{m, k}$
and $\xi \sim N(0,\sigma^2)$, where $\xi$ is independent of $x$ and $\sigma \geq 0$ is known.
\new{We will call such an $(x,y)$ a {\em noisy sample} from $F$.}

The goal of the learner is to output a hypothesis
$H: \R^m \to \R$ that with high probability is close to $F$ in $L_2$-norm, i.e.,
satisfies \new{$\|H-F\|_2^2 \leq \eps^2 (\|F\|_2^2+\sigma^2)$}.
(For a function $F: \R^m \to \R$, we define $\|F\|_2 \eqdef \E_{x \sim N(0, I)}[F^2(x)]^{1/2}$.)
The hypothesis $H$ is allowed to lie in any efficiently representable hypothesis class $\mathcal{H}$.
If $\mathcal{H} = \mathcal{C}_{m, k}$, the PAC learning algorithm is called {\em proper}.

The main result of this section is the following theorem:

\begin{theorem}[PAC Learning for $\mathcal{C}_{m, k}$] \label{thm:sum-relus}
There is a PAC learning algorithm for $\mathcal{C}_{m, k}$ with respect to
the standard Gaussian distribution on $\R^m$ with the following performance guarantee:
Given $d\in \Z_+$, $\eps>0$, and access to noisy samples
from an unknown target $F \in \mathcal{C}_{m, k}$,
the algorithm draws $O(m^2k^2/\eps^6) + (2kd/\eps)^{O(d)}$ samples,
runs in time $\poly(mk/\eps)+(2kd/\eps)^{O(d^2 k^{1/d})}$,
and outputs a hypothesis $H$ that with high probability satisfies
\new{$\|H-F\|_2^2 \leq \eps^2 (\|F\|_2^2+\sigma^2)$.}
\end{theorem}

We note that the function $F(x)  = \sum_{i=1}^k a_i \relu (w_i \cdot x)$ satisfies
$\E_{x\sim N(0,I)}[F(x)] = \new{(1/\sqrt{2\pi}) A}$,
and \new{$\|F\|_2  = \Theta (A)$, where $A \eqdef \sum_{i=1}^k a_i$}.
\new{Moreover, we can assume
w.l.o.g. that we know the value of $\|F\|_2$, as this can be computed to arbitrary precision using samples
via a simple pre-processing.} \dnew{In particular, by dividing all samples by some sufficiently accurate
approximation to $\|F\|_2+\sigma$, we can reduce to the case where $\|F\|_2+\sigma=\Theta(1)$,
and we will assume that this holds throughout our analysis.}

\subsection{Moment Estimation}

\new{



The following lemma shows that we can efficiently approximate any entry of the tensor
$\sum_{i=1}^k a_i w_i^{\otimes 2d}$ to small error:

\begin{lemma}\label{lem:relu-moment-est}
There is an algorithm that, given $\delta>0$, $d \in \Z_+$, and a multi-index $\bi \in [m]^{2d}$,
draws $d^{O(d)}/\delta^2$ independent noisy samples from an unknown
$F(x)  = \sum_{i=1}^k a_i \relu (w_i \cdot x)$, runs in sample-polynomial time,
and outputs an approximation $T_{\bi}$ of $(\sum_{i=1}^k a_i w_i^{\otimes 2d})_{\bi}$
with expected squared error $O(\delta^2 )$.
\end{lemma}
}
\begin{proof}
\new{The proof proceeds by constructing an appropriate polynomial function that is an unbiased
estimator of $\sum_{i=1}^k a_i w_i^{\alpha}$, for any multi-index $\alpha \in \N^{m}$
with $|\alpha| = 2d$ and using samples
to approximate it.

We start with the following claim:

\begin{claim}\label{claim:relu-moment-formula}
For any $\alpha \in \N^{m}$ with $|\alpha| = 2d$, we have that:
\begin{equation}\label{eqn:relu-moment}
\sum_{i=1}^k a_i w_i^\alpha =
C_{\alpha} \, \E_{x \sim N(0,I), \xi \sim N(0, \sigma^2)}\left[ (F(x)+\xi) \left(\prod_{i=1}^m h_{\alpha_i}(x_i) \right)\right] \;,
\end{equation}
where $C_{\alpha}>0$ is an explicit constant satisfying $C_{\alpha} = d^{O(d)}$.
\end{claim}
}
\begin{proof}
\new{We will require the following basic facts about Hermite polynomials.}
Let $h_n(t) = He_n(t)/\sqrt{n!}$, $t \in \R$, be the normalized probabilist's Hermite polynomial.
In particular, for $G\sim N(0,1)$ we have $\E[h_n(G)h_m(G)] = \delta_{n,m}$.
It is easy to see that $h_n'(t) = \sqrt{n}h_{n-1}(t)$.

Note that the second derivative of $\relu(t)$, $t \in \R$, is $\delta_0(t)$.
By writing $\relu(t) = \sum_{n=0}^\infty c_n h_n(t)$ and taking the second derivative,
we obtain
$$
\delta_0 \new{(t)} = \sum_{n=0}^\infty h_n(0) h_n(t) = \sum_{n=0}^{\infty} c_{n+2} \sqrt{(n+2)(n+1)}h_n(t) \;.
$$
Equating terms, we find that
$$
c_n = h_n(0)/\sqrt{(n+1)(n+2)} =
\begin{cases}
\frac{(-1)^{n/2}}{2^{n/2}(n/2)!\sqrt{(n+1)(n+2)}} \;, & \textrm{ for }n>0\textrm{ even}\\
0 \;, & \textrm{ for }n>1\textrm{ odd \;.}
\end{cases}
$$
It is also easy to check that $c_1 = 1/2$ and $c_0 = \frac{1}{\sqrt{2\pi}}$.

We next evaluate $\relu(w \cdot x)$ for a unit vector $w$.
By the rotation formula for Hermite polynomials,
we get that
$$
\relu(x\cdot w) = \sum_{n=0}^\infty c_n h_n(x\cdot w) =
\sum_{\alpha} c_{|\alpha|} w^\alpha \sqrt{|\alpha|!/\alpha!}\left(\prod_{i=1}^m h_{\alpha_i}(x_i) \;.\right)
$$
Therefore, for $|\alpha|=2d$, we have that
$$
w^\alpha = \sqrt{\alpha!/(2d)!} c_{2d}^{-1} \E_{x \new{\sim} N(0,I)}\left[ \relu(x\cdot w) \left(\prod_{i=1}^m h_{\alpha_i}(x_i) \right) \right] \;.
$$
Extending this by linearity, we conclude that
$$
\sum_{i=1}^k a_i w_i^\alpha=\sqrt{\alpha!/(2d)!} c_{2d}^{-1} \E_{x \sim N(0,I)}\left[(F(x)+N(0,\sigma^2))
\left(\prod_{i=1}^m h_{\alpha_i}(x_i) \right) \right] \;.
$$
By the definition of $c_{2d}$, we have that
$C_{\alpha}: = \sqrt{\alpha!/(2d)!} c_{2d}^{-1} = d^{O(d)}$,
completing the proof of Claim~\ref{claim:relu-moment-formula}.
\end{proof}

Given Claim~\ref{claim:relu-moment-formula}, we can approximate the weighted moments of the $w_i$'s
by replacing the expectation by the corresponding empirical expectation.
To bound the error involved, it suffices to bound from above the variance of the term
$$
C_{\alpha} \, (F(x)+N(0,\sigma^2)) \left(\prod_{i=1}^m h_{\alpha_i}(x_i) \right) \;,
$$
appearing in the RHS of~\eqref{eqn:relu-moment}.
\new{To bound the variance, note that by the Cauchy-Schwarz inequality,} we get that
$$
\Var\left[ (F(x)+N(0,\sigma^2)) \left(\prod_{i=1}^m h_{\alpha_i}(x_i) \right)\right] \leq \left\|(F(x)+N(0,\sigma^2)) \right\|_4^2
\left\| \prod_{i=1}^m h_{\alpha_i}(x_i) \right\|_4^2.
$$
By the hypercontractive inequality (Theorem~\ref{thm:hc}),
we have that
$$
\left\| \prod_{i=1}^m h_{\alpha_i}(x_i) \right\|_4^2 \leq \left\| \prod_{i=1}^m h_{\alpha_i}(x_i) \right\|_2^2 d^{O(d)} = d^{O(d)} \;,
$$
\new{where we used the fact that the Hermite polynomials have norm one.}
We also have that
$$
\left\|F(x)+N(0,\sigma^2) \right\|_4^2 \ll \|F(x)\|_4^2+\|N(0,\sigma^2)\|_4^2 \ll A^2 + \sigma^2 \dnew{\ll 1} \;,
$$
\new{where we used that $\|F(x)\|_4 \leq \sum_{i=1}^k a_i \| \relu(x \cdot w_i) \|_4 \ll  \sum_{i=1}^k a_i = A$.}
\new{Therefore, the variance of the relevant term is at most $d^{O(d)} $.}

\new{Taking $N = d^{O(d)}/\delta^2$, completes the proof of Lemma~\ref{lem:relu-moment-est}.}
\end{proof}

\new{
Using Lemma~\ref{lem:relu-moment-est} to approximate each entry of
$\sum_{i=1}^k a_i w_i^{\otimes 2d}$ to appropriately high accuracy,
we can approximate the entire tensor $\sum_{i=1}^k a_i w_i^{\otimes 2d}$ within small $L_2$-error.

\begin{corollary}\label{cor:relu-tensor-est}
By taking $N = d^{O(d)} m^{2d}/\delta^2$ noisy samples from $F$, we can efficiently compute
a tensor $T$ such that with high constant probability it holds
$\| T - \sum_{i=1}^k a_i w_i^{\otimes 2d} \|_2^2 \leq \delta^2.$
\end{corollary}
\begin{proof}
We take $N = d^{O(d)} m^{2d}/\delta^2$ noisy samples from $F$,
and consider the tensor $T = (T_{\bi})$, $\bi \in [m]^{2d}$, as our approximation to $\sum_{i=1}^k a_i w_i^{\otimes 2d}$.
By Lemma~\ref{lem:relu-moment-est}, we have that
$$\E \left[ \left\| T - \littlesum_{i=1}^k a_i w_i^{\otimes 2d} \right\|_2^2 \right] \leq m^{2d} (\delta/m^d)^2
= O(\delta^2)) \;.$$
The corollary follows from Markov's inequality.
\end{proof}
}

\subsection{Dimension Reduction}

\new{By Corollary~\ref{cor:relu-tensor-est}, applied for $d=1$ and $\delta = \eps^3/k$,}
with \new{$O(k^2 m^2/\eps^6)$} \new{noisy samples from $F$},
we can efficiently compute the \new{weighted degree-$2$ moment-tensor
$\sum_{i=1}^k a_i w_i^{\otimes 2}$} to \new{$L_2^2$}-error \new{$(\eps^6/k^2)$}.
By Proposition~\ref{dimensionReductionProp}, we can \new{efficiently} find a $k$-dimensional subspace $U$,
such that all of the $w_i$'s with corresponding coefficient \new{$a_i \geq (\eps/k)$}
are within $\ell_2$-distance $\eps$ of $U$.

By performing a change of \new{variables}, we can assume that $U$ is the span of the first $k$ coordinates.
\new{By Lemma~\ref{lem:relu-moment-est}, given $d^{O(d)}/\delta^2$ noisy samples
from $F$, we can efficiently approximate $\sum_{i=1}^k a_i w_i^\alpha$, for any $\alpha$ such that $|\alpha| = 2d$
within expected squared error $O(\delta^2 )$. We use this fact (with $\delta/k^d$ in place of $\delta$)
for all such $\alpha$ that are supported on the first $k$ coordinates. This gives us an approximation $T$ to the tensor
$T_U := \sum_{i=1}^k a_i \pi_U(w_i)^{\otimes 2d}$ that with high constant probability satisfies $\|T - T_U \|_2^2 \leq \delta^2$.
This takes sample complexity $d^{O(d)}k^{2d}\delta^{-2}$ and sample-polynomial time.
}

\subsection{Cover}


\new{We apply the above procedure to produce an approximation $T$ to $T_U$ to within $L_2^2$-error
$\delta^2$, where $\delta = (\eps/(2kd))^{Cd}$, where $C>0$ is a sufficiently large constant.
This takes sample complexity $d^{O(d)}k^{2d}\delta^{-2} = d^{O(d)}k^{2d} (2kd/\eps)^{O(d)} $
and sample-polynomial time.}

Noting that \new{$\|\pi_U(w_i)\|_2 \leq 1$ for all $i\in [k]$}, we can apply the algorithm of Proposition~\ref{coverApplicationProp}
for $m=k$, $R=1$, and $T$ our tensor approximation to $T_U$.
This outputs a set $\mathcal{C} \subset \new{\R^k}$ of size $|\mathcal{C}| = S \leq (2k \new{d}/\eps)^{O(d^2 k^{1/d})}$
such that each \new{for each $i \in [k]$ with $\|\pi_U(w_i)\|_2 \geq (\eps/(2kd))^{Cd}$},
$\pi_U(w_i)$ is within $\ell_2$-distance $\eps$ of some element of $\mathcal{C}$.

\subsection{Computing a Non-proper Hypothesis}

We are given a set of $S$ functions of the form
$f_i = \relu(x\cdot c_i)/\|c_i\|_2$, for $c_i\in \mathcal{C}$.
\dnew{We claim that there is} a non-negative linear combination $\tilde F$ of the $f_i$'s such that
$\|F-\tilde F\|_2 = O(\eps)$. \dnew{This is because for every $i$ with $a_i\geq \eps/k$,
$w_i$ is $\eps$-close to $U$, and there is a $c_i\in\mathcal{C}$ with $\|c_i-\pi_U(w_i)\|_2 \leq \eps$.
By the triangle inequality, this implies that $\|c_i-w_i\|_2=O(\eps)$ and, since $\|w_i\|_2=1$,
that $\| c_i/\|c_i\|_2-w_i \|_2=O(\eps)$. Therefore, for $a_i\geq \eps/k$,
we have a corresponding $f_i$ such that $\|a_i \relu(x\cdot w_i)-a_i f_i\|_2 = O(a_i\eps)$.
For $a_i\leq \eps/k$, we have that $\|a_i \relu(x\cdot w_i)-0\|_2 =O(\eps/k)$.
Therefore, we have that
$$
\left\| F-\sum_{i:a_i\geq \eps/k} a_i f_i \right\|_2 \leq \sum_{i=1}^k O(\eps/k+a_i\eps) = O(\eps(1+A)) = O(\eps) \;.
$$}
We wish to find such a non-negative linear combination $\tilde F$. We note that if we can compute
each of the inner products $F \cdot f_i := \E_{x\sim N(0,I)} [F(x)f_i(x)]$ to error $O(\eps^2)$, this will be sufficient.
This is because if we take any $F'(x) = \sum_i a'_i f_i(x)$ with $a'_i \geq 0$ and $\sum_i a'_i = A' = O(1)$
(and note that $\E[F'(x)]$ is proportional to $A'$, so if it is much larger than $1$, \dnew{it cannot be close to $F$} in $L_2$-norm),
then we can compute $\|F-F'\|_2^2 - \|F\|_2^2 = -2 F\cdot F' +F'\cdot F'$ to error $O(\eps^2)$. \dnew{Thus, if we find a vector of $a'_i\geq 0$ that gives an empirical minimizer of $\|F-F'\|_2^2$, it will give us an $\tilde F$ with
$\|F-\tilde F\|_2^2 \leq O(\eps^2) + \inf \|F-\tilde F\|_2^2 = O(\eps^2)$.
Note that this problem is equivalent to finding numbers $a'_i\geq 0$ that minimize
$$
-2 \sum_i a'_i (\textrm{approximation of }F\cdot f_i) + \sum_{i,j} a'_i a'_j f_i\cdot f_j \;,
$$
which is a convex optimization problem that can be solved in $\poly(S)$ time.
}

\subsection{Putting it Together}
\new{
In summary, we have described an algorithm that obtains a hypothesis $H = \tilde F$ such that with high constant probability
$\|H - F\|^2_2 < \eps^2 (\|F\|_2^2+\sigma^2)$ with sample complexity
$$
O(m^2k^2/\eps^6) + (2kd/\eps)^{O(d)} \;,
$$
and running time
$$
\poly(mk/\eps)+(2kd/\eps)^{O(d^2 k^{1/d})}  \;.
$$
In particular, setting $d = \log(k)$, we get sample complexity
$O(m^2k^2/\eps^6) + (k/\eps)^{O(\log k)} $
and running time $\poly(mk/\eps)+(k/\eps)^{O(\log^2 k)} $.}

\section{Positive Linear Combinations of Generalized Linear Models} \label{sec:glm}

\subsection{Setup}

Here we show that the algorithmic results of the last section can be generalized 
to linear combinations from any generalized linear model, under certain mild assumptions on the model.

\begin{definition}\label{def: plc-GLM}
Let $\sigma:\R\rightarrow \R$ be a fixed function. Let $\mathcal{C}_{\sigma,m,k}$ denote the class of 
real-valued functions on $\R^m$ of the form $F(x) = \sum_{i=1}^k a_i \sigma(x\cdot w_i)$, 
where the $w_i$'s are unit vectors and $a_i \in \R_+$.
\end{definition}

The setup will be similar to the one in the previous section. The algorithm will be given access to samples of the form $(x,y)$ where $x \sim N(0, I)$ and $y=F(x)+\xi$, where $\xi$ is an error term. We will no longer assume that $\xi$ is independent of $x$, but we will assume that it is unbiased for any given $x$ and not too large. In particular, we will assume that $\E[\xi \mid x] = 0$ for every value of $x$, and that 
$\E[\xi^4] \leq E^4$, for some known constant $E>0$. Finally, we will need to assume that $\sigma$ is reasonably well-behaved. In particular, we will say that $\sigma$ is {\em $L$-continuous} to mean that 
$\E[|\sigma(v\cdot x)-\sigma(w\cdot x)|^2]\leq L^2 \|v-w\|_2^2$, for $v$ and $w$ any unit vectors 
and $x$ a standard Gaussian. Under these assumptions, we state our main result.

\begin{theorem} \label{thm:glm}
Let $\sigma:\R\rightarrow \R$ be a known $L$-continuous function, and $F\in \mathcal{C}_{\sigma,m,k}$ an unknown function with $\sum_{i=1}^k a_i \leq 1$. Assume that for some positive integer $d$ that $\E[h_{2d}(G)\sigma(G)]=c_{2d}\neq 0$ (where $h_{2d}$ is the degree-$2d$ Hermite polynomial and $G$ a standard Gaussian), and that $\|\sigma(G)\|_4 \leq M$. There exists an algorithm that given $E,d, \eps>0$ 
and access to independent samples from a distribution $(x,y)$, where $x\in G^m$ and $y=F(x)+\xi$ 
with $\E[\xi|x]=0$ and $\Var(\xi)<E^2$, takes 
$N=(m)^{4d}(M+E)^2/(c_{2d}^2)L^{O(dk^{1/d})}(2kmd/\eps)^{O(d^2k^{1/d})}$ samples 
and runs in sample polynomial time and with probability at least $2/3$ returns a function $\hat F$ 
with $\|\hat F(x)-F(x)\|_2 \leq \eps$.
\end{theorem}

We are not aware of prior work on this problem that leads to algorithms with sub-exponential dependence on $k$.

We note that the non-vanishing even degree Fourier coefficient will be necessary for us. In particular, this means that our algorithm will not work if $\sigma$ is an odd function, like the logistic function. 
This difficulty seems hard to circumvent as our algorithm will operate by trying to find a small cover 
of the set of possible $w$'s that appear in the decomposition. Unfortunately, if $\sigma$ is odd, 
we could have the function $F(x)= \sigma(w\cdot x) + \sigma(-w\cdot x) = 0$, 
and it is information-theoretically impossible to recover $w$ from $F$.

We can hope to circumvent these issues if our function is given as a mixture rather than a sum. In particular if the $a_i$'s sum to $1$ one possibility we could have is that $F(x)$ is equal to $\sigma(w_i\cdot x)$ with probability $a_i$. In this case, we note that for any function $g$ that $\E[g(y)|x] = \sum_{i=1}^k a_i g(\sigma(x\cdot w_i))$, and if we can find a function $g$ so that $g\circ \sigma$ has a non-vanishing even-degree Fourier coefficient, we can hope to make our algorithm work. A specific example of this, with important practical relevance is given in the next section.

\subsection{Moment Estimation}

The necessary moment computation is relatively straightforward.
\begin{lemma}
Let $T=\sum_{i=1}^k a_i w_i^{\otimes 2d}$. There exists an algorithm that given $N = O(m)^{4d}(M+E)^2/(\delta^2 c_{2d}^2)$ for some $\delta>0$, runs in sample polynomial time and returns an estimate of $T$ that is accurate to error at most $\delta$ with constant probability.
\end{lemma}
\begin{proof}
It is clear that the degree-$2d$ Fourier moment tensor associated to $\sigma(x\cdot w)$ for unit vector $w$ is $c_{2d} w^{\otimes 2d}$. Linearity implies that the corresponding Fourier moment tensor for $F$ is $c_{2d} T$. Therefore, we can get an unbiased estimator for any given entry of $T$ as $\E[y h_a(x)]/(c_{2d})$, where $h_a$ is the multivariate Hermite polynomial $\prod_{i=1}^m h_{a_i}(x_i)$, where $a_i$ is the number of occurrences of $i$ in the index of the entry of $T$ we are trying to estimate.

If we estimate $T$ by taking an empirical average of this for each entry over $N$ entries, 
we will get expected entry-wise error on the order of $\sqrt{\Var(y h_a(x))}/(\sqrt{N}c_{2d})$. 
We can bound the variance using Holder's Inequality by 
$\sqrt{\E[y^4]\E[h_a^4(x)]} = O(\|F\|_4^2+E^2)O(1)^d$. 
Thus, we can learn $T$ to error $O(m)^{2d}(\|F\|_4+E)/(\sqrt{N}c_{2d})$. 
We note that $\|F\|_4 \leq \sum_{i=1}^k a_i |\sigma(w_i\cdot x)|_4 \leq M$. 
Plugging in an appropriate value of $N$ gives our result.
\end{proof}

\subsection{Finishing the Proof}

From here the argument is straightforward. For $C$ a sufficiently large constant, we learn $T$ as above to accuracy $\delta=L^{-C}(\eps/(2kmd))^{Cd}$ and then apply Proposition \ref{coverApplicationProp}. This gives us a set $\mathcal{C}$ of size at most $(2kmd/\eps)^{O(d^2 k^{1/d})}L^{O(dk^{1/d})}$ so that every $w_i$ with $a_i\geq \eps/(8k)$ is within $\eps/(4L)$ of some element of $\mathcal{C}$. By modifying the points of $\mathcal{C}$ slightly if necessary, we can assume that they all are unit vectors. We note that if $\hat w_i$ is an element of $\mathcal{C}$ with $\|w_i-\hat w_i \|_2 < \eps/(4L)$ 
whenever $a_i \geq \eps/(8k)$, then
\begin{align*}
\left \| F(x)-\sum_{i=1}^k a_i \sigma(x\cdot \hat w_i) \right \|_2 
&\leq \sum_{i=1}^k a_i \|\sigma(x\cdot w_i) - \sigma(x\cdot \hat w_i)\|_2\\
& \leq \sum_{i=1}^k \max(\eps/(4k),a_i (\eps/4))\\
& \leq \sum_{i=1}^k \eps/(4k) + \sum_{i=1}^k a_i \eps/4\\
& = \eps/2 \;.
\end{align*}
Thus, letting $V$ be the span of all functions of the form $\sigma(x\cdot v)$ for $v\in\mathcal{C}$, we note that $F$ is within $L_2$-distance $\eps/2$ of some element of $V$. If we compute the dot product of $F$ with each $\sigma(x\cdot v)$, for $v\in\mathcal{C}$, to error $\eps/(2|\mathcal{C}|)$, 
this is sufficient to compute the $L_2$-norm of $F$ with every non-negative linear combination with coefficients summing to at most $1$ of these functions to error $\eps/2$. 
Taking a minimizer over such functions, which can be computed by a linear program in polynomial time, 
will give an appropriate answer.

To do this computation for each basis element $b$, we can use the empirical average of $b(x)y$, which gives an unbiased estimator. The number of samples required to achieve error $\eps'$ is $O(1/\eps')^2\sqrt{\Var(b(x)y)}$. The latter term, we can bound by Holder's inequality 
as $\E[y^4]^{1/4} \E[b(x)^4]^{1/4}$. The former term is $O(M+E)$ and the latter is $O(M)$. 
Thus, this computation can be done with an appropriate number of samples and time.

This completes our proof.

\section{Mixtures of Linear Regressions} \label{sec:mlr}

\subsection{Setup}

\begin{definition}[Mixtures of Linear Regressions] \label{def:mlr}
Given mixing weights $w_i \geq 0$ such that $\sum_{i=1}^k w_i=1$ and regression vector
$\beta_i \in \R^m$, $i \in [k]$, an $m$-dimensional {\em $k$-mixture of linear regressions} ($k$-MLR)
is the distribution on pairs $(x, y) \in \R^m \times \R$, where $x \sim N(0, I)$
and $y = \beta_i \cdot x + \nu$, where $\beta_i$ is sampled with probability $w_i$
and $\nu \sim N(0, \sigma^2)$ is independent of $x$ and $\sigma \geq 0$ is known.
\end{definition}

We study both density estimation and parameter estimation for $k$-MLRs.
\new{We will assume an upper bound \dnew{$R$} on the $\max_i \|\beta^{(i)}\|_2$.}
It will be convenient to assume that there is some known value
$p_{\min}$ so that $w_i \geq p_{\min}$ for all $i \in [k]$.

For density estimation, we prove:

\begin{theorem}[Density Estimation for $k$-MLR]\label{thm:mlr-density}
For a known minimum weight $p_{\min}$, degrees $d, d' \in \Z_+$, error parameter $\eps>0$ and upper bound $R>0$,
there is an algorithm that draws
$$N=\left(m^2\poly(k/p_{\min})+(2kd/p_{\min})^{O(d)}\right) \log(R/\sigma)\log\log(R/\sigma)+(2kd'/(\eps p_{\min}))^{O((d')^2 k^{1/d'}))}$$
samples from a $k$-MLR $Z$ on $\R^m \times \R$, runs in time $\poly\left(N,(2kd/p_{\min})^{d^2 k^{1/d}}\right)$, and
outputs a hypothesis $H$ that with high
probability satisfies $\dtv(H, Z) \leq \eps$.
\end{theorem}

For parameter estimation without noise ($\sigma = 0$), we show:

\begin{theorem}[Parameter Estimation for $k$-MLR, Noiseless Case]\label{thm:mlr-param-no-noise}
For a known minimum weight $p_{\min}$, degree-$d$, error $\eps$, upper bound $R$, and separation $\Delta$,
there is an algorithm that learns the $\beta_i$'s exactly using sample complexity
$$N=\left(m^2\poly(k/p_{\min})+(2kd/p_{\min})^{O(d)}\right) \log(R\log(m)/(p_{\min}\Delta))\log\log(R\log(m)/(p_{\min}\Delta)) $$
and runtime $\poly\left(N,(2kd/p_{\min})^{d^2 k^{1/d}}\right)$.
\end{theorem}

For parameter estimation with noise, we show:

\begin{theorem}[Parameter Estimation for $k$-MLR, Noisy Case]\label{thm:mlr-param-noise}
For a known weights $k$-MLR with minimum weight $p_{\min}$, degree-$d$, error $\eps$, upper bound $R$
and separation $\Delta/\sigma$ at least a sufficiently large polynomial in $\log(m)/p_{\min}$,
there exists an algorithm that solves parameter estimation to error $\eps$ with sample complexity
$$N=\left(m^2\poly(k/p_{\min})+(2kd/p_{\min})^{O(d)}\right)\log(R\log(m)/(p_{\min}\Delta))\log\log(R\log(m)/(p_{\min}\Delta))
\new{+ \tilde{O}(m) \poly(k/\eps)}$$
and runtime $\poly\left( N,(2kd/p_{\min})^{d^2 k^{1/d}} \right)$.
\end{theorem}

\dnew{The structure of this section is as follows:
Once we determine how to compute appropriate moment bounds (Section~\ref{ssec:moments-mlr}),
this will immediately provide a straightforward algorithm to solve these problems.
First compute second moments and use Proposition \ref{dimensionReductionProp} to reduce the problem to a $k$ dimensional one.
Then in those $k$-dimensions, compute the first $d$ moments to get a cover, and use that cover either in conjunction
with Proposition \ref{mixtureProp} to do density estimation or some relatively straightforward clustering
in order to do parameter learning.

Unfortunately, this simple technique will not be sufficient to obtain the efficiency that we desire.
This is because our sample complexity and runtime will be polynomial in $R^d$ and $R^{d^2 k^{1/d}}$, respectively,
when we would like a poly-logarithmic dependence. This is actually a relatively common problem with linear regression
problems. Learning the parameters in one-go will introduce too much error or require too high sample complexity.
Instead, the situation can be improved by learning only a rough approximation to the $\beta$'s and using
this approximation to learn iteratively better ones. A similar idea was used in~\cite{DKS19-lr}.

So, our refined overall strategy will be to learn a cover with relatively large error.
Using some elementary techniques, we can refine this cover to a relatively small list of potential hypotheses.
Now if these hypotheses are far enough apart (relative to $\sigma$ and the approximation error),
we will be able to figure out which hypothesis the mixing component of \emph{most} samples is close to.
However, it will be hard to tell whether $y$ approximates $\beta_i\cdot x$ or $\beta_j\cdot x$,
when $|(\beta_i-\beta_j)\cdot x|$ is small. This means that we will only be able to successfully cluster most points,
and will need our moment computation algorithm to work even if we have conditioned on only seeing the samples
that we can reliably cluster (which, fortunately, is determined by some known condition on $x$ alone).

As we will be needing to make use of several clusterings throughout this algorithm,
the following definition will be convenient.
\begin{definition}
An \emph{$(s,r)$-cover} is a set $\mathcal{C}$ of size at most $s$ such that for each $1\leq i\leq k$
there exists a $c\in \mathcal{C}$ with $\|c-\beta_i \|_2\leq r$.
\end{definition}
Note that we initially have a $(1,R)$-cover.
}

\subsection{Moment Computation} \label{ssec:moments-mlr}

\dnew{The first step in our algorithm is to compute the moment tensor}
$T = \sum_{i=1}^k w_i \beta_i^{\otimes 2d}$ to error $\delta$.
To do so, it suffices for every $|\alpha|=2d$ to compute $\sum_{i=1}^k w_i \beta_i^{\alpha}$ to error $\delta/m^d$.
We can do this given iid samples from $(x, y)$. However,
we will also want to be able to do it just given iid samples from $(x, y)$ {\em conditional on some known event $E$
on $x$ with probability at least $1/2$}.

\new{
\begin{lemma}\label{lem:mlr-moment-est}
Suppose that we have sample access to a $k$-MLR $X$ with parameters $(w_i, \beta_i)$, $i \in [k]$, where $\max_i \|\beta_i\|_2 \leq R$,
for a parameter $R>0$. There is an algorithm that, given $\delta>0$, $d \in \Z_+$,
and a multi-index $\bi \in [m]^{2d}$, draws $d^{O(d)}\dnew{(R+\sigma)^{4d}}/\delta^2$ conditional samples $(x, y)$ from any event $E(x)$ \dnew{depending on only the first $\ell$ coordinates and for which the algorithm is given oracle access}
with $\pr[E] \geq 1/2$, runs in sample-polynomial \dnew{plus $\poly((\ell d)^d/\delta)$} time,
and outputs an approximation $T_{\bi}$ of $(\sum_{i=1}^k w_i \beta_i^{\otimes 2d})_{\bi}$
with expected squared error $O(\delta^2)$.
\end{lemma}
}
\begin{proof}
Since $y=\beta\cdot x + \nu$, we can write
$$y^{2d} = \sum_{t=0}^{2d} \sum_{|\alpha|=t} \binom{2d}{t} \frac{|\alpha|!}{\alpha!} \beta^\alpha x^\alpha \nu^{2d-t} \;. $$
We would like to find a degree-$2d$ polynomial $p_\alpha$, such that for any polynomial $q$
of degree at most $2d$ the $x^\alpha$-coefficient of $q$ equals $\E_{x\sim N(0,I)}\left[p_\alpha(x)q(x)|E(x) \right]$.

Recall the normalized Hermite polynomials $h_n(x) = He_n(x)/\sqrt{n!}$ and define $h_\alpha(x) = \prod_{i=1}^m h_{\alpha_i}(x_i)$.
Note that we can write $q(x) = \sum_{|\alpha|\leq 2d} c_\alpha h_{\alpha}$ and that the $x^\alpha$-coefficient of $q$ is exactly
$\alpha! c_\alpha$.

For a vector $a$ whose entries are indexed by the $\alpha$ with $|\alpha|\leq 2d$,
we define $p_a(x):= \sum_{|\alpha|\leq 2d} a_\alpha h_\alpha(x)$.
Let $A$ be the symmetric matrix given by the quadratic form
$$
a^T Ab := \E_{x\sim N(0,I)}\left[p_a(x)p_b(x)|E(x)\right] \;.
$$
Note that without the conditioning, $A$ would just be the identity matrix.
We claim that with the conditioning, $A$ still has eigenvalues bounded away from $0$.

In particular, we have that $a^T Aa = \E[p_a^2(x)|E(x)]$.
We note that $\|p_a\|_{2} = \|a\|_2$.
We also note that, by anti-concentration of Gaussian polynomials (Theorem~\ref{thm:cw}),
except with probability at most $1/4$,
$|p_a(x)| \geq d^{-O(d)}\|p_a\|_{2}$.
In particular, even conditioned on $E(x)$, there is at least a $1/2$ probability that
$|p_a(x)|\geq d^{-O(d)}\|p_a\|_{2}$. This implies that
$$
a^T Aa = \E[p_a^2(x)|E(x)] \gg d^{-O(d)} \|p_a\|_{2}^2 = d^{-O(d)}\|a\|_2^2 \;.
$$
Thus, the smallest singular vector of $A$ is at least $d^{-O(d)}$.
Finally, we consider $p(x):= p_b(x)$, where $b=\alpha! A^{-1} e_{\alpha}$ and
$e_{\alpha}$ is the unit vector whose $\alpha$-entry is $1$ and whose other entries are $0$.
Then noting that $q=p_c$, we have that
$$
\E[p(x)q(x)|E(x)] = \E[p_b(x)p_c(x)|E(x)] = b^T A c = \alpha! e_{\alpha}^T A^{-1} A c = \alpha! c_\alpha \;,
$$
which is exactly the $x^\alpha$-coefficient of $q$.
By our bounds on the singular values of $A$, we have that $\|b\|_2 \leq d^{O(d)}$.

\dnew{We note that in order to run this algorithm, we will need to compute $b$ to sufficient accuracy.
This requires computing $A$ to some accuracy, which we can do by sampling (conditioned on $E$).
Fortunately, we only need to compute the entries of $A$ corresponding to monomials in the coordinates
on which $E$ and $\alpha$ depend. This can be done to sufficient accuracy with $\poly((\ell d/R)^d/\delta)$
samples to $x$ conditioned on $E$.}

Therefore, by linearity,
\dnew{
$$
\E\left[y^{2d}p_\alpha(x) \frac{\alpha!}{(2d)!} \mid E(x)\right]=\sum_{i=1}^k w_i \beta_i^\alpha.
$$
}
We can attempt to approximate this empirically given conditional samples.
The rate of convergence will depend on the variance,
which we can bound from above as
$$
2\E_{x\sim N(0,I)}[y^{4d}p^2_\alpha(x)] \leq d^{O(d)} \|y\|_2^{4d} \|p_\alpha\|_{2}^2 \leq d^{O(d)}\dnew{(R+\sigma)^{4d}} \;.
$$
Thus, we can approximate our tensor $T$ to error $\delta$ in $d^{O(d)} \dnew{(R+\sigma)^{4d}}m^d/ \delta^2$ samples.
\end{proof}

\new{
Using Lemma~\ref{lem:mlr-moment-est} to approximate each entry of
$\sum_{i=1}^k w_i \beta_i^{\otimes 2d}$ to appropriately high accuracy,
we can approximate the entire tensor $\sum_{i=1}^k w_i \beta_i^{\otimes 2d}$ within small $\ell_2$-error.

\begin{corollary}\label{cor:mlr-tensor-est}
\dnew{Given an $\ell$-dimensional subspace $H$ and} $N = d^{O(d)} \dnew{(R+\sigma)^{4d}} \dnew{\ell}^{2d}/\delta^2$
conditional samples from a $k$-MLR $X$, \dnew{conditioned on an $E$ with $\pr(E(x))\geq 1/2$
and $E$ depending only on $\ell$ linear functions of $X$,}
we can \dnew{in time $\poly(N,\ell^d)$} compute a tensor $T$ such that with high constant probability
it holds $\| T - \sum_{i=1}^k w_i \dnew{\pi_H}(\beta_i)^{\otimes 2d} \|_2^2 \leq \delta^2$.
\end{corollary}
\begin{proof}
\dnew{By performing an appropriate rotation, we can assume that $E$
depends only on the first $\ell$ coordinates and $H\subset \R^{2\ell}$.}
We take $N = d^{O(d)} R^{2d} \dnew{\ell}^{2d}/\delta^2$ noisy samples from $X$,
and consider the tensor $T = (T_{\bi})$, $\bi \in [\dnew{\ell}]^{2d}$, as our approximation to
$\sum_{i=1}^k w_i \dnew{\pi_H}(\beta_i)^{\otimes 2d}$.
By Lemma~\ref{lem:mlr-moment-est}, we have that
$$\E \left[ \left\| T - \littlesum_{i=1}^k w_i \dnew{\pi_H}(\beta_i)^{\otimes 2d} \right\|_2^2 \right] \leq \dnew{\ell}^{2d} (\delta/\dnew{\ell}^d)^2
= O(\delta^2) \;.$$
The corollary follows from Markov's inequality.
\end{proof}
}

\dnew{
As we will need to be doing this many times in the several rounds of our algorithm,
we will want to ensure that the above guarantee holds with high probability rather than constant probability.
This is  easy to do with independent repetition.

\begin{corollary}\label{MLRTensorEstHighProbCor}
Given an $\ell$-dimensional subspace $H$ and $N = d^{O(d)} (R+\sigma)^{4d}\log(1/\tau) \ell^{2d}/\delta^2$
conditional samples from a $k$-MLR $X$, conditioned on an $E$ with $\pr(E(x))\geq 1/2$
and $E$ depending only on $\ell$ linear functions of $X$,
we can in time $\poly(N,\ell^d)$ compute a tensor $T$ such that with probability at least $1-\tau$ it holds
$\| T - \sum_{i=1}^k w_i \pi_H(\beta_i)^{\otimes 2d} \|_2^2 \leq \delta^2$.
Furthermore, this works even if a sample is erroneous with probability $N/(10\log (1/\tau))$.
\end{corollary}
\begin{proof}
We run the algorithm from Corollary \ref{cor:mlr-tensor-est} $100\log(1/\tau)$ times with error $\delta/3$.
With probability at least $1-\tau$, a majority of the tensors $T_i$ computed are within $\delta/3$ of
$T_0 := \sum_{i=1}^k w_i \beta_i^{\otimes 2d}$ in $L_2$-norm. Note that the erroneous samples will only affect
one tenth of our trials, and so will not change this. If this is the case, our algorithm can return any $T_i$
that is within $2\delta/3$ of at least half of the other $T_i$'s.

Such a $T_i$ must exist because any close $T_i$ will be at most this far from any other close $T_i$.
Additionally, any $T_i$ that is this close to a majority, will be distance at most $2\delta/3$ from some $T_i$
at distance at most $\delta/3$ from $T_0$. Therefore, by the triangle inequality, any such $T_i$
will have error at most $\delta$.
\end{proof}
}

\dnew{
\subsection{Iteration}

Our overall algorithm will depend on obtaining iteratively better covers of our $\beta_i$'s.
The goal of the next few sections will be to show that if we have a $(k,r)$-cover,
with $r$ substantially larger than $\sigma$, we can (with tiny probability of failure)
use this to compute a $(k,r/2)$-cover. This procedure will break down further into the following steps:
\begin{enumerate}
\item Clustering: We will have an algorithm that assigns to most sample points a cluster,
so that almost all samples from the same mixing component are assigned to the same cluster,
and so that each cluster has an associated center that is not too far from the corresponding $\beta_i$.
If we then subtract from the $y$-value of such a sample, the expected $y$-value based on its cluster center,
we can reduce ourselves to considering samples from a mixture of linear regressions
with parameters not too much larger than $r$.

\item Dimension Reduction: Taking samples from this simulated mixture, we can use Proposition \ref{dimensionReductionProp}
to reduce to a $k$-dimensional subspace.

\item Rough Cover: Computing more moments within this subspace, we can use Proposition \ref{coverApplicationProp}
to compute an $(s,r')$-cover for $r'=r/\poly(k)$. Unfortunately, $s$ will usually be substantially larger than $k$ here.

\item Cover Refinement: We can throw away many of the points in this cover for which there are not enough samples with $y\approx c\cdot x$.
The remaining points can be grouped into at most $k$ groups each with radius at most $r/2$, giving our final new cover.
\end{enumerate}

In the end we will prove the following lemma:

\begin{lemma}\label{MLRCoverIterationLem}
There is an algorithm that given a $(k,r)$-cover for some known $r \gg k\sigma/p_{\min}$ (with a sufficiently large implied constant)
and a $\tau >0$, takes at most $N=(m^2\poly(k/p_{\min})+ (2kd/p_{\min})^{O(d)})\log(1/\tau)$ samples and
$\poly(N,(2kd/p_{\min})^{d^2 k^{1/d}})$ time and with probability at least $1-\tau$ returns a $(k,r/2)$-cover.
\end{lemma}

Applying this repeatedly gives the following:

\begin{corollary}\label{MLRSmallCoverCor}
Given sample access to a mixture of linear regressions $X$ with $w_i\geq p_{\min}$ and $\|\beta_i\|_2\leq R$ for all $i \in [k]$,
and $r$ at least a sufficiently large multiple of $k\sigma/p_{\min}$, there exists an algorithm that takes
$N=\left(m^2\poly(k/p_{\min})+ (2kd/p_{\min})^{O(d)}\right)\log(R/r)\log\log(R/r)$ samples and $\poly(N,(2kd/p_{\min})^{d^2 k^{1/d}})$ time,
and with large constant probability computes a $(k,r)$-cover.
\end{corollary}

\subsection{Clustering}

Here we show that given a cover, we can use this to compute a clustering on most of the points.

\begin{lemma}\label{MLRClusteringLem}
Given a $(k,r)$-cover $\mathcal{C}$ and a sufficiently small parameter $\eta>0$,
there exists a polynomial time computable condition $E(x)$ with probability at most $1/2$,
a polynomial time computable function $f:\R^{m+1}\rightarrow \mathcal{C}$
and an (unknown) function $g:[k]\rightarrow \mathcal{C}$ such that
\begin{enumerate}
\item For all $i\in [k]$, $\|\beta_i-g(i)\|_2 < O(k^3(r+\sigma)\log(1/\eta))$.
\item If $(x,y)\sim X$, then conditioned on $E(x)$, we have that $(x,y-f(x,y)\cdot x)$ is $\eta$-close
in total variation distance to $(x,y')$ conditioned on $E(x)$, where $(x,y')$ is the mixture of linear regressions
that has $y'=x\cdot(\beta_i-g(i))+\nu$ with probability $w_i$, for each $i\in [k]$.
\end{enumerate}
Furthermore, $E(x)$ depends only on the inner products of $x$ with the elements of $\mathcal{C}$.
\end{lemma}

The basic idea of the proof is that if $(x,y)$ comes from a component with $\beta\approx c\in \mathcal{C}$,
then $y$ ought to be (with high probability) close to $c\cdot x$. This should give us a unique possible $c$
that $y$ came from, unless either there is another $c'\in\mathcal{C}$ close to $c$,
or if $x$ is unusually close to being to orthogonal to $c-c'$. In the former case, we declare that such $c$ and $c'$
are in the same cluster and don't distinguish between points close to one and points close to the other.
For the latter case, we note that $x\cdot(c-c')$ is small for $c,c'\in\mathcal{C}$ with $\|c-c'\|_2$ large
only with small probability, and we define our event $E$ to exclude such values of $x$.

\begin{proof}
Call two elements $a,b\in\mathcal{C}$ \emph{close} if $\|a-b\|_2 \leq 10k^2(r+\sigma)\log(1/\eta)$.
Declare that two elements of $\mathcal{C}$ are in the same \emph{cluster}
if we can reach one from the other by a chain of close pairs. Since this chain can have length at most $k$,
we know that each cluster has diameter at most $O(k^3(r+\sigma)\log(1/\eta))$.
To each cluster we designate one of the elements of $\mathcal{C}$
in that cluster to be the \emph{representative} of that cluster.

We now let $E(x)$ be the set of $x$ values such that for all pairs $a,b\in\mathcal{C}$,
either $a$ and $b$ are close or $|x\cdot (a-b)| > 2(r+\sigma)\log(1/\eta)$.
We note that for any not-close pair, the probability of this happening is at most $1/(2k^2)$,
and therefore, the probability of $E(x)$ is at most $1/2.$

The function $f(x,y)$ is defined by first finding the element $a\in\mathcal{C}$ minimizing $|y-a\cdot x|$,
and letting $f$ be the representative of the cluster of $a$. For each $i$, we will let $g(i)$ be the representative
of the cluster of the element $a\in\mathcal{C}$ with $\|\beta_i-a\|_2$ as small as possible.
Note that since $\|\beta_i-a\|_2\leq r$ and since clusters have bounded diameter,
this implies that $\|\beta_i-g(i)\|_2 \leq O(k^3(r+\sigma)\log(1/\eta))$ by the triangle inequality.

It remains to prove our second statement about the distribution of $(x,y-f(x,y)\cdot x)$.
This will follow from the claim that if $(x,y)$ is drawn from the $i$-th component of the mixture,
then conditioned on $E(x)$ the probability that $f(x,y)\neq g(i)$ is at most $\eta$.
To show this, we will show unconditionally that if $y=\beta_i\cdot x+\nu$, then
the probability that $E(x)$ holds and $f(x,y)\neq g(i)$ is at most $\eta/2$.

Let $a$ be the closest element of $\mathcal{C}$ to $\beta_i$, so that in particular $\|\beta_i-a\|_2<r$.
We note that $y-a\cdot x = (\beta_i-a)\cdot x+\nu$ is a Gaussian with standard deviation less
than $r+\sigma$, and thus except with probability $\eta/2$ we have that $|y-a\cdot x|\leq (r+\sigma)\log(1/\eta).$
We claim that if this is the case and if $E(x)$ holds, then $f(x,y)$ will be $g(i)$. In particular, we need to
show that if this holds and if $E(x)$ also does, then $|y-a\cdot x|$ will be less than $|y-b\cdot x|$,
for all $b\in \mathcal{C}$ not close to $a$ (note that this is sufficient, as it will imply that the best $b$
must either be $a$ or in the same cluster). However, for $b$ not close to $a$, since $E(x)$ holds, we have that
$$
|(y-b\cdot x)-(y-a\cdot x)| = |(b-a)\cdot x| > 2(r+\sigma)\log(1/\eta) \;.
$$
Thus, by the triangle inequality
$$
|y-b\cdot x| >2(r+\sigma)\log(1/\eta)-|y-a\cdot x| > (r+\sigma)\log(1/\eta)\geq |y-a\cdot x| \;.
$$
This completes our proof.
\end{proof}
}

\dnew{
\subsection{Dimension Reduction}

Here we prove the following lemma:
\begin{lemma}\label{MLRDimensionReductionLem}
Given an explicit event $E$ with probability at least $1/2$
and sample access to a mixture of linear regressions conditioned on $E(x)$
with $\max_i \|\beta_i \|_2 \leq r$ and $w_i\geq p_{\min}$ for all $i \in [k]$,
there is an algorithm that given parameters $r+\sigma>\eps>0 ,1/2> \tau>0$,
uses $N=O(((r+\sigma)/\eps)^4 m^2 p_{\min}^{-2} \log(1/\tau))$ samples
and $\poly(N)$ time and computes a dimension at most $k$ subspace $H$, such
that with probability at least $1-\tau$ every $\beta_i$ is within $\ell_2$-distance $\eps$ of $H$.
Furthermore, this works even if a sample is erroneous with probability $N/(10\log (1/\tau))$.
\end{lemma}
\begin{proof}
We use Corollary \ref{MLRTensorEstHighProbCor} to compute with probability at least $1-\tau$
an estimate to the tensor $T=\sum_{i=1}^k w_i \beta_i^{\otimes 2}$ with error at most $\delta < \eps^2 p_{\min}$.
Then, we use Proposition \ref{dimensionReductionProp} to compute $H$.
\end{proof}
}

\dnew{
\subsection{Cover}

Here we use our technology to get a cover.
\begin{lemma}\label{MLRLargeCoverLem}
Given an explicit event $E$ with probability at least $1/2$ and sample access
to a mixture of linear regressions conditioned on $E(x)$ with $\max_i \|\beta_i\|_2\leq r$
and $w_i\geq p_{\min}$ for all $i \in [k]$, and an $\ell$-dimensional subspace $H$,
there is an algorithm that given parameters $d$ and $r+\sigma>\eps>0 ,1/2> \tau>0$,
uses $N=(2k\ell d(r+\sigma)/\eps)^{O(d)} \poly(1/p_{\min})\log(1/\tau)$ samples
and computes with probability at least $1-\delta$ an $(s,\eps)$-cover of the set of
$\pi_H(\beta_i)$ with $s=(2kd\ell((r+\sigma)/\eps)/p_{\min})^{O(d^2 k^{1/d})}$ in $\poly(N,s)$ time.
Furthermore, this works even if a sample is erroneous with probability $N/(10\log (1/\tau))$.
\end{lemma}
\begin{proof}
Using Corollary \ref{MLRTensorEstHighProbCor}, we can with probability $1-\tau$ compute an approximation
to the tensor $T=\sum_{i=1}^k w_i \pi_H(\beta_i)^{\otimes 2d}$ with error at most
$(k\ell(r+\sigma)/\eps)^{-\Omega(d)} \poly(p_{\min})$ with sufficiently large constants in the exponent.
Applying Proposition \ref{coverApplicationProp} yields our result.
\end{proof}
}

\dnew{
\subsection{Cover Refinement}

Here we show that, given a cover, we can use a small number of samples reduce it to a smaller cover.
The basic idea will be to come up with a smaller set of plausible hypotheses
(those for which $y\approx c\cdot x$ for a reasonable fraction of samples). It is not hard to show
that given a large enough sample set, with high probability all plausible hypotheses
will be close to some $\beta_i$. From there one can cluster together hypotheses that are nearby.
Formally, we show:

\begin{lemma}\label{MLRCoverRefinementLem}
Suppose that we have a mixture of linear regressions $X$ with parameters $w_i,\beta_i$ for $1\leq i\leq k$ and $w_i \geq p_{\min}$.
Suppose furthermore that we are given an $(s,r)$-cover $\mathcal{C}$ of $X$.
Then there is an algorithm which takes $N=O(\log(s/\tau)/p_{\min})$ samples from $X$,
runs in $\poly(N,s,m,k)$ time, and with probability $1-\tau$ computes a $(k,O(k(r+\sigma)/p_{\min}))$ cover of $X$.
\end{lemma}
\begin{proof}
Take $N=O(\log(s/\tau)/p_{\min})$ samples with a sufficiently large implied constant.

Call a hypothesis $c\in\mathcal{C}$ \emph{good} if at least a $p_{\min}/4$-fraction of our $N$ samples
satisfy $|y-c\cdot x| \leq 2(r+\sigma)$. We note that if $\|\beta_i-c\|_2 \leq r$ for some $i$,
then with probability at least $p_{\min}/2$ over samples from $X$,
we have that the sample is from the $i$-th component and $|y-c\cdot x| \leq 2(r+\sigma)$.
Therefore, with probability at least $1-\eta/2$, every such hypothesis is good.

On the other hand, suppose that we have a hypothesis vector $c$ for which
$\|\beta_i-c\|_2 >  10(r+\sigma)/p_{\min}$ for all $i \in [k]$.
Then no matter which part of the mixture we are drawing from, $y-c\cdot x$ is distributed
as a normal distribution with standard deviation at least $\|\beta_i-c\|_2 > 10(r+\sigma)/p_{\min}$.
This means that the probability of it being less than $2(r+\sigma)$ is at most $p_{\min}/5$.
Therefore, with probability at least $1-\eta/2$, no such hypothesis $c$ is good.

Hence, with probability at least $1-\eta$ we have that every hypothesis $c$ that is within distance
$2(r+\sigma)$ of some $\beta_i$ is good, and all good hypotheses are within $10(r+\sigma)/p_{\min}$ of some $\beta_i$.
We declare two good hypotheses to be \emph{close} if they are within $20(r+\sigma)/p_{\min}$ of each other,
and in the same cluster if they are connected by some chain of close hypotheses.
Note that since any two hypotheses within $10(r+\sigma)/p_{\min}$ of the same $\beta_i$ are close,
these chains can have length at most $k$, and so each cluster has diameter $O(k(r+\sigma)/p_{\min})$.
This also implies that there are at most $k$ clusters.

We return as our cover one representative hypothesis from each cluster
(plus a number of other random elements to pad the size out to $k$).
We note that every $\beta_i$ by assumption is $r$-close to some good hypothesis,
and thus must be within distance $O(k(r+\sigma)/p_{\min})$ of one of our representatives.
This completes the proof.
\end{proof}
}

\dnew{
\subsection{Proof of Lemma \ref{MLRCoverIterationLem}}
The proof now follows from the machinery that we have built up.
\begin{proof}
Let our cover be $\mathcal{C}$.

We begin by applying Lemma \ref{MLRClusteringLem} with $\eta$ a sufficiently small polynomial
in $(p_{\min}/m)(2kd)^{-d}$ to produce an event $E(x)$ with probability at least $1/2$
and a method for simulating samples of $(x,y')$ conditioned on $E$ (up to $\eta$ error in total variation distance),
where $y'$ is a mixture of linear regressions with mixing weights $w_i$
and parameters $\beta_i-c_i$, for some $c_i\in\mathcal{C}$. We then use these samples
with Lemma \ref{MLRDimensionReductionLem} to compute (with probability at least $1-\tau/10$)
a $k$-dimensional subspace $H$, such that all of the $\beta_i-c_i$ are within distance $\eps$ of $H$,
for $\eps$ a sufficiently small multiple of $rp_{\min}/k$. We use more simulated samples
along with Lemma \ref{MLRLargeCoverLem} to compute a $((2kd/p_{\min})^{O(d^2 k^{1/d})},\eps)$-cover
of the $\pi_H(\beta_i-c_i)$, which will be a $((2kd/p_{\min})^{O(d^2 k^{1/d})},2\eps)$-cover of the $\beta_i-c_i$.
If we call this cover $\mathcal{C'}$, then the set of points $a+b$, for $a\in \mathcal{C}, b\in\mathcal{C'}$
will be a $((2kd/p_{\min})^{O(d^2 k^{1/d})},2\eps)$-cover of the $\beta_i$'s.
Finally, we apply Lemma \ref{MLRCoverRefinementLem} to get a $(k,O(k(\eps+\sigma)/p_{\min}))$-cover
(which is a $(k,r/2)$-cover) with probability at least $1-\tau/10$.

It is straightforward to verify that this procedure fits within our bounds for runtime,
sample complexity and probability of error, completing the proof.
\end{proof}
}

\dnew{
\subsection{Density Estimation}

Here we prove Theorem \ref{thm:mlr-density}.
\begin{proof}
We begin by applying Corollary \ref{MLRSmallCoverCor} to obtain a $(k,O(k\sigma/p_{\min}))$-cover.
As in the proof of Lemma \ref{MLRCoverIterationLem}, we use $(2k/(\eps p_{\min}))^{O(d')}$ additional samples
to compute an $(s,\eps\sigma)$-cover with $s=(2kd'/(\eps p_{\min}))/p_{\min})^{O((d')^2 k^{1/d'})}$
in $\poly(s)$ time. 
We then have that $X$ is $O(\eps)$-close in total variation distance to a mixture of the linear regressions
with parameters given by the terms of this cover. Using Proposition \ref{mixtureProp},
we can learn an $O(\sqrt{\eps \log(s/\eps)})$-approximation to $X$.

Substituting $\eps^2/(d' \log(2kd'/(\eps p_{\min})))$ for $\eps$ yields the result.
\end{proof}
}

\dnew{
\subsection{Parameter Estimation}

Here we prove Theorems \ref{thm:mlr-param-noise} and \ref{thm:mlr-param-no-noise}.
\begin{proof}
We begin by applying Corollary \ref{MLRSmallCoverCor} to obtain a $(k,c\Delta/(k^3\log(mk/p_{\min}))$-cover,
for a sufficiently small constant $c>0$. We then apply Lemma \ref{MLRClusteringLem},
with $\eta$ a sufficiently small polynomial in $mk/p_{\min}$. We note that since the $\beta_i$'s
are separated by at least $\Delta$, while each sample in which $E(x)$ holds (ignoring probability $\eta$ events)
has $\|\beta_i - f(x,y) \|_2< \Delta/3$. This implies that any two samples (again ignoring probability $\eta$ events)
will have $f(x,y)$-values within $2\Delta/3$ of each other if and only if they come from the same component of the mixture.

Taking $m^2\poly(k/p_{\min})$ samples (and noting that this probability $\eta$ of error likely never happens),
and this ability to sort the samples for which $E(x)$ holds by component, we can use Corollary \ref{cor:mlr-tensor-est}
to estimate each $\beta_i$ to $\ell_2$-error $\Delta/k$. From this warm start, we can use the algorithm of~\cite{KC19}
to improve this to error $\eps$.

Alternatively, if $\sigma=0$, $m$ samples from each component correctly identified
can be used along with linear algebra to solve exactly for the $\beta_i$'s.
\end{proof}
}

\subsection{Sample Complexity Lower Bound for Mixtures of Linear Regressions} \label{sec:lb-mlr}

\new{In this subsection, we show that if the pairwise separation $\Delta$ is sufficiently small,
the problem of parameter estimation for MLRs with noise requires a sub-exponential in $k$ number of samples.}

We consider the $m=1$ case of a linear regression. Let a $\nu$-sparse $\sigma^2$-variance Gaussian be a pseudo-distribution supported on points $x \equiv \theta\pmod{\nu}$ for some constant $\theta$ assigning probability mass to $x$ equal to $g(x/\sigma)\nu$ where $g(x) = \frac{1}{\sqrt{2\pi}}e^{-x^2/2}$ is the Gaussian density. Note that this will not in general be a normalized probability distribution.

\begin{lemma}
Let $X$ be a $\nu$-sparse variance-$\sigma_1^2$ Gaussian and $Y=N(0,\sigma_2^2)$ for $\nu \ll \min(\sigma_1,\sigma_2)$. Then the convolution $X \ast Y$ is $\exp(-\Omega(\min(\sigma_1,\sigma_2)/\nu)^2))$-close to $N(0,\sigma_1^2+\sigma_2^2)$ in $L^1$.
\end{lemma}
\begin{proof}
We begin by considering the Fourier transforms. We have that $\widehat{(X\ast Y)} = \hat{X}\hat{Y}$. Now $\hat{Y}(\xi) = e^{-\sigma_2^2 \xi^2/2}.$ Now $X(x) = \Sha_\nu(x) g(x)$ where $\Sha_\nu(x) = \sum_{y\equiv \theta\pmod{\nu}} \nu \delta(x-y)$. This tells us that
\begin{equation}\label{FTEqn}\hat{X} = \hat{\Sha_\nu}\ast \hat{g}= \left(\sum_{y\equiv 0\pmod{1/\nu}} \delta(\xi-y)e^{2\pi i \theta y}\right) \ast \left(e^{-\sigma_1^2 \xi^2/2} \right) = \sum_{y \equiv \xi \pmod{1/\nu}} e^{-\sigma_1^2 y^2/2}e^{2\pi i (\xi-y)\theta}.\end{equation}

Now assuming that $\nu \ll \sigma_1$, we have that the sum on the right of Equation \eqref{FTEqn} has at most one term more than $\exp(-\Omega(\sigma_1/\nu)^2)$, and that all remaining terms together contribute at most $\exp(-\Omega(\sigma_1/\nu)^2)$. Therefore we have that
$$
\hat{X}(\xi) = \exp(-\sigma_1^2 [\xi]^2/2)\exp(2\pi i (\xi-[\xi])\theta)\pm \exp(-\Omega(\sigma_1/\nu)^2),
$$
where $[\xi]$ is the nearest multiple of $1/\nu$ to $\xi$. Plugging in $\xi=0$, we find that the total mass of $X$ is $1 \pm \exp(-\Omega(\sigma_1/\nu)^2)$. Returning to our original $X\ast Y$ we have that
$$
\widehat{(X\ast Y)} = \exp(-\sigma_2^2 \xi^2/2)\exp(-\sigma_1^2 [\xi]^2/2)\exp(2\pi i (\xi-[\xi])\theta)\pm \exp(-\Omega(\sigma_1/\nu)^2).
$$
Now $[\xi]=\xi$ unless $|\xi| \geq 1/(2\nu)$. In that case, $\exp(-\sigma_2^2 \xi^2/2) = \exp(-\Omega(\sigma_2/\nu)^2).$ Therefore, we have that for all $\xi$,
$$
\widehat{(X\ast Y)} = \exp(-(\sigma_1^2+\sigma_2^2)\xi^2/2)\pm \begin{cases}\exp(-\Omega(\sigma_1/\nu)^2) & \textrm{if }|\xi| \leq 1/(2\nu)\\ \exp(-\Omega(\sigma_2 \xi)^2) & \textrm{else} \end{cases}.
$$
Note that the first term is just the Fourier transform of $N(0,\sigma_1^2+\sigma_2^2)$. The latter term can be seen to have total integral at most $(1/\nu)\exp(-\Omega(\min(\sigma_1,\sigma_2)/\nu)^2)$. This means that $X\ast Y$ is $(1/\nu)\exp(-\Omega(\min(\sigma_1,\sigma_2)/\nu)^2)$-close to $N(0,\sigma_1^2+\sigma_2^2)$ in $L^\infty$.

However, since $X$ is nearly normalized by the above, the normalized version of $X\ast Y$ (namely $X\ast Y/ |X\ast Y|_1$) is also $(1/\nu)\exp(-\Omega(\min(\sigma_1,\sigma_2)/\nu)^2)$-close to $N(0,\sigma_1^2+\sigma_2^2)$ in $L^\infty$. However, the $L^1$ distance between two distributions is equally divided between the amount that one is bigger than the second and the amount that the second is bigger than the first. Therefore, if $f(x) = \frac{1}{\sqrt{2\pi \sigma_2^2}} \exp(-x^2/(2\sigma_2^2))$ is the probability density function of $N(0,\sigma_1^2+\sigma_2^2)$, we have that the $L^1$ distance between it and $X\ast Y/ |X\ast Y|_1$ is at most
$$
\int \min(f(x), (1/\nu)\exp(-\Omega(\min(\sigma_1,\sigma_2)/\nu)^2))dx.
$$
This is easily seen to be $\exp(-\Omega(\min(\sigma_1,\sigma_2)/\nu)^2))$, completing our theorem.
\end{proof}

Next consider the pseudodistribution where $X\sim N(0,1)$ and $y=\sigma s x + N(0,\sigma)$ where for some $\nu$ and $\theta$, $s$ is taken to be $n\nu+\theta$ (for integer $n$) with probability $\frac{\nu}{\sqrt{2\pi}}e^{-(n\nu+\theta)^2/2}$ (namely $s$ is distributed as a $\nu$-sparse variance-$1$ Gaussian). We note that for given $x$, $\sigma\cdot s\cdot x$ is distributed as a $\nu \sigma x$-sparse variance-$(\sigma x)^2$ Gaussian. Therefore, by our Lemma, if $|x| \ll 1/\nu$, then the distribution of $y$ conditioned on that value of $x$ is $\exp(-\Omega(\min(1/\nu, 1/(\nu x)))^2)$-close in $L^1$ to $N(0,\sigma_1^2 x^2+\sigma_2^2)$. Therefore, integrating over $x$, the distribution $(x,y)$ is close to the distribution where $(y|x) \sim N(0,\sigma^2 x^2+\sigma^2)$ with total $L^1$ error at most $\exp(-\Omega(1/\nu))$.

Now, you can think of this pseudodistribution as a mixture of linear regressions, except that the number of mixing terms in infinite and that it is not normalized. However, it assigns $s$ to be a value bigger than $1/\sqrt{\nu}$ with probability only $\exp(-\Omega(1/\nu))$. Therefore, removing these out and renormalizing, we get an honest mixture of $O(\nu^{-3/2})$ linear regressions that is $\exp(-\Omega(1/\nu))$-close to $(y|x) \sim N(0,\sigma^2 x^2+\sigma^2)$ in total variational distance.

However, if we do this with $\theta=0$ vs. $\theta = \nu/2$, no two parameters in the supports of these mixtures are closer than $\nu/2$ of each other. Letting $\nu=k^{-2/3}$, this shows that it is impossible to learn the individual parameters of a mixture of $k$ linear regressions to error better than $\sigma/k^{2/3}$ with only $\exp(o(k^{2/3}))$ samples.

\section{Mixtures of Hyperplanes} \label{sec:mh}

\subsection{Setup}

\begin{definition}[Mixtures of Hyperplanes] \label{def:mh}
An $m$-dimensional {\em $k$-mixture of hyperplanes}
is a distribution $X$ on $\R^m$ with density function $F(x) = \sum_{j=1}^k w_j N(0, \Sigma_j)$,
where for $j \in [k]$, we have that $w_j \geq 0$, $\sum_{j=1}^k w_j = 1$,
and $\Sigma_j = I - v_j v_j^T$ with $v_j \in \R^m$ and $\|v_j\|_2=1$.
\end{definition}

We study parameter estimation under $\Delta$ pairwise separation for the $v_i$'s.
\new{Specifically, we will assume that we know some $\Delta>0$ such that for all $i \neq j$
and $\sigma_i, \sigma_j \in \{ \pm 1\}$, we have that $ \| \sigma_i v_i- \sigma_j v_j\|_2 \geq \Delta$.
Note that the $v_i$'s are only identifiably up to sign, which motivates this definition.

For simplicity of the exposition, we will assume uniform weights in this section, i.e., that all the $w_i$'s are $1/k$.
The goal of parameter learning in this context is to output a list of unit vectors $\{\tilde{v}_j\}_{j=1}^k$ such that
there is a permutation $\pi\in \mathbf{S}_k$ and a list of signs $\sigma_j \in \{ \pm 1\}$ for which
$v_j  = \sigma_j \tilde{v}_{\pi(j)}$ for all $j \in [k]$.
}

Our main result in this section is the following theorem:

\begin{theorem}[Parameter Estimation for $k$-mixtures of Hyperplanes]\label{thm:mh-param}
There is an algorithm that on input $d \in \Z_+$, \dnew{with $d=O(\log(k))$}, and
sample access to a uniform $k$-mixture of hyperplanes on $\R^m$ with pairwise separation $\Delta>0$,
the algorithm outputs the target parameter vectors using
$N = \dnew{O(k/\Delta)^{O(d)}}+O(m^2) \poly(k\log(m)/\Delta)$ samples and
$\poly(N)+m^2 \log(\dnew{\log(m)}/\Delta) (kd )^{O(d^2 k^{1/d})}$ time.
\end{theorem}

\subsection{Moment Computation}

\new{
The following lemma shows that we can efficiently approximate the tensor
$\sum_{i=1}^k w_i v_i^{\otimes 2d}$ to small error:

\begin{lemma}\label{lem:mh-moment-est}
Suppose that we have sample access to $X = \sum_{i=1}^k w_i N(0, I-v_i v_i^T)$.
There is an algorithm that, given $\delta>0$, and $d \in \Z_+$, draws $(m d)^{O(d)}/\delta^2$ samples from $X$,
runs in sample-polynomial time, and outputs an approximation $T$ of the tensor $\sum_{i=1}^k w_i v_i^{\otimes 2d}$
with expected squared error $O(\delta^2)$.
\end{lemma}
}

\begin{proof}
For this section, it suffices to assume that each $v_i$, $i \in [k]$ is a vector with $\|v_i\|_2\leq 1$
and not necessarily equal to $1$. We note that the $2d^{th}$ moment tensor of $N(0,\Sigma)$ is given by
$$
\E_{X\sim N(0,\Sigma)}[X^{\otimes 2d}] = (2d-1)!!\mathrm{Sym} (\Sigma^{\otimes d}),
$$
where $\mathrm{Sym}(T)_{a_1,\ldots,a_{2d}}$ is the symmetrization
$\frac{1}{(2d)!} \sum_{\pi\in S_{2d}} T_{a_{\pi(1)},a_{\pi(2)},\ldots,a_{\pi(2d)}}$.

From here, it is easy to see that if $X\sim N(0,I-vv^T)$ that
$$
v^{\otimes 2d} = \sum_{t=0}^d \binom{d}{t} \mathrm{Sym}\left((I-vv^T)^{\otimes t}\otimes I^{\otimes (d-t)} \right)
= \sum_{t=0}^d \binom{d}{t} \frac{1}{(2t-1)!!}\mathrm{Sym}\left( \E[X^{\otimes 2t}]\otimes I^{\otimes (d-t)}\right) \;.
$$
Thus, by linearity,
$$
\sum_{i=1}^k w_i v_i^{\otimes 2d} = \sum_{t=0}^d \binom{d}{t} \frac{1}{(2t-1)!!}\mathrm{Sym}\left( \E[X^{\otimes 2t}]\otimes I^{\otimes (d-t)}\right) \;.
$$
\new{Using the same arguments as in previous subsections to bound the variance of the relevant term,}
this quantity can be efficiently computed to $\ell_2$-error $\delta$ empirically using $(dm)^{O(d)}/\delta^2$ samples.

\end{proof}

\subsection{Dimension Reduction}

By Proposition \ref{dimensionReductionProp}, if we compute this for $d=1$ and $\delta$ a sufficiently small multiple of $\eta^2/k$ (which can be done in $O(m^2 k^2 \eta^{-4})$ samples), we can compute a subspace $U$ so that all $v_i$'s are within $\eta/2$ of $U$. Taking the projection of $X$ onto $U$, we are left with
$$
\sum_{i=1}^k w_i N(0,I-\pi_U(v_i)\pi_U(v_i)^T).
$$

\subsection{Cover}

Next, we can take $(2dk/\eta)^{O(d)}$ samples and compute an approximation to $\sum \pi_U(v_i)^{\otimes 2d}$ with error at most $(\eta/(2dk))^{Cd}$. We could then use Proposition \ref{coverApplicationProp} to produce a set of size $(2dk/\eta)^{O(d^2 k^{1/d})}$ so that each $v_i$ is guaranteed to be within $\eta$ of some hypothesis. However, this will prove to be more expensive than necessary. Instead for some $\eps>\eta$, we can in compute a cover of size $S=(2dk/\eps)^{O(d^2 k^{1/d})}$ in $\poly(S)$ time.

\subsection{Clustering}

\dnew{Given what we have so far, we could just take $\eta$ substantially smaller than $\Delta$,
and get a cover at enough granularity to distinguish our components. However, this will require
$\Delta^{-d^2 k^{1/d}}$ time, which we would like to avoid. Instead, we will have an iterative process
by which we locate which hypotheses are actually close to our parameters
and use this to iteratively refine our clusters. In particular, by seeing which hypotheses are nearly
orthogonal to many samples, we can figure out which ones are plausible.
By naively clustering the plausible hypotheses, we can find a size $k$ cover of substantially larger radius.
We can then use our existing approximation to the higher moments
to get more precise covers only near these few hypotheses. By iterating this technique,
we can eventually find a small cover of radius less than $\Delta/\log(m)$.
This can be used to reliably classify which component various samples actually came from,
and if we find $m$ samples from the same component, linear algebra
can be used to exactly compute the corresponding $v$ up to sign.}

Suppose that we have a set $\mathcal{C}$ of $S$ samples with the guarantee that each $v_i$
is within distance $\eps$ of some element of $\mathcal{C}$. We can take $N=Ck\log(S)$ samples,
for $C$ a sufficiently large constant. We then call a hypothesis $c\in \mathcal{C}$ \emph{good}
if, for at least a $1/(2k)$-fraction of these samples,
we have that $|c\cdot x| \leq 10\sqrt{\eps}$. Note that if $c$ is within distance $\eps$ of some $v_i$,
then it will be good with high probability, because the samples from that part of the mixture
will mostly satisfy the necessary condition. Furthermore, if $c$ is not within $O(k \sqrt{\eps})$
of any $v_i$, then with high probability this will hold for at most a $1/(10k)$ fraction of the samples
from each component, and hence $c$ will not be good. Thus, with high probability,
all $c$ within $\eps$ of some $v_i$ are good, while all good $c_i$ are within $O(k\eps)$ of some $v_i$.

We call two hypotheses close if they are within $O(k\eps)$ of each other,
and split the good hypotheses into clusters given by the connected components of the closeness operation.
Note that if the high probability events mentioned above hold, each cluster will have diameter $O(k^2\eps)$,
because each good hypothesis must be within $O(k\eps)$ of some $v_i$,
and thus the longest chain of close hypothesis we will need to deal with will have length $O(k)$.

Next, for each cluster centered at some vector $u$, since we know an $(\eta/(2dk))^{Cd}$ approximation
to the tensor $\sum w_i\pi_{U\cap u^\perp}(v_i)^{\otimes 2d}$, by Proposition \ref{coverApplicationProp},
as long as $\eps > \eta$, we can compute a set $\mathcal{C}$ of size at most
$S = (((k\eps)/(\eps/2))2kd)^{O(d^2 k^{1/d})} = (2kd)^{O(d^2 k^{1/d})}$, such that
every $v_i$ with $|\pi_{U\cap u^\perp}(v_i)| = O(k\eps)$ has a hypothesis
in $c\in\mathcal{C}$ within $\eps/4$ of $\pi_{U\cap u^\perp}(v_i)$.
Recalling that $v_i$ has unit length and is within $\eta/2$ of $U$,
if this were the case, then $v$ will be within distance $\eps/2$ of $c+u\sqrt{1-\|c\|_2^2}$.

Thus, in time $\poly(S)$, we can compute a set of at most $S$ hypotheses,
such that every $v_i$ within $O(k\eps)$ of $u$ is within distance $\eps/2$ of some hypothesis in our set.
By applying this to every cluster, we can compute a set of size $(2kd)^{O(d^2 k^{1/d})}$, such that
every $v_i$ is within distance $\eps/2$ of some element of our set.

Note that what we did here was that given a set of size $S$,
where each $v_i$ was within distance $\eps$ of some element,
we produced another such set, but where each $v_i$ was within distance $\eps/2$.
Repeating this procedure \new{$O(\log(1/\eta))$} times, we get a set of size $S$,
where each $v_i$ is within distance $\eta$ of some element of our set.

Next suppose that $\min_{i\neq j}\|v_i-v_j\|_2 \geq \Delta >C k^4 \log(km) \eta$, for $C$ sufficiently large.
Note that with our final set of hypotheses, if we compute clusters as described above,
the good hypotheses in a given cluster will be close to one and only one of the $v_i$'s.

Take an additional $Cmk$ samples. For each of these samples,
associate it with a cluster if $|c\cdot x| \leq \sqrt{k \log(km) \eps}$,
for $c$ the representative of that cluster but not for the representative of any other cluster.
Note that, with high probability, all samples $x$ coming from the $i$-th component satisfy
$|c\cdot x| \leq \sqrt{k \log(km) \eps}$, when $c$ is the representative of their cluster.
But this holds with probability at most $1/(2k)$ when $c$ is the representative of any other cluster.
Thus, with high probability, the set of samples associated with a given cluster consists
of at least $m$ samples coming only from that component of the mixture.
Almost certainly such samples span $v_i^\perp$. Thus, from these samples
we can recover the components $N(0,I-v_iv_i^T)$ with high probability.

Thus, if we assume that $\min \|v_i-v_j\|_2 \geq \Delta$ for some known $\Delta>0$,
we can learn the $v_i$'s with
$N = O(kd \log(m)/\Delta)^{O(d)}+O(m^2)\poly(k\log(m)/\Delta)$ samples and
$\poly(N)+m^2 \log(\dnew{\log(m)}/\Delta) (kd )^{O(d^2 k^{1/d})}$ time.

Given the assumption that $d=O(\log(k))$, this expression can be simplified.
In particular, $N$ can be rewritten as $N=(k/\Delta)^{O(d)} + (\log(m))^{O(d)}+O(m^2)\poly(k\log(m)/\Delta).$
We note that if $k\gg \log(m)$, the $(\log(m))^{O(d)}$ term is dominated by the $(k/\Delta)^{O(d)}$ term and can be removed.
However, if $\log(m) \gg k$, we have that $d\ll \log\log(m)$ and the $(\log(m))^{O(d)}$ term
is dominated by the $O(m^2)$ term and can again be removed.
Thus, we can bound $N$ by $(k/\Delta)^{O(d)} +O(m^2)\poly(k\log(m)/\Delta).$

\bibliographystyle{alpha}
\bibliography{allrefs}

\appendix

\newpage

\section*{APPENDIX}

\section{Omitted Proofs and Facts} \label{sec:appendix}

\subsection{Proof of Fact~\ref{fact:codim-bound}} \label{ssec:codim}

We note that Fact~\ref{fact:codim-bound} is standard and we include a proof for the sake of completeness.
If $A \subseteq B \subseteq C$ are finite dimensional vector spaces,
by the definition of the codimension we have that
\begin{equation}\label{eqn:codim-cor}
\codim_C(A) = \codim_B(A) + \codim_C(B) \;.
\end{equation}
By the subadditivity property of codimension under intersection
we have that
\begin{equation}\label{eqn:codim-subadd}
\codim_W(V \cap U) \leq \codim_W(V) + \codim_W(U) \;.
\end{equation}
An application of \eqref{eqn:codim-cor} for $A = U \cap V$, $B = U$, and $C=W$ gives that
$$\codim_U(V \cap U) = \codim_W(V \cap U) - \codim_W(U) \;.$$
Therefore,
$$\codim_U(V \cap U) \leq \codim_W(V) \;,$$
as desired.

\subsection{Proof of Lemma~\ref{lem:poly-bounds}} \label{ssec:poly-bounds}

Note that a degree-$d$ homogeneous polynomial $p$ on $\R^n$ can be expressed as
$p(x) = \langle A , x^{\otimes d} \rangle$, where $A$ is a real symmetric tensor of dimension $n$ and order $d$.
An application of the Cauchy-Schwarz inequality gives that $|p(x)| \leq \|A\|_2 \, \|x^{\otimes d}\|_2$
for any $x \in \R^n$.
By definition, we have that $\|p\|_{\ell_2} = \|A\|_2$ and $ \|x^{\otimes d}\|_2 = \|x\|_2^{d}$,
giving statement (i).

To prove (ii), we similarly note that
$|p(x)-p(y)| = \left| \langle A , x^{\otimes d} -  y^{\otimes d} \rangle \right| \leq \|A\|_2 \, \|x^{\otimes d}-y^{\otimes d}\|_2$,
where the inequality is Cauchy-Schwarz. Recalling that $\|p\|_{\ell_2} = \|A\|_2$, it suffices to bound
from above $\|x^{\otimes d}-y^{\otimes d}\|_2$. Note that
$$x^{\otimes d}-y^{\otimes d} =
\sum_{i=0}^{d-1} \left( y^{\otimes i} \otimes x^{\otimes (d-i)} - y^{\otimes (i+1)} \otimes x^{\otimes (d-i-1)} \right)
 = \sum_{i=0}^{d-1} \left(y^{\otimes i} \otimes (x-y) \otimes x^{\otimes (d-i-1)}  \right) \;.$$
For all $x \neq y \in \R^n$, we can thus write:
\begin{eqnarray*}
\|x^{\otimes d}-y^{\otimes d}\|_2 &\leq&
\sum_{i=0}^{d-1} \left\| (y^{\otimes i} \otimes (x-y) \otimes x^{\otimes (d-i-1)} ) \right\|_2 \\
 &=& \sum_{i=0}^{d-1} \|y\|_2^i \; \|x-y\|_2 \; \|x\|_2^{d-i-1} \\
 &\leq& d \, \|x-y\|_2  \; \max\{\|x\|_2, \|y\|_2\}^{d-1}  \;.
\end{eqnarray*}
This gives (ii) and completes the proof of Lemma~\ref{lem:poly-bounds}.

\subsection{Proof of Claim~\ref{claim:inner-prod-d}} \label{app:inner-prod-d}

For $i \in \{1, 2\}$ we have that $q_i(x) = \sum_{\alpha: |\alpha| = 1} \wh{q}_{i}^{(\alpha)} \, x^{\alpha}$
and $p_i(y) =  \sum_{\beta: |\beta| = d-1} \wh{p}_{i}^{(\beta)} \, y^{\beta}$, for some
$\wh{q}_{i}^{(\alpha)}, \wh{p}_{i}^{(\beta)} \in \R$. By linearity of the inner product and orthogonality
of monomials (Fact~\ref{fact:ip-monomials}), it suffices to prove the claim for the case 
that the $q_i$'s and $p_i$'s are monomials.
Specifically, it suffices to show that
$\langle x^{\alpha} \, y^{\beta}, x^{\alpha} \, y^{\beta} \rangle = (1/d) \, \langle x^{\alpha}, x^{\alpha} \rangle \, \langle y^{\beta}, y^{\beta}\rangle$.
By viewing $\alpha, \beta$ as $m$-dimensional multi-indices with zero coordinates on the variables corresponding to $y$ and $x$
respectively, we have that
$$\langle x^{\alpha} \, y^{\beta}, x^{\alpha} \, y^{\beta} \rangle = \frac{(\alpha+\beta)!}{|\alpha+\beta|!}
= \frac{\alpha! \beta!}{d!} = \frac{\beta!}{d!} = \frac{\beta!}{d |\beta|!} 
= (1/d) \, \langle x^{\alpha}, x^{\alpha} \rangle \, \langle y^{\beta}, y^{\beta}\rangle \;,$$
where the first equality uses (the second branch of) Fact~\ref{fact:ip-monomials}, 
the second equality uses that $\alpha$ and $\beta$
have disjoint supports, the third and fourth use that $|\alpha| = 1$ and $|\beta| = d-1$ respectively, and the last one
follows from Fact~\ref{fact:ip-monomials}. Furthermore, it is clear that if $\alpha\neq \alpha'$ or $\beta\neq \beta'$ then
$$
\langle x^\alpha y^\beta, x^{\alpha'}y^{\beta'}\rangle = 0 = \langle x^\alpha, x^{\alpha'}\rangle\langle y^\beta, y^{\beta'} \rangle \;.
$$
This completes the proof of Claim~\ref{claim:inner-prod-d}.

\subsection{Additional Probabilistic Tools} \label{app:prob}
Here we record a few additional useful facts from analysis and probability.

\paragraph{KL Divergence and Pinsker's Inequality.}
The \emph{KL divergence} between $P$ and $Q$,
denoted $\dkl(P \| Q)$, is defined as $\dkl (P \| Q) = \int_{\mathbb{R}^m} \log \frac{d P}{d Q} dP$.
The following inequality relates this to the total variation distance.

\begin{fact}[Pinsker's inequality]
\label{fact:pinsker}
Let $P, Q$ be two probability distributions over $\R^m$.
Then $\dtv (P, Q) \leq \sqrt{\frac{1}{2} \dkl (P \| Q)}$.
\end{fact}

\paragraph{VC Inequality.} 
We will require the VC inequality, a standard result from empirical process theory.
To state this theorem, we will need the classical definition of $\VC$ dimension:
\begin{definition}[VC dimension]
A collection of sets $\A$ is said to \emph{shatter} a set $S$ if for all $S' \subseteq S$, there is an $A \in \A$ such that $A \cap S = S'$.
The VC dimension of $\A$, denoted $\VC(\A)$, is the largest $d$ such that there exists a $S$ with $|S| = d$ that $\A$ shatters $S$.
\end{definition}

\noindent
For any collection $\mathcal{A}$ of measurable subsets in $\R^m$, we define the $\mathcal{A}$-norm,
denoted $\| \cdot \|_\mathcal{A}$, on measurable real-valued functions on $\R^m$, to be
$\| f \|_\A = \sup_{A \in \A} |f(A)|$.

We are now ready to state the classical version of the VC theorem:

\begin{theorem}[c.f. Devroye \& Lugosi Theorems 4.3 and 3.2] \label{thm:vc}
Let $f: \R^m \to \R$ be a probability measure, and let $\hat{f}_n$ denote the empirical distribution after $n$ independent draws from $f$.
Then
\[
\E \left[ \| f - \hat{f}_n \|_\A \right] \leq \sqrt{\frac{\VC(\A)}{n}} \;.
\]
\end{theorem}
\noindent By standard uniform deviation arguments (e.g., McDiarmid's inequality), Theorem~\ref{thm:vc} 
has the following simple corollary:
\begin{corollary}
\label{cor:vc}
Let $f, \hat{f}_n, \A$ be as in Theorem \ref{thm:vc}.
Then, for all $\delta > 0$, we have
\[
\pr \left[ \| f - \hat{f}_n \|_\A \geq \sqrt{\frac{\VC(\A) + \log 1 / \delta}{n}} \right] \leq \delta \;.
\]
\end{corollary}

\paragraph{Basics of Hermite Analysis and Concentration} \label{ssec:hermite}
We  review the basics of Hermite analysis over $\R^n$
under the standard $n$-dimensional Gaussian distribution $N(0, I)$.
Consider $L_2(\R^n, N(0, I))$, the vector space of all
functions $f : \R^n \to \R$ such that $\E_{x \sim N(0, I)}[f(x)^2] <\infty$.
This is an inner product space under the inner product
$$\langle f, g \rangle = \E_{x \sim N(0, I)} [f(x)g(x)] \;.$$
This inner product space has a complete orthogonal basis given by
the \emph{Hermite polynomials}. For univariate degree-$i$ Hermite polynomials, $i \in \N$,
we will use the {\em probabilist's} Hermite polynomials, denoted by $He_i(x)$, $x \in \R$,
which are scaled to be {\em monic}, i.e., the lead term of $He_i(x)$ is $x^i$.
For $a \in \N^n$, the $n$-variate Hermite polynomial $He_{a}(x)$, $x = (x_1, \ldots, x_n) \in \R^n$,
is of the form $\prod_{i=1}^n He_{a_i}(x_i)$, and has degree $\|a\|_1 = \sum a_i$.
These polynomials form a basis for the vector space of all polynomials
which is orthogonal under this inner product. \new{We will use various well-known properties of these
polynomials in our proofs.}

\end{document}